\documentclass[twoside]{article}

\usepackage[accepted]{aistats2024}
\usepackage[round]{natbib}

\usepackage{hyperref}
\hypersetup{
	colorlinks   = true, 
	urlcolor     = blue, 
	linkcolor    = blue, 
	citecolor   = blue 
}

\usepackage{array}
\usepackage{tabularx}
\usepackage{microtype}
\usepackage{wrapfig}
\usepackage{graphicx}
\usepackage{subfigure}
\usepackage{bm}
\usepackage{comment}
\usepackage[colorinlistoftodos]{todonotes}
\usepackage{amsmath,amssymb,amsthm,mathtools}

\usepackage{dsfont}
\usepackage[ruled,vlined]{algorithm2e}
\SetKwRepeat{Do}{do}{while}
\usepackage{setspace}

\usepackage[english]{babel} 
\usepackage{blindtext}
\usepackage{booktabs} 
\usepackage{float}
\usepackage{pifont}
\usepackage{titletoc}
\usepackage[toc,page,header]{appendix}
\usepackage{stmaryrd}
\usepackage{textcomp}
\usepackage{cleveref}
\crefname{equation}{}{}
\usepackage{tablefootnote}
\usepackage{enumitem}

\usepackage{lmodern}

\renewcommand{\tilde}{\widetilde}
\renewcommand{\hat}{\widehat}
\renewcommand{\bar}{\overline}

\newcommand{\distsampleq}{Q} 

\newcommand{\finalsampleQt}{Q_t}  

\newcommand{\finalsampleQsub}{Q_}

\newcommand{\barSigmaInv}{\bm{\bar{\Sigma}}^{-1}}
\newcommand{\barSigmaBlock}{\mathbf{\bar{\Sigma_b}}}
\newcommand{\barSigmaBlockInv}{\mathbf{\bar{\Sigma_b}}^{-1}}
\newcommand{\barSigma}{\bm{\bar{\Sigma}}}

\newcommand{\tilSigmaBlock}{\mathbf{\widetilde{\Sigma_b}}}
\newcommand{\tilSigmaBlockInv}{\mathbf{\widetilde{\Sigma_b}}^{-1}}

\newcommand{\alphacirct}{\alpha_t}
\newcommand{\alphaone}{c'_1}
\newcommand{\alphatwo}{c'_2}
\newcommand{\cmark}{\ding{51}}%
\newcommand{\upd}{\scalebox{0.9}{\textsf{upd}}}

\newcommand{\poly}{\mathrm{poly}}

\newcommand{\tilSigmaInv}{\tilde{\mathbf{\Sigma}}^{-1}}
\newcommand{\tilSigma}{\tilde{\bm{\Sigma}}}
\newcommand{\SigmaInv}{\mathbf{\Sigma}^{-1}}
\newcommand{\hatSigmaInv}{\widehat{\mathbf{\Sigma}}_{t,a}^{+}}
\newcommand{\lambdaminSigma}{\lambda_{\min}(\mathbf{\Sigma)}}

\newcommand{\algmainFTRL}{\textsc{FTRL-LC}}
\newcommand{\algmainConEW}{\textsc{MWU-LC}}
\newcommand{\MGR}{{\textsc{MGR}}}

\newcommand{\hattheta}{\widehat{\bm{\theta}}}
\newcommand{\tiltheta}{\widetilde{\bm{\theta}}}
\newcommand{\hatSigma}{\widehat{\mathbf{\Sigma}}}
\newcommand{\one}{\mathbf{1}}

\newtheorem{theorem}     {Theorem}
\newtheorem{lemma}       {Lemma}
\newtheorem{proposition} {Proposition}
\newtheorem{corollary}   {Corollary}
\newtheorem{definition}  {Definition}

\newtheorem{remark} {Remark}

\newcommand{\E}{\mathbb{E}}

\newcommand{\R}{\mathbb{R}}

\newcommand{\bell}{\boldsymbol{\ell}}

\newcommand{\cA}{\mathcal{A}}
\newcommand{\cB}{\mathcal{B}}

\newcommand{\cD}{\mathcal{D}}

\newcommand{\cF}{\mathcal{F}}

\newcommand{\cL}{\mathcal{L}}

\newcommand{\cM}{\mathcal{M}}
\newcommand{\cO}{\mathcal{O}}

\newcommand{\cX}{\mathcal{X}}

\DeclareMathOperator*{\argmin}{arg\,min}
\newcommand{\la}{\langle}
\newcommand{\ra}{\rangle}
\newcommand{\simplexK}{\Delta([K])}

\newcommand{\ind}[1]{\mathds{1}\left[#1\right]}

\newenvironment{prooftext}[1]{\par\noindent{\bf Proof#1.}\quad}{\nopagebreak$\qed$\\}

\newcommand{\spara}[1]{\smallskip\noindent{\bf #1}}

\newcommand{\compilefullversion}{false} 

\ifthenelse{\equal{\compilefullversion}{true}}{%
	\newcommand{\OnlyInFull}[1]{}
	\newcommand{\OnlyInShort}[1]{#1}
}{%
	\newcommand{\OnlyInFull}[1]{#1}%
	\newcommand{\OnlyInShort}[1]{}%
}%

\newcommand{\compilehidecomments}{true}

\definecolor{brown}{rgb}{0.55, 0.25, 0.0}
\definecolor{forestgreen}{rgb}{0.0, 0.5, 0.0}
\definecolor{blue-violet}{rgb}{0.54, 0.17, 0.89}
\definecolor{dartmouthgreen}{rgb}{0.05, 0.5, 0.0}
\definecolor{dark-red}{rgb}{0.8, 0.0, 0.0}
\definecolor{light-blue}{rgb}{0.5, 0.5, 0.99}
\definecolor{brinkpink}{rgb}{0.85, 0.25, 0.5}
\definecolor{columbiablue}{rgb}{0.61, 0.87, 1.0}
\definecolor{cyan(process)}{rgb}{0.0, 0.55, 0.85}
\definecolor{darkcyan}{rgb}{0.0, 0.0, 0.8}
\definecolor{darkorange}{rgb}{0.85, 0.45, 0.0}
\definecolor{deeplilac}{rgb}{0.6, 0.33, 0.73}
\definecolor{electricultramarine}{rgb}{0.25, 0.0, 1.0}
\definecolor{electricviolet}{rgb}{0.56, 0.0, 1.0}

\ifthenelse{ \equal{\compilehidecomments}{true} }{%
	\newcommand{\NCB}[1]{}	
        \newcommand{\FV}[1]{}
	\newcommand{\AR}[1]{}
	\newcommand{\TT}[1]{}
	\newcommand{\YK}[1]{}
}{

	\newcommand{\NCB}[1]{{\color{teal} [\text{NCB:} #1]}}
	\newcommand{\FV}[1]{{\color{electricultramarine}  [\text{FV:} #1]}}
	\newcommand{\AR}[1]{{\color{violet} [\text{AR:} #1]}}
	\newcommand{\TT}[1]{{\color{orange} [\text{TT:} #1]}}
	\newcommand{\YK}[1]{{\color{purple} [\text{YK:} #1]}}
}

\renewcommand{\log}{\ln}

\begin{document}

%

%
\runningauthor{Yuko Kuroki, Alberto Rumi, Taira Tsuchiya, Fabio Vitale, Nicol\`o Cesa-Bianchi}

\twocolumn[

\aistatstitle{Best-of-Both-Worlds Algorithms for Linear Contextual Bandits}

\aistatsauthor{ Yuko Kuroki \And Alberto Rumi \And  Taira Tsuchiya}
\aistatsaddress{ CENTAI Institute\\ Turin, Italy \And  Universit\`a degli Studi di Milano\\ and CENTAI Institute \\ Milan, Italy \And The University of Tokyo\\ Tokyo, Japan}
\aistatsauthor{  Fabio Vitale \And  Nicol\`o Cesa-Bianchi}
\aistatsaddress{ CENTAI Institute\\ Turin, Italy  \And Universit\`a degli Studi di Milano \\ and Politecnico di Milano\\ Milan, Italy}
]



\begin{abstract}
We study best-of-both-worlds algorithms for $K$-armed linear contextual bandits. Our algorithms deliver near-optimal regret bounds in both the adversarial and stochastic regimes, without prior knowledge about the environment. In the stochastic regime, we achieve the polylogarithmic rate $\frac{(dK)^2\poly\!\log(dKT)}{\Delta_{\min}}$, where $\Delta_{\min}$ is the minimum suboptimality gap over the $d$-dimensional context space. In the adversarial regime, we obtain either the first-order $\widetilde{\cO}(dK\sqrt{L^*})$ bound, or the second-order $\widetilde{\cO}(dK\sqrt{\Lambda^*})$ bound, where $L^*$ is the cumulative loss of the best action and $\Lambda^*$ is a notion of the cumulative second moment for the losses incurred by the algorithm. Moreover, we develop an algorithm based on FTRL with Shannon entropy regularizer that does not require the knowledge of the inverse of the covariance matrix, and achieves a polylogarithmic regret in the stochastic regime while obtaining $\widetilde\cO\big(dK\sqrt{T}\big)$ regret bounds in the adversarial regime.
\end{abstract}

\section{INTRODUCTION}
\label{sec:intro}
Because of their relevance in practical applications, contextual bandits are a fundamental model of sequential decision-making with partial feedback. In particular, linear contextual bandits \citep{abe1999associative,auer2002using}, in which contexts are feature vectors and the loss is a linear function of the context, are among the most studied variants of contextual bandits. Traditionally, contextual bandits (and, in particular, their linear variant) have been investigated under stochastic assumptions on the generation of rewards. Namely, the loss of each action is a fixed and unknown linear function of the context to which some zero-mean noise is added.  For this setting, efficient and nearly optimal algorithms, like OFUL \citep{abbasi2011improved} and a contextual variant of Thompson Sampling \citep{agrawal2013thompson}, have been proposed in the past.

Recently, \citet{neu2020efficient} introduced an adversarial variant of linear contextual bandits, where there are $K$ arms and the linear loss associated with each arm is adversarially chosen in each round. They prove an upper bound on the regret of order $\sqrt{dKT}$ disregarding logarithmic factors, where $d$ is the dimensionality of contexts and $T$ is the time horizon. A matching lower bound $\Omega\big(\sqrt{dKT}\big)$ for this model is implied by the results of \citet{Zierahn+2023}. The upper bound has been recently extended by \citet{olkhovskaya2023first}, who show first and second-order regret bounds respectively of the order of $K\sqrt{dL^*}$ and $K\sqrt{d\Lambda^*}$ (again disregarding log factors), where $L^*$ is cumulative loss of the best action and $\Lambda^*$ is a notion of cumulative second moment for the losses incurred by the algorithm.

The above model of $K$-armed linear contextual bandits has also been studied in a stochastic setting---see, e.g., \citep{bastani2021mostly}.
By reducing $K$-armed linear contextual bandits to linear bandits,
and applying the gap-dependent bound of OFUL \citep{abbasi2011improved}, one can show
a regret bound of the order of $\frac{dK}{\Delta_{\min}}\log(T)$ for the stochastic setting, ignoring logarithmic factors in $K$ and $d$, where $\Delta_{\min}$ is the minimum sub-optimality gap over the context space.

\begin{table*}\label{table:mainresults}
  \caption{A comparison of regret bounds for linear contextual bandits. $\tilde{\cO}$ ignores (poly)logarithmic factors.
  The $\sqrt{C}$ column specifies whether in the corrupted stochastic regime the algorithm achieves the optimal $\sqrt{C}$ dependence on the corruption level $C \geq 0$. For the bound in the adversarial regime, we omit additive terms polylogarithmic in $T$.
  See \Cref{sec:problem_setting} for a formal definition of the quantities appearing in the bounds.
  }
  \centering
  \scalebox{0.85}{
  \begin{tabular}{lcccc}
  \toprule
    reference & stochastic & adversarial  & $\sqrt{C}$ & $\SigmaInv$
    \\
    \midrule
     \citet{neu2020efficient} & --   & $\cO\left(\sqrt{T K\max\Big\{d,\frac{\log T}{\lambdaminSigma} \Big\} \log(K)}\right)$   & -- & Unknown \\
     \citet{olkhovskaya2023first} &  -- & $\tilde{\cO}\left(K\sqrt{d \Lambda^*} \right)$ &--  & Known\\
     \citet{olkhovskaya2023first} &  -- & $\tilde{\cO}\left(K\sqrt{d L^*} \right)$ &--  & Known\\
     \citet{Zierahn+2023} &  -- & $\cO\left( \sqrt{T K \max\Big\{d,\frac{ \log T}{\lambdaminSigma} \Big\}\log(K)  } \right)$ &--  & Unknown\\
     \Cref{propsi:linexp3_iw_bounds}  &   $\cO\left( \frac{K^2}{\Delta_{\min}}  \left(d+\frac{1}{\lambdaminSigma}\right)^2 \log(K)\log T \right)$ & $\cO\left(\sqrt{T K^2  \left(d+\frac{1}{\lambdaminSigma}\right)^2 \log(K)} \right)$ & \cmark  & Known \\
     \Cref{thm:secondorder} &  $\cO\left(\frac{(dK)^2}{\Delta_{\min}} \log^2(dKT)\log^3T \right)$  & $\tilde{\cO}\left(dK\sqrt{\Lambda^*}\right)$ & \cmark  & Known\\
     \Cref{cor:min of first and second order bound} & $\cO\left(\frac{(dK)^2}{\Delta_{\min}} \log^2(dKT)\log^3T \right)$ & 
 $\tilde{\cO}\left(dK\sqrt{\min\{L^*, \bar{\Lambda}\}}\right)$ & \cmark & Known \\
     \Cref{thm:FTRLforcontextual} &  $\cO\left(\frac{K}{\Delta_{\min}} \left(d+\frac{\log T}{\lambdaminSigma} \right) \log(KT)\log T  \right)$  & $\cO  \left( 
\sqrt{T K\left(d+\frac{\log T}{\lambdaminSigma}\right)\log(T)\log(K)} \right)$ & \cmark  & Unknown \\  
    \bottomrule
  \end{tabular}
  }
\end{table*}

In this work, we address the problem of obtaining \emph{best-of-both-worlds} (BoBW) bounds for $K$-armed linear contextual bandits: namely, the problem of designing algorithms simultaneously achieving good regret bounds in both the adversarial and stochastic regimes without any prior knowledge about the environment. Starting from the seminal work of  \citet{BubeckSlivkins12b, seldin2014one} for $K$-armed bandits,
there is a growing interest in BoBW results~\citep{Seldin17a, WeiLuo2018, ZimmertSeldin2021}.
Various bounds have been established for different models, including linear bandits~\citep{Lee+2021,kong2023bestofthreeworlds, ito2023best,ItoTakemura2023}, contextual bandits \citep{pacchiano2022best, Dann+2023}, $K$-armed bandits with feedback graphs \citep{ito+2022nearly,rouyer2022near},
combinatorial semi-bandits \citep{Zimmert+19,Ito2021Hybrid},
episodic MDPs \citep{jin2021best}, to name a few.
However, known BoBW results for contextual bandits are not satisfying. The algorithm of \citet{Dann+2023} essentially relies on \textsc{Exp}4, which is computationally expensive when the class of policies is large.
In this paper, we devise the first BoBW algorithms for $K$-armed linear contextual bandits that, among other advantages, can be implemented in time polynomial in $d$ and $K$.
Next, we list the main contributions of this work.

\paragraph{Contributions.}
We introduce the first BoBW algorithms for $K$-armed linear contextual bandits. In the stochastic regime, our algorithms achieve the (poly)logarithmic rate
$\frac{(dK)^2\poly\!\log(dKT)}{\Delta_{\min}}$.
In the adversarial regime, we obtain either a first-order $\widetilde{\cO}(dK\sqrt{L^*})$ bound, or a second order $\widetilde{\cO}(dK\sqrt{\Lambda^*})$ bound (\Cref{thm:secondorder} and \Cref{cor:min of first and second order bound}).
We also propose a simpler and more efficient algorithm based on the follow-the-regularized-leader (\textsc{FTRL}) framework,
that simultaneously achieves polylogarithmic regret in the stochastic regime and $\widetilde\cO\big(dK\sqrt{T}\big)$ regret in the adversarial regime (\Cref{thm:FTRLforcontextual}),
 without prior knowledge of the inverse of the contextual covariance matrix $\mathbf{\Sigma}$.
Our proposed algorithms are also applicable to the corrupted stochastic regime.

\paragraph{Techniques.}
Our data-dependent bounds are based on the black-box framework proposed by \citet{Dann+2023}, who provide a unified algorithm transforming a bandit algorithm for the adversarial regime into a BoBW algorithm.
Directly adapting to our setting the results for contextual bandits with finite policy classes in their work involves a prohibitive computational cost, since it is known that the number of policies to consider in the adversarial regime is of order $\big(TK^{-2}d^{-1}\big)^{Kd}$~\citep{pmlr-v80-allen-zhu18b,olkhovskaya2023first}.
Within the same framework, we may also apply the \textsc{Exp3}-type algorithm of \citet{neu2020efficient}. However, this only results in zero-order (i.e., not data-dependent) regret bounds $\cO(\sqrt{T})$---see ~\Cref{propsi:linexp3_iw_bounds} in \Cref{subsec:appendix naive adaption}.
In order to obtain data-dependent guarantees, we instead apply the continuous exponential weights algorithm for adversarial linear contextual bandits recently investigated by~\citet{olkhovskaya2023first}. In particular, we show that it is possible to choose the learning rates so as to fulfill the data-dependent stability condition required in~\citet{Dann+2023} for applying their black-box framework.

The data-dependent bounds achieved by the black-box approach are favorable in the sense that the algorithm performs well when there is an action achieving a small cumulative loss or the loss has a small variance.
However, this approach may have limitations as it requires knowledge of the inverse of the covariance matrix $\SigmaInv$ and may not be practical to implement.
To overcome this issue, we show how \textsc{FTRL} with Shannon entropy regularization---which is a much more practical algorithm---can be run with an estimate of $\SigmaInv$ computed using \emph{Matrix Geometric Resampling} (\textsc{MGR}) of \citet{Neu_Bartok2013,Neu_Bartok2016}, thus avoiding the advance knowledge of $\SigmaInv$.
In order to construct this algorithm,
we rely on an adaptive learning rate framework for obtaining BoBW guarantees in \textsc{FTRL} with Shannon entropy regularization,
proposed in~\citet{ito+2022nearly} and later used in~\citet{Tsuchiya+2023,Tsuchiya+2023stability,kong2023bestofthreeworlds}.
The difference from their work is that
while they crucially rely on the unbiasedness of the loss estimator,
we need to deal with the \emph{biased} loss estimator that comes from the use of the covariance matrix estimation in \textsc{MGR}.
\citet{neu2020efficient} and \citet{Zierahn+2023} applied \textsc{FTRL}+\textsc{MGR}, which allows controlling the bias of the loss estimator, but they focused only on the adversarial regime.
 Moreover, their methods only attain a sub-optimal regret bound $\cO(\sqrt{T})$ in the stochastic regime.
The derivation of our bounds for $K$-armed linear contextual bandits requires nontrivial scheduling of the learning rates and of the adaptive mixing rates of exploration. 
With these techniques, we successfully provide the first BoBW bounds for $K$-armed linear contextual bandits without knowing  $\mathbf{\Sigma}^{-1}$.

Table~\ref{table:mainresults} summarizes our results in the context of the previous literature. 
The upper bound of \citet{Zierahn+2023} is for a combinatorial contextual setting where the action space satisfies $\cA \subseteq \{0,1\}^K$ and we assume $\max_{a \in \cA}\|a\|_1 \leq 1$.
The best known lower bound for the adversarial or distribution-free setting is $\Omega\big(\sqrt{d K T}\big)$ also due to \citet{Zierahn+2023}, see \Cref{appendix:lowerbound}.

\spara{Related work.}
Despite the vast literature on contextual bandits~\citep{Chu+11,Syrgkanis16, Rakhlin16, zhao2021linear,Ding+22c,He+2o22, liu2023bypassing}, only a few data-dependent bounds have been proven since the question was posed by~\citet{pmlr-v65-agarwal17a}.
The first result is by \citet{pmlr-v80-allen-zhu18b}, but the algorithm is not applicable to a large class of policies.
\citet{foster2021firstorder} obtained first-order bounds for stochastic losses via an efficient regression-based algorithm. Recently \citet{olkhovskaya2023first} proved first- and second-order bounds for stochastic contexts but adversarial losses. Yet, BoBW bounds are not addressed in these studies. 
There are some BoBW results in the model selection problem \citep{Pacchiano+2020stochastic,pacchiano2022best, agarwalcorral17b, cutkosky21a, Lee+2021, Wei+22}. In particular, \citet{pacchiano2022best} achieved the first BoBW high-probability regret bound for general contextual linear bandits. However, the algorithm achieving this result has a running time linear in the number of policies, which makes it intractable for infinite policy classes.
A more detailed review of related works can be found in  \Cref{appendix:related work}.

\section{PROBLEM STATEMENT}
\label{sec:problem_setting}


Given a $K$-action set $[K] :=\{1,2,\ldots,K\}$, a context space of a full-dimensional compact set $\cX \subseteq \mathbb{R}^d$,
and a distribution $\cD$ over $\cX$, our learning protocol can be described as follows.
At each time step $t=1,2,\ldots, T$: 
\begin{itemize}[topsep=0pt, itemsep=0pt, partopsep=0pt, leftmargin=12pt] 
    \item For each action $a \in [K]$, the environment chooses a loss vector $\bm{\theta}_{t,a} \in \mathbb{R}^d$

    \item Independently of the choice of loss vectors $\bm{\theta}_{t,a}$ for $a \in [K]$,
    the environment draws 
    the context vector $X_t \in \cX$ from the context distribution $\cD$ unknown to the learner

    \item The learner observes context $X_t$ and chooses action $A_t \in [K]$ 
    
     \item The learner incurs and observes the loss $\ell_t(X_t, A_t)$.
\end{itemize}

\spara{Assumptions.}
Like previous works on adversarial linear contextual bandits~\citep{neu2020efficient,olkhovskaya2023first,Zierahn+2023} and linear bandits~\citep{Lee+2021,Dann+2023}, we make the following assumptions:
\begin{itemize}

    \item The distribution $\cD$ from which contexts $X$ are independently drawn satisfies $\E[X X^{\top}]=\bm{\Sigma} \succ 0$;

    \item  $\|X\|_2 \leq 1 $ $\cD$-almost surely;

    \item  $\| \bm{\theta}_{t,a}\|_2 \leq 1$ for all $a \in [K]$ and $t \in [T]$;

    \item $\ell_t(\bm{x},a) \in [-1,1]$ 
    for all $\bm{x} \in \cX$, $a \in [K]$, and $t \in [T]$.

\end{itemize}

Further conditions on the loss functions $\ell_t(\bm{x},a)$ as well as the loss vectors $\bm{\theta}_{t,a}$ for each $a \in [K]$ and $t$ are defined in each regime as follows.

\spara{Adversarial regime.}
The loss function is defined by $\ell_t(X_t,a):=\la X_t, \bm{\theta}_{t,a}\ra$, where $\bm{\theta}_{t,a}$ is chosen by an oblivious adversary for all $a$ and $t$.

\spara{Stochastic regime.}
The loss function is defined by $\ell_t(X_t,a):= \la X_t,\bm{\theta}_{a} \ra+\varepsilon_t(X_t,a)$, 
where $\bm{\theta}_a$ for each action $a$ is fixed and unknown, and  $\varepsilon_t(X_t,a)$ is independent and bounded zero-mean noise.

\spara{Corrupted stochastic regime.}
The loss function is defined by $\ell_t(X_t,a):= \la X_t,\bm{\theta}_{t,a} \ra+\varepsilon_t(X_t,a)$, where $\varepsilon_t(X_t,a)$ is independent and bounded zero-mean noise and the vectors $\bm{\theta}_{t,a}$ are such that there exist fixed and unknown vectors $\bm{\theta}_1,\ldots,\bm{\theta}_K$ and an unknown constant $C > 0$ for which
$\sum_{t=1}^T \max_{a \in [K]} \|\bm{\theta}_{t,a}-\bm{\theta}_a \|_2 \leq C$ holds. Note that $C=0$ corresponds to the stochastic regime and $C = \Theta(T)$ corresponds to the adversarial regime with additional zero-mean noise.

Let $\Pi$ be the set of all deterministic policies $\pi : \cX \rightarrow [K]$ mapping contexts to actions.
We define $\pi^* \in \Pi$ as the optimal policy:
\begin{equation}\label{def:optimal_policy}
\pi^*(\bm{x}):=\argmin_{a \in [K]}\mathbb{E}\left[ \sum_{t=1}^T    \ell_t(\bm{x},a)\right] \quad \forall \bm{x} \in \cX,
\end{equation}
where the expectation is taken over the randomness by loss functions. \YK{I added $\E$ since the loss and noise is random in the stochastic regime}
Then, the learner's goal is to minimize the total expected regret against the optimal policy $\pi^*$:
\begin{equation}
    R_T=
     \mathbb{E}\left[ \sum_{t=1}^T \Big(\ell_t(X_t, A_t)-\ell_t(X_t, \pi^*(X_t))\Big)\right],
\end{equation}
where the expectation is taken over the learner's randomness as well as the sequence of random contexts and loss functions.

In the (corrupted) stochastic regime, given $\bm{\theta}_1,\ldots,\bm{\theta}_K$,
let $\Delta_{\min}(\bm{x}):=\min_{a \neq \pi^*(\bm{x})} \la \bm{x},\bm{\theta}_{a}-\bm{\theta}_{\pi^*(\bm{x})}\ra$
for all $\bm{x} \in \cX$. 
Then, we define the minimum sub-optimality gap by
$
\Delta_{\min}:=\min_{\bm{x} \in \cX} \Delta_{\min}(\bm{x}) > 0
$.

We denote 
the cumulative loss incurred by the optimal policy by $L^*: = \E\big[\sum_{t=1}^T\ell_t(X_t, \pi^*(X_t)) \big]$ and
the cumulative variance of a policy choosing actions $A_1,A_2,\ldots$ with respect to a predictable loss sequence $\bm{m}_{t,a} \in \mathbb{R}^d$ for action $a$ by $\Lambda^*:=\E\big[\sum_{t=1}^T (\ell_t(X_t,A_t)-\la X_t,\bm{m}_{t,A_t} \ra)^2 \big]$.
We use $\bar{\Lambda}:=\E\big[\sum_{t=1}^T (\ell_t(X_t,A_t)-\la X_t, \bar{\bm{\theta}} \ra)^2 \big]$ 
with $\bar{\bm{\theta}}:=\frac{1}{TK}\sum_{t=1}^T \sum_{a=1 }^K\bm{\theta}_{t,a}$.

\spara{Additional notation.}
We denote by $\E_X[\cdot]$ the expectation over a random variable (r.v.)\ $X$. We denote by $\E_X[\cdot|Y]$ the expectation over $X$ conditioned on $Y$. When we write $\E[X]\cdot\E[X|Y]$, we take the expectation conditioned on $Y$ with respect to all sources of randomness in $X$.
We denote by $\cF_t=\sigma(X_s, A_s,\forall s \leq t)$ the 
filtration generated by all the random variables $X_s$ and the set of actions $A_s$, for each $s \leq t$. 
Then we write $\E_t[\cdot]=\E[\cdot|\cF_{t-1}]$.
For any semi-definite matrix $\mathbf{B} \in \mathbb{R}^{d \times d}$, we use $\lambda_{\min}(\mathbf{B})$ to denote the smallest eigenvalue of $\mathbf{B}$,
and write $\|\bm{u}\|_\mathbf{B}=\sqrt{\bm{u}^{\top}\mathbf{B}\bm{u}}$ for $\bm{u} \in \mathbb{R}^d$.
We also define the \emph{probabilistic policy} mapping each context $\bm{x}$ to a probability distribution $\pi(\cdot\mid\bm{x})$ over $[K]$ (i.e., an element of the simplex $\Delta([K])$).
For the analysis of data-dependent bounds, we use the notion $\xi_{t,a}:=(\ell_t(X_t,a)-\la X_t,\bm{m}_{t,a} \ra) \in \mathbb{R}$ with a loss predictor $\bm{m}_{t,a}$ for $t \in [T]$ and $a \in [K]$.
 We write $\ind{\cdot}$ to denote the indicator function.

\section{FOLLOW-THE-REGULARIZED-LEADER}\label{sec:preliminary}

Following the existing BoBW algorithms, 
we rely on the \textsc{FTRL} framework.
Given context $X_t$, we consider the \textsc{FTRL} predictor in $\simplexK$ defined as
\begin{equation*}
    p_t(\cdot|X_t) \in \argmin_{r \in \simplexK}\left\{ \sum_{s=1}^{t-1} \la r, \widehat{\bell}_{s}(X_t) \ra + \psi_t(r)\right\},
\end{equation*}
where 
$\widehat{\bell}_{s}(X_t):=\big( \la X_t, \hat{\bm{\theta}}_{s,1} \ra, \ldots,  \la X_t, \hat{\bm{\theta}}_{s,K} \ra \big)^\top \in \mathbb{R}^K$,
and $\hat{\bm{\theta}}_{t,a}$ is an estimator of the linear loss $\bm{\theta}_{t,a} \in \mathbb{R}^d$.
We use the (negative) Shannon entropy $\psi_t(r) = -\frac{H(r)}{\eta_t}$ as the regularizer, where $H$ is the Shannon entropy
and $\eta_t>0$ is a learning rate. 
It is well known that $p_t(\cdot|X_t)$ is equivalent to the \textsc{Exp3}-type prediction
\begin{equation}\label{eq: exp3}
    p_t(a|X_t)= \frac{\exp{(-\eta_t \sum_{s=1}^{t-1} \la X_t, \hat{\bm{\theta}}_{s,a} \ra  )}}{\sum_{b \in [K]}\exp{(-\eta_t \sum_{s=1}^{t-1} \la X_t, \hat{\bm{\theta}}_{s,b} \ra  )}}\,.
\end{equation}
The learner's policy $\pi_t(\cdot|X_t) \in \simplexK$ that selects the next action usually combines $p_t(\cdot|X_t)$ with some exploration strategy to control the variance of the loss estimates.

We next introduce the Optimistic FTRL (\textsc{OFTRL}) framework~\citep{RakhlinSridharan2013online}.
In \textsc{OFTRL}, a loss predictor $\bm{m}_{t,a} \in \mathbb{R}^d$ for each action $a$ is available to the learner at the beginning of each round $t$. \textsc{OFTRL} can be viewed as adding $\bm{m}_{t,a}$ to the objective as a guess for the next loss vector. The \textsc{OFTRL} prediction $p_t(\cdot|X_t)$ is then defined as
\begin{equation*}
     \argmin_{r \in \simplexK} \left\{ \sum_{s=1}^{t-1}\la r, \widehat{\bell}_{s}(X_t) \ra + \la r, \bm{m}_t(X_t)\ra + \psi_t(r) \right\},
\end{equation*}
where $\bm{m}_t(X_t) := \big( \la X_t, \bm{m}_{t,1} \ra,\ldots,\la X_t,\bm{m}_{t,K}\ra \big) \in \mathbb{R}^K$.

In the following sections, we apply \textsc{OFTRL} in \Cref{thm:secondorder} exploiting the predicted loss $\bm{m}_t(X_t)$ to achieve first- and second-order regret bounds, and in \Cref{thm:FTRLforcontextual}, we apply \textsc{FTRL} to obtain a worst-case regret bound in the adversarial regime, while guaranteeing the polylogarithmic regret in the stochastic regime.

\section{DATA-DEPENDENT BOUNDS}\label{sec:data-dependent}
In this section, we discuss how the reduction framework is adapted to $K$-armed linear contextual bandits. 
We design an algorithm, \algmainConEW, that satisfies the data-dependent stability conditions (\Cref{proposi:ddiw_forlincon}), so that we can use it as a base algorithm in the black-box reduction of \citet{Dann+2023} and obtain the desired BoBW bound for arbitrary $\bm{m}_{t,a}$ (\Cref{thm:secondorder}).
 By choosing the appropriate loss predictor $\bm{m}_{t,a}$, we also show how to simultaneously achieve first- and second-order bounds (\Cref{cor:min of first and second order bound}).

\algmainConEW\ (Algorithm~\ref{alg:secondorder_linearcon}) is an instance of \textsc{OFTRL} 
using a multiplicative weight update.
Notably, such an approach has been taken by \citet{Ito_secondorder_linear2020} for adversarial linear bandits where they use truncated distribution techniques to make an unbiased loss estimator stable.
Recently, \citet{olkhovskaya2023first} extended \citet{Ito_secondorder_linear2020} to the adversarial $K$-armed linear contextual bandits.
\algmainConEW\ is built upon the algorithm of \citet{olkhovskaya2023first}, but in a setting where a loss is observed with some probability $q_t$. The design of the learning rate is significantly modified in order to achieve BoBW bounds.
In particular, we show that \algmainConEW\ achieves a stability condition called \emph{data-dependent importance-weighting stability} (see Definition~\ref{def:ddiw-stable} in \Cref{appendix:dann+2023}).

\spara{Additional assumptions.}
If the density function $h:\mathbb{R}^d \rightarrow \mathbb{R}_{\geq 0}$ has a convex support and $\log(h(y))$ for $y \in \mathbb{R}^d$ is a concave function on the support, we call the distribution \emph{log-concave}.
As in \citet{olkhovskaya2023first}, we assume that (i) context distribution $\cD$ is log-concave and its support is known to the learner,
and (ii) the learner has access to $\SigmaInv$, the inverse of the covariance matrix of contexts.
However, these assumptions will be both dropped in \Cref{sec:FTRLShannon}.
We assume that loss predictors satisfy $\la X_t, \bm{m}_{t,a} \ra \in [-1,1]$ for all $t$ and $a \in [K]$. Finally, when we discuss first-order regret bounds, we assume $0 \leq \ell_t(X_t,a) \leq 1$ for all $t$ and $a \in [K]$, which is a standard assumption to ensure that $L^* = \E\Big[\sum_{t=1}^T\ell_t(X_t, \pi^*(X_t))\Big] \geq 0$.


\spara{Continuous MWU method.}
The learner has access to a loss predictor $\bm{m}_{t,a} \in \mathbb{R}^d$ for each action $a$ at round $t$, also called the hint vector. 
The learner computes the density $p_t(\cdot|X_t)$ supported on $\simplexK$ and based on the continuous exponential weights $w_t(\cdot|X_t)$:
\begin{align}\label{dist:lincon_continuous}
    &w_t(r|X_t): =\exp{\left(-\eta_t \left(\sum_{s=1}^{t-1}\la r, \widehat{\bell}_{s}(X_t) \ra + \left\la  r, \bm{m}_t(X_t)\right\ra  \right)\right)},\notag\\
    &p_t(r|X_t):= \frac{w_t(r|X_t)}{\int_{ \simplexK  } w_t(y|X_t) \ dy },
\end{align}
where $r\in\simplexK$, $\eta_t>0$ is a learning rate, and $\hat{\bm{\theta}}_{s,a}$ is the unbiased estimate for the loss vectors $\bm{\theta}_{s,a}$, which will be determined later.

For the rejection sampling step in \Cref{algline:rejection step2}-\ref{algline:rejection step}, we use the following covariance matrix $\bm{\bar{\Sigma}}_{t,a} \in \mathbb{R}^{d \times d}$:
\begin{equation}\label{eq: def of barSigma}  
\bm{\bar{\Sigma}}_{t,a}:=\E_{X,y_t \sim p_t(\cdot|X)} \left[y_t(a)^2 X X^{\top} | \cF_{t-1}\right]~.
\end{equation}
The number of steps required for the rejection sampling is $\cO(1)$, which can be shown via the concentration property of the log-concave distribution (e.g., Lemma 1 of \citet{Ito_secondorder_linear2020}) and the log-concavity of $\cD$.
The truncated distribution $\widetilde{p}_t(\cdot | X_t)$ of $p_t(\cdot | X_t)$ is defined as:
\begin{equation*}
    \widetilde{p}_t(r | X_t):=  \frac{p_t(r|X_t) \ind {\sum_{a=1}^K r_a^2 \|X_t\|^2_{\barSigmaInv_{t,a}}  \leq dK \tilde{\gamma}_t^2 } }{\mathbb{P}_{y \sim p_t(\cdot |X_t)}\left[\sum_{a=1}^K y_a^2 \|X_t\|_{\barSigmaInv_{t,a}}^2  \leq d K\tilde{\gamma}_t^2  \right]}
\end{equation*}
for $r \in \simplexK$, where $\tilde{\gamma}_t >1$ is the truncation level to be specified soon.
Thus, $Q_t \in \simplexK$ is sampled from the truncated distribution $\widetilde{p}_t( \cdot | X_t)$ and the learner chooses an action $A_t \sim Q_t$.
The probability that the learner can observe a loss, $q_t \in  (0,1]$ (calculated in \Cref{alg:dd LSBviaCorral} in \Cref{appendix:dann+2023}), is given to the base algorithm in the reduction framework.
If the learner observes a loss, then $\upd_t$ is set to $1$, otherwise $\upd_t$ is set to $0$.
Then \algmainConEW\ constructs an unbiased estimator $\widehat{\bm{\theta}}_{t,a}$ of $\bm{\theta}_{t,a}$ for each $a \in [K]$ as follows:
\begin{equation}\label{eq:unbiased_linearcon}
   \widehat{\bm{\theta}}_{t,a:}= \bm{m}_{t,a}+ \frac{\upd_t}{q_t}  
  \finalsampleQt(a) \tilSigmaInv_{t,a} X_t
  \xi_{t,a} \ind{A_t=a},
\end{equation}
where  $\xi_{t,a}=(\ell_t(X_t,a)-\la X_t,\bm{m}_{t,a} \ra)$ and $\tilSigma_{t,a} \in \mathbb{R}^{d \times d}$ is given by:
\begin{equation}\label{eq: def tilde Sigma}
  \tilSigma_{t,a}:= \E_{X}\left[\finalsampleQt(a)^2 X X^{\top} | \cF_{t-1}\right]~.
\end{equation}

For \algmainConEW\ with update probability $q_t$, we design a novel update rule for the learning rate $\eta_t>0$ as follows:
     \begin{equation}\label{eq: def of eta for MWU}
   \eta_t: = \left(\frac{ 800dK \tilde{\gamma}_t^2}{\min_{j \leq t} q_j} +    \sum_{j=1}^{t-1} \frac{\beta_j}{q_j} \right)^{-\frac{1}{2}},
     \end{equation}
where we set $\beta_t:=16 \tilde{\gamma}_t^2 \xi^2_{t,A_t}$ and  $\tilde{\gamma}_t:=4 \log (10dKt)$ for $t \in [T]$.

\begin{algorithm}[t]
\caption{Continuous MWU (\algmainConEW)}
\label{alg:secondorder_linearcon}
 	\SetKwInOut{Input}{Input}
 	\SetKwInOut{Output}{Output}
	\Input{Set of $K$ arms}

    Receive update probability $q_t$;

   \For{$t=1,2,\ldots, T$}{

    Observe $X_t$;

   \Do{$\sum_{a=1}^K  \distsampleq(a)^2 \|  X_t \|^2_{\bm{\bar{\Sigma}}^{-1}_{t,a}}\leq dK\tilde{\gamma}_t^2$ \label{algline:rejection step} 
    }{
    Draw $\distsampleq \sim p_t(\cdot|X_t)$ defined in \eqref{dist:lincon_continuous}  \label{algline:rejection step2} 
    }
  
   $\distsampleq_t \leftarrow \distsampleq  \in \simplexK$;

   Choose an action $A_t \sim \distsampleq_t$;
   
    With probability $q_t$, observe the loss $\ell_t(X_t, A_t)$ as a feedback;

         Compute  $\widehat{\bm{\theta}}_{t,a}$ for $a\in [K]$ as in \Cref{eq:unbiased_linearcon};

         Update $p_t(\cdot|X_t)$ as in \Cref{dist:lincon_continuous};

         Update $\eta_t$  as in \Cref{eq: def of eta for MWU};

    }
 
\end{algorithm}

\spara{Theoretical results.}
The following proposition implies that \algmainConEW\ satisfies the data-dependent importance-weighting stability.
The proof is provided in \Cref{appendix:second_order}.
\begin{proposition}
\label{proposi:ddiw_forlincon}
Assume that $\bm{\bar{\Sigma}}_{t,a}$ in \Cref{eq: def of barSigma} and $\bm{\bar{\Sigma}}_{t,a}$ in \Cref{eq: def tilde Sigma} are known to the learner at each round $t$ and action $a$. 
Given an adaptive sequence of weights $q_1, q_2, \ldots \in (0,1]$, suppose that \algmainConEW\ observes the feedback in round $t$ with probability $q_t$.
Let $R(\tau,a^*)=\E\left[ \sum_{t=1}^{\tau }\ell_{t}(X_t,A_t)-\ell_t(X_t,a^*) \right]$ for round $\tau \in [1,T]$ and comparator action $a^* \in [K]$.
Let $\kappa(d,K,T)=32 K d \log(10dK \tau) \log(\tau)$.
Then, 
for any  $\tau$ and $a^*$,
the regret  $R({\tau},a^*)$ of \algmainConEW\  is bounded by 
    \begin{align*}
\kappa(d,K,T) \left(\sqrt{  \E \left[\sum_{t=1}^{\tau}\frac{\upd_t \xi^2_{t,A_t}}{q_t^2}\right]  } 
+ \E \left[ 
\frac{\sqrt{50dK}}{\min_{j \leq \tau} q_j} \right] \right).
\end{align*}
\end{proposition}
Owing to \Cref{proposi:ddiw_forlincon},
if \algmainConEW\ is run with the black-box reduction procedure (Algorithms \ref{alg:LSBtoBOBW} and \ref{alg:dd LSBviaCorral} in \Cref{appendix:dann+2023}) as a base algorithm, we obtain the following BoBW guarantee.
\begin{theorem}\label{thm:secondorder}
Assume that $\bm{\bar{\Sigma}}_{t,a}$ in \Cref{eq: def of barSigma} and $\bm{\bar{\Sigma}}_{t,a}$ in \Cref{eq: def tilde Sigma} are known to the learner at each round $t$ and action $a$. 
Let  $\kappa_1(d,K,T)=K^2 d^2 \log^2(dK T) \log^2(T)$ and $\kappa_2(d,K,T)= (dK)^{3/2} \log(dK T) \log(T)$ be problem-dependent constants.
Combining the base algorithm \algmainConEW 
 \ (\Cref{alg:secondorder_linearcon}) with Algorithms \ref{alg:LSBtoBOBW} and \ref{alg:dd LSBviaCorral}, 
 it holds that
 \begin{equation*}
 R_T=\cO\left(\sqrt{\kappa_1(d,K,T) \Lambda^* \log^2 T}  +\kappa_2(d,K,T) \log^2(T) \right)
 \end{equation*}
 in the adversarial regime and 
 \begin{alignat*}{4}
 R_T=&\cO \left(\frac{\kappa_1(d,K,T) \log(T)}{\Delta_{\min}}+  \sqrt{ \frac{\kappa_1(d,K,T) \log T C}{\Delta_{\min}}}  \right. \\  &  \left. + \kappa_2(d,K,T) \log(T)\log \frac{C} {\Delta_{\min}}\right)
 \end{alignat*}
 in the corrupted stochastic regime. 
\end{theorem}
For a concrete choice of $\bm{m}_{t,a}$ for each $a \in [K]$, which in turn determines $\Lambda^*$,
we utilize the online optimization method.
For any positive semi-definite matrix $\mathbf{S} \in \mathbb{R}^{d \times d}$, define the predictor $\bm{m}_{t,a}$ as a vector in $\cM:=\{\bm{m} \in \mathbb{R}^d \mid \la \bm{x},\bm{m} \ra \leq 1, \ \forall \bm{x} \in \mathcal{X} \}$ that minimizes the following expression:
\begin{equation}\label{eq:predictor}
\|\bm{m}\|_{\mathbf{S}}^2+  \sum_{j=1}^{t-1}\ind{A_j=a}\left(\left\langle \bm{\theta}_{j,a}-\bm{m}, X_j\right\rangle\right)^{2}
\end{equation}
Based on \cite{Ito_secondorder_linear2020}, we construct $\mathbf{S}$ via the barycentric spanner for $\cX$~\citep{AwerbuchKleinberg2004}, which is given by \Cref{eq:choice of S} in \Cref{appendix:second_order}.
Then, we show the following corollary using $\mathbf{S}$, which implies that we obtain the regret bound depending on $\sqrt{\min\{L^*, \bar{\Lambda}\}}$, see Section~\ref{sec:problem_setting} for a definition of $\bar{\Lambda}$.
\begin{corollary}\label{cor:min of first and second order bound}
Let $\bm{m}_{t,a}$ at each $t \in [T]$ and $a \in [K]$ be given by the minimizer of \Cref{eq:predictor}.
Then, under the same assumptions as Theorem~\ref{thm:secondorder} and for any $\bm{m}^* \in \cM$, $R_T$ is bounded by
  \begin{equation*}
\tilde{\cO}\! \left( Kd \sqrt{\! \min\left\{ L^*,  \E\Big[\sum_{t=1}^T \la X_t, \bm{\theta}_{t,A_t}-\bm{m}^* \ra^2 \Big]\right\}}\!\! +K^2 d^2  \! \right)
 \end{equation*}
 for the adversarial regime, and is the same regret as \Cref{thm:secondorder} for the corrupted stochastic regime. 
\end{corollary}

    \begin{remark}
        Although the first-order bound is obtained by just setting $\bm{m}_{t,a} = \bm{0}$ (see \Cref{coro:first-order} in \Cref{appendix:second_order}), computing the minimizer of \Cref{eq:predictor} as $\bm{m}_{t,a}$ allows a single algorithm to achieve first- and second-order bounds simultaneously. Compared with \citet{olkhovskaya2023first}, our results only have an additional factor $\sqrt{d}$ in the adversarial regime while also providing gap-dependent polylogarithmic regret in the (corrupted) stochastic regime.
    \end{remark} 


We just saw how our first approach in this section achieves theoretical advantages and a polynomial-time running time due to the log-concavity of $\cD$. However, removing the prior knowledge of $\SigmaInv$ seems challenging, as computation of \Cref{eq: def of barSigma} and \Cref{eq: def tilde Sigma} involves expectation depending on both $\cD$ and a learner's policy. 
Moreover, the continuous exponential weights incur a high (yet polynomial) sampling cost, resulting in $\cO\big((K^5+\log T)g_{\mathbf{\Sigma}_t}\big)$ per round running time,
where $g_{\mathbf{\Sigma}_t}$ is the time to construct the covariance matrix for each round (see Section~3.3 in \citet{olkhovskaya2023first} or Section~4.4 in \citet{Ito_secondorder_linear2020} for detailed discussion).
To address these issues, we next devise a simpler solution using FTRL instead of relying on the reduction framework.

\section{UNKNOWN $\mathbf{\Sigma}^{-1}$ CASE}\label{sec:FTRLShannon}

We present a computationally efficient algorithm, called \algmainFTRL, based on FTRL with negative Shannon entropy. This algorithm does not require knowledge of $\SigmaInv$, and only needs access to context distribution $\cD$ and minimum eigenvalue $\lambdaminSigma$.

\spara{Proposed method. }
Recall that, given context $X_t$, 
\textsc{FTRL} computes the probability vector $p_t(\cdot | X_t) \in \simplexK$ as follows: 
\begin{equation}\label{eq:FTRL_shannon_entropy 1}
    p_t(\cdot |X_t)
    :=
    \argmin_{r\in \simplexK}
    \left\{ \sum_{s=1}^{t-1} \, \la r, \widetilde{\bell}_{s} (X_t)  \ra  +\psi_t(r) \right\},
\end{equation}
where
$\psi_t \colon \simplexK \to \R$ is the convex regularizer,
$\widetilde{\bell}_{s}(X_t):=( 
 \la X_t, \tiltheta_{s,1} \ra, \ldots,  \la X_t, \tiltheta_{s,K} \ra) \in \mathbb{R}^K$, and
$\widetilde{\bm{\theta}}_{s,a} \in \mathbb{R}^d$ is the (possibly biased) estimator for $\bm{\theta}_{s,a}$.
Then, the policy $\pi_t( \cdot | X_t)$ that selects the action $A_t$ is defined by mixing the probability vector $p_t( \cdot | X_t)$ with uniform exploration, where the adaptive mixture rate $\gamma_t \in [0,1/2]$ is defined later in~\Cref{def_etabetagamma}.
For the regularizer in~\Cref{eq:FTRL_shannon_entropy 1}, we use the (negative) Shannon entropy $\psi_t(r)=-\frac{1}{\eta_t}H(r)$ as introduced in \Cref{sec:preliminary},
where the learning rate $\eta_t>0$ will be specified later.
The pseudo-code of \algmainFTRL\ is given in \Cref{alg:FTRLforcontextual}.

\begin{algorithm}[t]
\caption{FTRL with Shannon entropy (\algmainFTRL)}
\label{alg:FTRLforcontextual}
 	\SetKwInOut{Input}{Input}
 	\SetKwInOut{Output}{Output}
	\Input{Arms $[K]$}

        \textbf{Initialization:} Set $\tilde{\bm{\theta}}_{0,a} = \mathbf{0}$ for all $a \in [K]$. Initialize $\eta_1$ and $\gamma_1$ by \Cref{def_etabetagamma}. Set $M_1 \leftarrow 1$.

	\For{$t=1,2,\ldots, T$}{

        Observe $X_t$;

        Compute $p_t(\cdot | X_t)$ by FTRL in~\cref{eq:FTRL_shannon_entropy 1} with regularizer $\psi_t(r)=-\frac{1}{\eta_t}H(r)$;

         Set  
         \vspace{-0.5cm}
        \begin{equation} \label{eq:FTRL_shannon_entropy}
         \pi_t(a|X_t) \leftarrow (1-\gamma_t) p_t(a|X_t)+\gamma_t \frac{1}{K};
        \end{equation}  
        
        Sample an action $A_t \sim \pi_t(\cdot |X_t)$;

        Observe the loss $\ell_t(X_t,A_t)$ and compute $\widetilde{\bm{\theta}}_{t,a}$
        for all $a \in [K]$ using~\Cref{eq:estimator_MGR};

        Update $\eta_t$ and $\gamma_t$ by \Cref{def_etabetagamma};
        
        Update $M_t$ by \Cref{def:adaptitve_iterations_MGR};

	}
 
\end{algorithm}

\spara{Loss estimation.}
Here we describe the method for estimating $\bm{\theta}_{t,a}$.
Given the covariance matrix $\bm{\Sigma}_{t,a}:=\E_t[\ind{A_t=a} X_t X_t^{\top}]$,
it is known that 
we can construct the unbiased estimator $\widehat{\bm{\theta}}_{t,a}$ defined by
\begin{equation*}
    \widehat{\bm{\theta}}_{t,a}
    :=
    \bm{\Sigma}_{t,a}^{-1} X_t \ell_t(X_t,A_t)\ind{A_t=a}, \quad \forall a \in [K].
\end{equation*} 
While this estimate is unbiased, $\E_t[ \widehat{\bm{\theta}}_{t,a}]=\bm{\theta}_{t,a}$, 
computing this estimator is computationally inefficient as its construction requires computing the inverse of the $d \times d$ covariance matrix $\bm{\Sigma}_{t,a}$.
Such a heavy computation requiring time equal to $\mathcal{O}(d^3)$ is prohibitive when $d\gg 1$.
Furthermore, this estimation approach assumes that the covariance matrix is known in advance, which is not the case in most real-world scenarios.

To avoid such practical problems,
we consider relying on the approach of \emph{Matrix Geometric Resampling} (MGR) developed by~\citet{Neu_Bartok2013,Neu_Bartok2016} and later used in \citet{neu2020efficient,Zierahn+2023}.
The MGR procedure, detailed in \Cref{appendix:MGR}, has $M_t>0$ iterations and outputs $\hatSigmaInv$ as the estimate of $\bm{\Sigma}_{t,a}^{-1}$.
MGR can be implemented in 
$\cO(M_tKd+Kd^2)$ time~\citep{neu2020efficient}.
Using $\hatSigmaInv$,
we can define the estimator of $\bm{\theta}_{t,a}$ by
\begin{equation}\label{eq:estimator_MGR}
\widetilde{\bm{\theta}}_{t,a}:=\hatSigmaInv X_t \ell_t(X_t,A_t)\ind{A_t=a}, \quad 
\forall a \in [K].
\end{equation}
However, $\hatSigmaInv$ is biased in general when $M_t > 0$ is finite,
implying that the estimator $\widetilde{\bm{\theta}}_{t,a}$ in \Cref{eq:estimator_MGR} may be biased 
(although $\E_t[\hatSigmaInv]=\bm{\Sigma}_{t,a}^{-1}$ when $M_t \to \infty$).
This biasedness needs to be handled when designing the learning rate $(\eta_t)_t$ for FTRL.

\spara{Learning rate.}
To achieve BoBW guarantees while dealing with a biased estimator,
we need to design a learning rate $\eta_t$ and a mixture rate $\gamma_t$ achieving 
$\mathcal{O}(\sqrt{T})$ regret in the adversarial regime and $\mathcal{O}(\mathrm{poly}(\log T))$ regret in the stochastic regime. 
To achieve this goal, 
we define the learning rate and mixture rate as follows:
\begin{align}\label{def_etabetagamma}
   & \beta'_{t+1}=\beta'_t + \frac{\alphaone }{\sqrt{1+ (\log K)^{-1} \sum_{s=1}^t H(p_s(\cdot | X_s))}}\,, \notag \\ 
   & \beta_t=\max\left\{2,\alphatwo  \log T, \beta'_t\right\}\,, \notag\\
   &\eta_t=\frac{1}{\beta_t}\,,
   \ \gamma_t=\alphacirct \cdot \eta_t,\ 
  \alphacirct=\frac{4K\log (t)}{\lambdaminSigma}\,,  
\end{align}
where $\alphaone=\sqrt{\left(3Kd+\frac{2 K \log T}{\lambdaminSigma }\right)\frac{\log T}{\log K}}$,  
$\alphatwo=\frac{8K}{\lambdaminSigma}$, and 
we set $\beta_1' =\alphaone  \geq 1$.
These definitions ensure $0 \leq \gamma_t \leq 1/2$ and $0 < \eta_t \leq 1/2$.

Unlike the existing algorithms, which are designed for the adversarial regime and use a fixed number of iterations of \MGR\ (i.e., $M_t=M$ for some $M>0$ at all $t \in [T]$~\citep{neu2020efficient, Zierahn+2023}),
determining $M_t$ adaptively is also crucial to prove BoBW guarantees. 
We set $M_t$ at round $t > 1$ to
\begin{equation}\label{def:adaptitve_iterations_MGR}
M_t = \left\lceil \frac{4K}{\gamma_t  \lambdaminSigma} \log (t) \right\rceil \ (\geq 1).
\end{equation}

\spara{Theoretical results.}
Here, we formally state the main result and sketch a summary of the key analysis to guarantee the regret upper bound.
The complete proof of \Cref{thm:FTRLforcontextual} and the following lemmas can be found in \Cref{appendix_forFTRL}.
\begin{theorem}\label{thm:FTRLforcontextual}
Let $c_{4}=\cO\big(\frac{K\log (K) }{\lambdaminSigma } \log (T)\big)$ be a problem-dependent constant.
The regret $R_{T}$ of \algmainFTRL\ (\Cref{alg:FTRLforcontextual}) for the adversarial regime is bounded by
\begin{equation*}
 R_{T} =
    \mathcal{O}\Bigg( 
        \sqrt{T \left(d+\frac{\log T}{\lambdaminSigma}\right)K\log(K)\log(T)} 
    +c_{4}
    \Bigg).
\end{equation*}
For the stochastic regime, the regret is bounded by
\begin{equation*}
    R_T 
    =
    \mathcal{O} \Bigg( \frac{K}{\Delta_{\min}}\left(d+\frac{\log T}{\lambdaminSigma } \right) \log (KT)\log T \Bigg)
    \eqqcolon R^{\mathrm{sto}}_{T},
\end{equation*}
and for the corrupted stochastic regime, the regret is bounded by
\begin{equation*}
    R_{T} 
    =
    \mathcal{O} \Bigg( R^{\mathrm{sto}}_{T} + \sqrt{ C  R^{\mathrm{sto}}_{T} }\Bigg).
\end{equation*}
\end{theorem}
Our bound achieves $\tilde{\cO}\left(\sqrt{T K\max\Big\{d,\frac{1}{\lambdaminSigma} \Big\}}\right)$ recovering the best-known result in the adversarial regime~\citep{neu2020efficient,Zierahn+2023} up to log-factors when $T \geq \frac{K^2}{\lambdaminSigma^2}$ and has a performance comparable to $\frac{dK}{\Delta_{\min}}\log(T)$ in the stochastic regime. In the corrupted stochastic regime, we have the desired dependence of $\sqrt{C}$ for the corruption level $C>0$.

\spara{Regret analysis.}
For the sake of simplicity, in our analysis we introduce a variant of our bandit problem that we call \emph{auxiliary game}, where the context vector $\bm{x} \in \cX$ does not change over time, and for each trial $t\in [T]$ the incurred loss is obtained replacing $\bm{\theta}_{t,a}$ by a (possibly biased) loss vector estimator $\tiltheta_{t,a}$ as follows.
Let $\tiltheta_{t,a} \in \mathbb{R}^d$ be an estimator of the loss vector $\bm{\theta}_{t,a}$ with bias $\bm{b}_{t,a} \in \mathbb{R}^d$ and $a \in [K]$.
Suppose that the learner's action $A_t$ is selected by a probabilistic policy $\pi_t(\cdot | \bm{x}) \in \simplexK$.
Then, the regret in the auxiliary game against the comparator $\pi^*(\bm{x})$ defined in \Cref{def:optimal_policy} for the estimated loss
is defined as
\begin{equation}\label{eq:auxiliary game}
\widetilde{R}_T(\bm{x}):=\E \left[\sum_{t=1}^T \la \bm{x}, \tiltheta_{t,A_t}  \ra -\la \bm{x}, \tiltheta_{t,\pi^*(\bm{x})} \ra \right].
\end{equation}
As in \citet{neu2020efficient,olkhovskaya2023first,Zierahn+2023},
we define a \emph{ghost sample} $X_0 \sim \cD$, which is drawn independently of the entire interaction history, i.e., $X_0$ is independent of any of $X_1, \ldots ,X_t$ used to construct the loss estimators $\tiltheta_{t,a}$.
With this notation, it is known that $R_T$ is bounded as follows (see Eq.(6) in \citet{neu2020efficient} and \Cref{lemma:equ6_neu2020} in \Cref{appendix:usefullemmas}):
\begin{equation*}
R_{T} \leq  \E[\widetilde{R}_T(X_0)]+2 \sum_{t=1}^{T}\max_{a \in [K]}\big|\E[\la X_t, \bm{b}_{t,a} \ra]\big|.
\end{equation*}  
Thanks to this upper bound, it suffices to bound the regret of the auxiliary game and control the bias.    
To do so, we start with \Cref{lem:regretdecom_auxiliarygame}, 
which can be proven via the standard analysis of FTRL with Shannon entropy while taking the context into account.
\begin{lemma}\label{lem:regretdecom_auxiliarygame}
Suppose that $\max_{\bm{x} \in \cX} |\eta_t \la \bm{x},\tiltheta_{t,a} \ra| \leq 1$ holds,
and $A_t$ is chosen by $\pi_t(\cdot|\bm{x})$ defined by \Cref{eq:FTRL_shannon_entropy} for $\bm{x} \in \cX$.
Then, we have
\begin{align}\label{eq:aux_bound_lem_1}
\widetilde{R}_T(\bm{x})
&\leq 
\sum_{t=1}^T \left(\beta_{t+1}-\beta_t \right) H(p_{t+1}(\cdot|\bm{x})) +\beta_1 \log K \notag \\
&\quad+\sum_{t=1}^T\eta_t \sum_{a=1}^K \pi_t(a|\bm{x}) \la \bm{x}, \widetilde{\bm{\theta}}_{t,a} \ra^2
+ U(\bm{x}),
\end{align}
where $U(\bm{x})=\sum_{t=1}^T \gamma_t \sum_{a \neq \pi^*(\bm{x})}  \frac{1}{K} \la \bm{x}, \widetilde{\bm{\theta}}_{t,a} \ra$ is the regret due to the uniform exploration.
\end{lemma}

We next state the following lemma, showing that our careful parameter tuning allows us to bound the RHS of~\eqref{eq:aux_bound_lem_1}. 
\begin{lemma}\label{lemma:parametersforMGR}
Suppose that $\eta_t \leq \frac{1}{2}$, $\gamma_t=\alphacirct \cdot \eta_t$,
and set $M_t$ as in~\eqref{def:adaptitve_iterations_MGR}.
Then, it holds that
$\textup{(i)}~|\E_t[\la X_t, \tiltheta_{t,a}-\hattheta_{t,a} \ra] | \leq  \exp{\left(- \frac{\gamma_t \lambdaminSigma M_t}{2K}  \right)} 
\leq {1}/{t^2}$ and $
\textup{(ii)}~|\eta_t \la \bm{x},\tiltheta_{t,a} \ra| \leq 1 , \quad \forall \bm{x} \in \cX$.
\end{lemma}
Thanks to $\textup{(ii)}$, the requirement of \Cref{lem:regretdecom_auxiliarygame} is met by our parameter tuning.
The statement $\textup{(i)}$ is useful to bound the penalty term caused by the biased $\widetilde{\bm{\theta}}_{t,a}$, i.e., $\E[U(X_0)]$ and $\sum_{t=1}^{T}\max_{a \in [K]}|\E[\la X_t, \bm{b}_{t,a} \ra]|$. 

From \Cref{lem:regretdecom_auxiliarygame}, we can derive Lemma \ref{keylemma:Expected_regret_auxiliarygame} providing an upper bound on the expected regret of the auxiliary game dependent on the sum of the Shannon entropy over $[T]$. 
\begin{lemma}
[Entropy-dependent regret bound for the auxiliary game]
\label{keylemma:Expected_regret_auxiliarygame}
Let $X_0 \sim \cD$ be a ghost sample drawn independently of the entire interaction history.
Let $\kappa=\alphaone  \sqrt{\log K}
+\frac{\left( 3Kd+\frac{2 K \log T}{\lambdaminSigma}\right) \log T }{\alphaone  \sqrt{\log K}}
$.
If $A_t$ is chosen by $\pi_t(\cdot |X_0)$ defined by \Cref{eq:FTRL_shannon_entropy} for $X_0$,
then, the expected regret of the auxiliary game $\E[\widetilde{R}_T(X_0)]$ is bounded by
\begin{equation*}
\mathcal{O}\left(\kappa\sqrt{ \E \left[\sum_{t=1}^T H(p_t(\cdot | X_0))\right]}
+\frac{K\log K}{\lambdaminSigma}  \log T\right).
\end{equation*}
\end{lemma}

We introduce the following notation for the further analysis:
Let $\varrho_0(\pi^*):=\sum_{t=1}^T (1-p_t(\pi^*(X_0)|X_0))$ and $\varrho_{(X_t)_{t=1}^T}(\pi^*):=\sum_{t=1}^T (1-p_t(\pi^*(X_t)|X_t))$.
Now, we are ready to sketch the proof of \Cref{thm:FTRLforcontextual}.

\spara{Proof Sketch of \Cref{thm:FTRLforcontextual}.}
For the adversarial regime, by the fact that $H(p_t(\cdot|X_0)) \leq \log K$, we immediately have the desired regret bound from the above lemmas.
To analyze the corrupted stochastic regime we start with a lower bound on the regret.
We can show that $R_T \geq \frac{\Delta_{\min}}{2}\E[\varrho_{(X_t)_{t=1}^T}(\pi^*)]-2C$ from the definition of the stochastic regime with adversarial corruption (\Cref{lem:self-boundingineq} in \Cref{appendix_forFTRL}).
For the upper bound depending on $\varrho_0(\pi^*)$, we use the inequality  of   $\sum_{t=1}^T  H(p_t(\cdot | X_0)) \leq \varrho_0(\pi^*) \log \frac{\mathrm{e}KT}{\varrho_0(\pi^*)}$  (\Cref{lemma4_ito+2022nearly} in \Cref{appendix_forFTRL}).
When $\varrho_0(\pi^*) < \mathrm{e}$, then we have the desired bound trivially from this inequality. In the case of $\varrho_0(\pi^*) \geq \mathrm{e}$,
using $\E \left[\sum_{t=1}^T  H(p_t(\cdot | X_0))\right] \leq \E[\varrho_0(\pi^*)] \log (KT)$, we have $R_T =\tilde{\cO}(\mathrm{poly}(\log T)\cdot \sqrt{\E[\varrho_0(\pi^*)]}+c_4)$, where $c_4$ is a problem-dependent constant.
Here, we use the fact that $X_0$ and $X_t$ follows the same distribution to show $\E[\varrho_{(X_t)_{t=1}^T}(\pi^*)] = \E[\varrho_0(\pi^*)]$ (\Cref{lemma:property of varrho} in \Cref{appendix_forFTRL}).
Then, the final part can be done via standard self-bounding techniques. 
Plugging the above upper and lower bound on $R_T$ into $R_T=(1+\lambda)R_T-\lambda R_T$ for $\lambda \in (0,1]$,
taking the worst-case with respect to $\E[\varrho_0(\pi^*)]$,
and then optimizing $\lambda \in (0,1]$ completes the proof for the corrupted stochastic regime.
\qed


\section{Conclusions}\label{sec:conclusions}

We proposed the first algorithms for $K$-armed linear contextual bandits to achieve the BoBW guarantees.
The first approach is to use a continuous MWU method with a reduction framework,
thereby attaining either first- or second-order regret bound in the adversarial regime and polylogarithmic regret in the (corrupted) stochastic regime.
We also designed a simpler FTRL with Shannon entropy that does not require the knowledge $\SigmaInv_{t,a}$ at each round $t$ for action $a$, and achieves the worst-case regret in the adversarial regime without sacrificing the polylogarithmic regret in the (corrupted) stochastic regime.

It is important to develop a computationally efficient algorithm that can achieve data-dependent bounds without relying on knowledge of $\SigmaInv$. Even without this knowledge, the \algmainFTRL\ algorithm achieved the optimal worst-case regret up to log factors in the adversarial regime. However, in the stochastic regime, additional $\log(T)$ and $\log(KT)$ terms arise due to \MGR\ and Shannon entropy, respectively. An additional log factor is also common when using \textsc{FTRL} with Shannon entropy in other bandit settings. Therefore, it would be interesting to explore alternative regularizers. Another direction is to extend the current results to the contextual combinatorial bandit setting.

\subsubsection*{Acknowledgements}
YK was partially supported by JST, ACT-X Grant Number JPMJAX200E.
TT was supported by JST, ACT-X Grant Number JPMJAX210E, Japan.
NCB acknowledges the financial support from the MUR PRIN grant 2022EKNE5K (Learning in Markets and Society), the NextGenerationEU program within the PNRR-PE-AI scheme (project FAIR), the EU Horizon CL4-2022-HUMAN-02 research and innovation action under grant agreement 101120237 (project ELIAS).


\bibliography{reference.bib}

 \section*{Checklist}



 \begin{enumerate}

 \item For all models and algorithms presented, check if you include:
 \begin{enumerate}
   \item A clear description of the mathematical setting, assumptions, algorithm, and/or model. [Yes]
   \item An analysis of the properties and complexity (time, space, sample size) of any algorithm. [Yes]
   \item (Optional) Anonymized source code, with specification of all dependencies, including external libraries. [Not Applicable]
 \end{enumerate}

 \item For any theoretical claim, check if you include:
 \begin{enumerate}
   \item Statements of the full set of assumptions of all theoretical results. [Yes]
   \item Complete proofs of all theoretical results. [Yes]
   \item Clear explanations of any assumptions. [Yes]     
 \end{enumerate}

 \item For all figures and tables that present empirical results, check if you include:
 \begin{enumerate}
   \item The code, data, and instructions needed to reproduce the main experimental results (either in the supplemental material or as a URL). [Not Applicable]
   \item All the training details (e.g., data splits, hyperparameters, how they were chosen). [Not Applicable]
         \item A clear definition of the specific measure or statistics and error bars (e.g., with respect to the random seed after running experiments multiple times). [Not Applicable]
         \item A description of the computing infrastructure used. (e.g., type of GPUs, internal cluster, or cloud provider). [Not Applicable]
 \end{enumerate}

 \item If you are using existing assets (e.g., code, data, models) or curating/releasing new assets, check if you include:
 \begin{enumerate}
   \item Citations of the creator If your work uses existing assets. [Not Applicable]
   \item The license information of the assets, if applicable. [Not Applicable]
   \item New assets either in the supplemental material or as a URL, if applicable. [Not Applicable]
   \item Information about consent from data providers/curators. [Not Applicable]
   \item Discussion of sensible content if applicable, e.g., personally identifiable information or offensive content. [Not Applicable]
 \end{enumerate}

 \item If you used crowdsourcing or conducted research with human subjects, check if you include:
 \begin{enumerate}
   \item The full text of instructions given to participants and screenshots. [Not Applicable]
   \item Descriptions of potential participant risks, with links to Institutional Review Board (IRB) approvals if applicable. [Not Applicable]
   \item The estimated hourly wage paid to participants and the total amount spent on participant compensation. [Not Applicable]
 \end{enumerate}

 \end{enumerate}

\clearpage

\appendix


\onecolumn
\aistatstitle{Best-of-Both-Worlds Algorithms for Linear Contextual Bandits: \\
Supplementary Materials}


\section{NOTATION}\label{appendix:notation}


In this appendix, we provide \Cref{tab:notaiton} summarizing the most important notations used in the paper.

\TT{need to fix, add notation}

\begin{table}[h]
    \centering
    \caption{Notations.}
    \label{tab:notaiton}
    \scalebox{0.8}{
    \begin{tabular}{ll}

    \toprule
     Symbol & Meaning \\
    \midrule
    $[K]:=\{1,2,\ldots,K\}$ & Finite action set\\
    $d \in \mathbb{N}$ & Dimension of loss vectors and contexts\\
    $\cX \subseteq \mathbb{R}^d$ & A context space of a full-dimensional compact set \\
    $\cD \in \Delta(\cX)$ & Context distribution over $\cX$\\
    $\mathbf{\Sigma} \in \mathbb{R}^{d \times d}$ & Covariance matrix of contexts, $\E_{X \sim \cD}[X X^{\top}]$\\ 
    $\bm{\theta}_{t,a} \in \mathbb{R}^d$ & Loss vector of action $a \in [K]$ at round $t \in [T]$\\
    $\bm{\theta}_{a} \in \mathbb{R}^d$  & Fixed and unknown vectors of action $a \in [K]$ at round $t \in [T]$ (corrupted and stochastic regime)\\ 
    $C \in [0,T]$ & Corruption level, upper bound of $\sum_{t=1}^T \max_{a \in [K]} \|\bm{\theta}_{t,a}-\bm{\theta}_a \|_2$\\
    \midrule 
    $\pi(\cdot\mid\bm{x}) \in \simplexK$ & Probabilistic policy mapping each context $\bm{x}$ to a probability distribution \\
    $\Pi$  &  Set of all deterministic policies $\pi : \cX \rightarrow [K]$\\
    $\pi^* \in \Pi$ & Optimal policy \\
    $\Delta_{\min}>0$ & Minimum sub-optimal gap over a context space, $\min_{\bm{x} \in \cX} \min_{a \neq \pi^*(\bm{x})} \la \bm{x},\bm{\theta}_{a}-\bm{\theta}_{\pi^*(\bm{x})}\ra$\\
    \midrule 
    $\bm{m}_{t,a} \in \mathbb{R}^d$  & Loss predictor for action $a \in [K]$ and $t \in [T]$\\
    $L^*$&  $\E\big[\sum_{t=1}^T\ell_t(X_t, \pi^*(X_t)) \big]$\\
    $\Lambda^*$ & $\E\big[\sum_{t=1}^T (\ell_t(X_t,A_t)-\la X_t,\bm{m}_{t,A_t} \ra)^2 \big]$ \\ 
    $\bar{\Lambda}$ & $\E\big[\sum_{t=1}^T (\ell_t(X_t,A_t)-\la X_t, \bar{\bm{\theta}} \ra)^2 \big]$ with $\bar{\bm{\theta}}:=\frac{1}{TK}\sum_{t=1}^T \sum_{a=1 }^K\bm{\theta}_{t,a}$. \\
     $\xi_{t,a} \in \mathbb{R}$& $(\ell_t(X_t,a)-\la X_t,\bm{m}_{t,a} \ra)$ with a loss predictor $\bm{m}_{t,a}$ for for action $a \in [K]$ and $t \in [T]$ \\
     $\hattheta_{t,a} \in \mathbb{R}^d$ & Unbiased estimator for $\bm{\theta}_{t,a}$ for $a \in [K]$ and $t \in [T]$\\
    $\widehat{\bell}_{s}(X_t) \in  \mathbb{R}^K$ & Estimated loss vector for $X_t$ at round $t \in [T]$, $\big( \la X_t, \hat{\bm{\theta}}_{s,1} \ra, \ldots,  \la X_t, \hat{\bm{\theta}}_{s,K} \ra \big)$\\
     $\bm{m}_t(X_t) \in  \mathbb{R}^K$ & Predicted loss vector for $X_t$ at round $t \in [T]$, $\big( \la X_t, \bm{m}_{t,1} \ra,\ldots,\la X_t,\bm{m}_{t,K}\ra \big)$\\
    $\widehat{R}_T(\bm{x})$ & Regret of auxiliary game for context $\bm{x}$ and unbiased loss estimator $\hattheta_{t,a}$ at round $t$, $\E \left[\sum_{t=1}^T \la \bm{x}, \hattheta_{t,A_t}  \ra -\la \bm{x}, \hattheta_{t,\pi^*(\bm{x})} \ra \right]$\\
    \midrule 
    $\tiltheta_{t,a} \in \mathbb{R}^d$ & Biased estimator for $\bm{\theta}_{t,a}$ for $a \in [K]$ and $t \in [T]$\\
      $\widetilde{\bell}_{s}(X_t) \in  \mathbb{R}^K$ & Estimated loss vector for $X_t$ at round $t \in [T]$, $\big( \la X_t, \tiltheta_{s,1} \ra, \ldots,  \la X_t, \tiltheta_{s,K} \ra \big)$\\  
    $\widetilde{R}_T(\bm{x})$ & Regret of auxiliary game for context $\bm{x}$ and loss estimator $\tiltheta_{t,a}$ at round $t$, $\E \left[\sum_{t=1}^T \la \bm{x}, \tiltheta_{t,A_t}  \ra -\la \bm{x}, \tiltheta_{t,\pi^*(\bm{x})} \ra \right]$\\
    
    \bottomrule
    \end{tabular}
    }
\end{table}

\section{ADDITIONAL RELATED WORK}\label{appendix:related work}
There is another line of research dedicated to studying the problem of model selection. A few notable works in this area include \citet{Pacchiano+2020stochastic,pacchiano2022best,agarwalcorral17b,cutkosky21a,Lee+2021,Wei+22}.
Among these, \citet{pacchiano2022best} addressed the general contextual linear bandit problem with a nested policy class. They achieved the first high probability regret bound, recovering the result of \citet{agarwalcorral17b} in the adversarial regime, and attained a gap-dependent bound in the stochastic regime.
They also showed a lower bound for the stochastic regime, indicating that a perfect model selection among $m$ logarithmic rate learners is impossible. Formally, this implies that the optimal dependence of the complexity parameter for the largest policy class cannot be improved over a quadratic, i.e., $\frac{R(\Pi_m)^2 \log T}{\Delta_{\min}}$, where $R(\Pi_m)$ is the complexity parameter for the largest policy class.
In their best-of-both-worlds model selection algorithm, the base learners aggregated by the meta-algorithm are required to satisfy anytime high-probability regret guarantees in the adversarial regime, along with notions of high probability stability and action space extendability.
Although a high-probability variant of \textsc{Exp4} of~\citet{auer2002nonstochastic} could be a viable option as a base learner to meet these requirements,
its running time, however, is generally linear in the number of policies.
 This makes it intractable for an infinite policy class of $\pi: \cX \rightarrow [K]$, where $\cX \subseteq \mathbb{R}^d$.
Leaving aside the computational issues, \citet{pacchiano2022best} have not addressed data-dependent bounds in the adversarial regime, nor have the corrupted regime been explicitly investigated.

Since \citet{Lykouris+2018} first proposed the stochastic $K$-armed bandits with adversarial corruptions,
different problem settings including contextual bandits, have been well-studied in the literature.
\citet{zhao2021linear,Ding+22c, He+2o22} extended the model studied in \citet{abbasi2011improved} under the corruption framework by \citet{Lykouris+2018} for the linear contextual bandits.
For further extensions, \citet{Bogunovic20+Gaussian} introduced the kernelized MAB problem. \citet{Ye+2023} recently studied nonlinear contextual bandits and Markov Decision Processes, and \cite{kang2023robust} introduced Lipschitz bandits in the presence of adversarial corruptions.
We also mention a few works of \citet{Jun+2018,Liu+2019, Garcelon+2020, Bogunovic_linear21} in this line of research that studied a different adversary model, where the adversary may add the corruption after observing the learner's action $A_t$.
\citet{Garcelon+2020} examined several attack scenarios and showed that a malicious adversary could manipulate a linear contextual bandit algorithm for the adversary's benefit.
It is also notable that regret can be defined in different ways, taking into account losses after corruption or losses without corruption. However, the difference between the two definitions is negligible, at most $O(C)$, where $C$ is the corruption level. For a more detailed discussion on these different notions of regret, refer to \cite{Gupta19+,Ito2021Hybrid}.

Algorithms for linear contextual bandits that provide regret guarantees have been developed with various assumptions on the losses and contexts. The stochastic linear contextual bandit is the most extensively studied model among them. Here, the context in each round can be arbitrarily generated while an unknown loss (reward) vector is fixed over time~\citep{Chu+11,abbasi2011improved, Li+19}.
Efficient computational techniques have also been developed to take advantage of the availability of a regression oracle~\citep{Foster+18practical}. \citet{Foster+2020} studied the misspecified linear contextual bandit problem for infinite actions with an online regression oracle. In addition, \citet{Foster+20abeyonducb} extended oracle-based algorithms for a general function class.

Despite the rich history of contextual bandits literature we described above,
few results have been known for data-dependent bounds as 
 the question was posed by \citet{pmlr-v65-agarwal17a}. \citet{pmlr-v80-allen-zhu18b} first affirmatively solved this question for adversarial losses and contexts. However, their algorithm only works for a moderate number of policies. \citet{foster2021firstorder} provided the first optimal and efficient reduction from contextual bandits to online regression with the cross-entropy loss, thereby achieving a first-order regret guarantee, but the loss function is assumed to be fixed over time.
The work of \citet{olkhovskaya2023first} first achieved the first- and second-order bounds for adversarial losses and i.i.d contexts case. The critical difference between the above-mentioned work and our study is that these have not investigated the BoBW guarantee.

\section{LOWER BOUND}\label{appendix:lowerbound}

An algorithm is said to be \emph{orthogonal} if it does not use the information from rounds in which $X_t \neq X_s$ for $s<t$ to make a prediction at round $t$~\citep{Zierahn+2023}.
For the class of orthogonal algorithms, \citet{Zierahn+2023} proved the following regret lower bound for the combinatorial full-bandit setting in the adversarial regime.
In the combinatorial full-bandit setting, the action space satisfies $\cA \subseteq \{0,1\}^K$ and $\max_{a \in \cA}\|a\|_1 \leq S$.

\begin{proposition}[Theorem 19 in \citet{Zierahn+2023}]
Suppose $T \geq d S K$ and $K \geq 2S$.
In the combinatorial full-bandit setting, any orthogonal algorithm satisfies
\[
R_T \geq \frac{S^{3/2}\sqrt{dKT}}{16(192+96 \ln(T))}.
\]

\end{proposition}
In their proof of the lower bound, they construct the $S$ instances of $n$-armed bandit problems for $n=\frac{K}{S} \in \mathbb{N}$.
Therefore, the statement for $S=1$ implies the lower bound for the $K$-armed contextual bandit case:
\[
R_T=\Omega(\sqrt{dKT}),
\]
which we are interested in.
Also note that both \algmainFTRL\ and \algmainConEW\ with $\bm{m}_{t,a}  =0$ are orthogonal, as formally stated in the following lemmas. 

\begin{lemma}\label{lem:orthogonal0}
    Suppose that $\cX$ consist of only basis vectors  i.e., $\cX=\{\bm{e}_1, \dots, \bm{e}_K\}$, and pick some $t \in [T]$. Let $X_{t'} \neq X_t$ and let $a \in [K]$.
    Then, $\la X_{t'}, \widetilde{\bm{\theta}}_{t,a} \ra=0$ holds for the biased estimator $\widetilde{\bm{\theta}}_{t,a}$ in \Cref{eq:estimator_MGR} of \algmainFTRL,
    and $\la X_{t'},  \widehat{\bm{\theta}}_{t,a} \ra =0$ holds for the unbiased estimator $\widehat{\bm{\theta}}_{t,a}$ with $\bm{m}_{t,a}=\bm{0}$ in \Cref{eq:unbiased_linearcon} of \algmainConEW. 
\end{lemma}

\begin{proof}[Proof of \Cref{lem:orthogonal0}]
We follow the proof of Lemma 17 in \citet{Zierahn+2023}.
First consider $\widetilde{\bm{\theta}}_{t,a}$ in \Cref{eq:estimator_MGR}.
Let $\hatSigma_{t,a}^+$ be a sample of the {\MGR} (\Cref{alg:MGR}) with $M$-iteration and it can be written as
\begin{align*}
   \hatSigma_{t,a}^+=\rho \sum_{k=0}^M \prod_{j=1}^k (\mathbf{I}-\rho \mathbf{B}_{k,a}).
\end{align*}
Notice that $\mathbf{B}_{k,a}=\ind{A(k)=a} X(k)X(k)^{\top}$ is diagonal since $\cD$ has the support of $\cX=\{\bm{e}_1, \dots, \bm{e}_K\}$.
So as $\hatSigma_{t,a}^+$ for all $a \in [K]$.
Let $X_{t'}=\bm{e}_i$ and $X_t=\bm{e}_j$ and pick $a \in [K]$.
Then  we see that 
\begin{align*}
  \la X_{t'}, \widetilde{\bm{\theta}}_{t,a} \ra&=\bm{e}_i^{\top} \widetilde{\bm{\theta}}_{t,a}\\
  &=\bm{e}_i^{\top}\hatSigmaInv X_t \ell_t(X_t,A_t)\ind{A_t=a}\\
  &=\bm{e}_i^{\top}\hatSigmaInv \bm{e}_j \ell_t(X_t,A_t)\ind{A_t=a}\\
  &=(\hatSigmaInv)_{i,j}\ell_t(X_t,A_t)\ind{A_t=a},
\end{align*}
concluding that $\la X_{t'}, \widetilde{\bm{\theta}}_{t,a} \ra=0$ if $i \neq j$.

Next, we consider  $\widehat{\bm{\theta}}_{t,a}$ in \Cref{eq:unbiased_linearcon}, where $\tilSigmaInv_{t,a}$ is given by \Cref{eq: def tilde Sigma} and $\xi_{t,a}=(\ell_t(X_t,a)-\la X_t,\bm{m}_{t,a} \ra)$.
By a similar discussion, we have
\begin{align*}
  \la X_{t'},  \widehat{\bm{\theta}}_{t,a} \ra&=\bm{e}_i^{\top}    \widehat{\bm{\theta}}_{t,a}\\
  &=\bm{e}_i^{\top} \big(\bm{m}_{t,a}+ \frac{\upd_t}{q_t}  
 Q_t(a)\tilSigmaInv_{t,a} X_t
 \xi_{t,a} \ind{A_t=a} \big)\\
  &=\bm{e}_i^{\top} \big(\bm{m}_{t,a}+ \frac{\upd_t}{q_t}  
 Q_t(a)\tilSigmaInv_{t,a} \bm{e}_j
 \xi_{t,a} \ind{A_t=a} \big)\\
  &= \bm{m}_{t,a}(i)+ (\tilSigmaInv_{t,a})_{i,j} \frac{\upd_t}{q_t}  
 Q_t(a)
 \xi_{t,a} \ind{A_t=a}.
\end{align*}
Therefore, we conclude that $\la X_{t'},  \widehat{\bm{\theta}}_{t,a} \ra =0$ if $i \neq j$ and $\bm{m}_{t,a}=\bm{0}$,
since $\tilSigma_{t,a}$ is diagonal in this case.

\end{proof}

\begin{lemma}\label{lem:orthogonal}
 Suppose that $\cX$ consist of only basis vectors i.e., $\cX=\{\bm{e}_1, \dots, \bm{e}_K\}$.
 Also, suppose that in round $t$,
 the context is a basis vector in the direction $i \in [K]$.
 Then, in {\algmainFTRL} and {\algmainConEW}  with $\bm{m}_{t,a}=\bm{0}$ for each $a \in [K]$,
 the observation obtained in round $t$ does not affect the algorithm's prediction in all subsequent rounds such that the context is a basis vector in direction $j \neq i$.
\end{lemma}
\begin{proof}[Proof of \Cref{lem:orthogonal}]
    Lemma~\ref{lem:orthogonal0} implies that when $X_t=\bm{e}_i$, then $p_t(\cdot | X_t)$ in \Cref{eq:FTRL_shannon_entropy} in {\algmainFTRL} can be written as
\[
p_t(\cdot |\bm{e}_i)=\argmin_{r\in \simplexK} \left\{ \sum_{s<t:X_s=\bm{e}_i}\la r, \widetilde{\bell}_{s} (\bm{e}_i)  \ra  +\psi_t(r) \right\},
\]
where $\widetilde{\bell}_{s}(\bm{e}_i:=( 
 \la \bm{e}_i, \tiltheta_{s,1} \ra, \ldots,  \la \bm{e}_i, \tiltheta_{s,K} \ra)^{\top} \in \mathbb{R}^K$.
Also, we can write $w_t(r|X_t)$ for $r\in \simplexK$ in \Cref{dist:lincon_continuous} of {\algmainConEW} as
\[
w_t(r| \bm{e}_i) =\exp{\left(-\eta_t \sum_{a \in [K]} r_a\left\la \bm{e}_i, \sum_{s<t: X_s=\bm{e}_i}\hattheta_{s,a} + \bm{m}_{t,a} \right\ra  \right)},
\]
where $\bm{m}_{t,a}=\bm{0}$.
These equations mean that both algorithms do not use the information at round $s<t$ wherein $X_t \neq X_s$.
\end{proof}

\section{USEFUL LEMMAS}\label{appendix:usefullemmas}


This section presents some known results from existing literature, such as basic regret bounds in FTRL and basic regret decompositions often used for $K$-armed linear contextual bandits.

\subsection{Analysis of FTRL}

We introduce a standard FTRL analysis (e.g.~Exercise 28.12 of \citealt{Lattimore+2020book}) when it is applied to $K$-armed linear contextual bandits with a fixed context $\bm{x} \in \cX$.
The following \Cref{lemma:FTRLanalysis_lincon} will be used to analyze the regret of the auxiliary game given by \Cref{eq:aux_bound_lem_1} in \Cref{lem:regretdecom_auxiliarygame}.

The Bregman divergence from $p \in \simplexK$ to $q \in \simplexK$ is defined as 
\[
D_t(q,p)=\psi_t(q)-\psi_t(p)-\la \nabla \psi_t(q), q-p\ra.
\]

\begin{lemma}
\label{lemma:FTRLanalysis_lincon}
Let $p_t(\cdot|\bm{x})$ be a FTRL prediction with loss estimators $\tiltheta_{t,a}$ for each $a \in [K]$, which is given by \Cref{eq:FTRL_shannon_entropy 1} with any convex regularizer $\psi_{t}(\cdot)$.
Suppose that $A_t$ is chosen by $\pi_t(\cdot |\bm{x}):= (1-\gamma_t)p_t(\cdot|\bm{x})+\gamma_t \frac{1}{K}$, where  $\gamma_t \in [0,1]$ is the mixture rate.
Then, for any context $\bm{x} \in \cX$, we have 
\begin{alignat*}{4}
&\E_{A_t} \left[\sum_{t=1}^T \left( \la \bm{x}, \tiltheta_{t,A_t}  \ra -\la \bm{x}, \tiltheta_{t,\pi^*(\bm{x})} \ra \right) \right]\\
&\leq
\sum_{t=1}^T \left(\psi_t(p_{t+1}(\cdot|\bm{x}))-\psi_{t+1}(p_{t+1}(\cdot|\bm{x})) \right) + \psi_{T+1}(\pi^*(\cdot|\bm{x}))-\psi_{1}(p_{1}(\cdot |\bm{x})) \notag\\
&\quad+\sum_{t=1}^{T} (1-\gamma_t)\left(\left\la p_t(\cdot |\bm{x})-p_{t+1}( \cdot |\bm{x})) , \tilde{\bell}_t(\bm{x})  \right\ra  - D_t(p_{t+1}(\cdot|\bm{x}), p_t(\cdot|\bm{x})) \right) 
   +U(\bm{x}),
\end{alignat*}
where $U(\bm{x})= \sum_{t=1}^T \gamma_t \left\la \frac{1}{K} \bm{1}-\pi^*( \cdot |\bm{x}) , \tilde{\bell}_t(\bm{x})  \right\ra $, and $\pi^*(a|\bm{x})=1$ if $a=\pi^*(\bm{x})$ otherwise 0.
\end{lemma}

\begin{proof}[Proof of Lemma~\ref{lemma:FTRLanalysis_lincon}]
From the definition of the auxiliary game and the design of the algorithm, for any $\bm{x} \in \cX$, we have
\begin{alignat*}{4}
\E_{A_t} \left[\sum_{t=1}^T \left( \la \bm{x}, \tiltheta_{t,A_t}  \ra -\la \bm{x}, \tiltheta_{t,\pi^*(\bm{x})} \ra \right) \right]
&=\sum_{t=1}^T \sum_{a \in [K]} ( \pi_t(a|\bm{x})-\pi^*(a|\bm{x})) \la \bm{x}, \tiltheta_{t,a} \ra\\
&= \sum_{t=1}^T  (1-\gamma_t) \sum_{a \in [K]} (p_t(a|\bm{x})-\pi^*(a|\bm{x})) \la \bm{x}, \tiltheta_{t,a} \ra +  \sum_{t=1}^T \gamma_t \left\la \frac{1}{K} \bm{1}-\pi^*( \cdot |\bm{x}) , \tilde{\bell}_t(\bm{x})  \right\ra \\
&= \sum_{t=1}^T (1-\gamma_t) \sum_{a \in [K]} (p_t(a|\bm{x})-\pi^*(a|\bm{x})) \la \bm{x}, \tiltheta_{t,a} \ra + U(\bm{x}).
\end{alignat*}
By the standard analysis of FTRL (see, e.g., Exercise 28.12 of \citealt{Lattimore+2020book}),
the first term in the RHS above is further bounded as 
\begin{alignat*}{4}
    &\sum_{t=1}^{T} (1-\gamma_t) \sum_{a \in [K]}(p_t(a|\bm{x})-\pi^*(a|\bm{x}))  \la \bm{x}, \tiltheta_{t,a} \ra\\
   & \leq \sum_{t=1}^{T}(1-\gamma_t)\left( \left\la p_t(\cdot |\bm{x})-p_{t+1}( \cdot |\bm{x})) , \tilde{\bell}_t(\bm{x})  \right\ra - D_t(p_{t+1}(\cdot|\bm{x}), p_t(\cdot|\bm{x}))\right) \notag \\
   &+\sum_{t=1}^{T} 
   (\psi_t(p_{t+1}(\cdot|\bm{x}))-\psi_{t+1}(p_{t+1}(\cdot|\bm{x})))  
   +\psi_{T+1}(\pi^*(\cdot |\bm{x}))-\psi_{1}(p_{1}(\cdot |\bm{x})).   
\end{alignat*}
Combining the above arguments completes the proof.
\end{proof}

\subsection{Fundamental bounds for $K$-armed linear contextual bandits}

First, we introduce a fundamental regret decomposition using the auxiliary game in \Cref{eq:auxiliary game}.
\begin{lemma}[c.f. Equation (6) of \citet{neu2020efficient}] \label{lemma:equ6_neu2020}
Let $X_0 \sim \cD$ be a ghost sample drawn independently from the entire interaction history. Then we have
      \begin{align*}
R_{\tau} \leq  \E[\widetilde{R}_\tau(X_0)]+2 \sum_{t=1}^{\tau}\max_{a \in [K]}|\E[\la X_t, \bm{b}_{t,a} \ra] |
     \end{align*} 
\end{lemma}

Next, we introduce the following lemma for analysis related to a ghost sample $X_0$, which will be used to prove \Cref{propo:linexp3_iwstability} and \Cref{keylemma:Expected_regret_auxiliarygame}.
\begin{lemma}[c.f. Lemma 6 in \citet{neu2020efficient}]\label{lemma:Lemma6_neu}
Let $X_0 \sim \cD$ be a ghost sample drawn independently from the entire interaction history. Suppose that $X_t$ is satisfying $\|X_t\|_2 \leq 1$, and $0 < \rho \leq \frac{1}{2}$. Then, for any time step $t$ and an estimator $\widetilde{\bm{\theta}}_{t,a}$,
we have
      \begin{align}
\E_t \left[  \sum_{a=1}^K \pi_t(a|X_0) \la X_0, \widetilde{\bm{\theta}}_{t,a} \ra^2 \right] \leq  3Kd.
     \end{align}    
\end{lemma}

Lastly, we introduce the following lemma, which will be used to prove \Cref{lemma:parametersforMGR} to control the biased term caused by \MGR\ procedure.
\begin{lemma}[c.f. Lemma 5 in \cite{neu2020efficient}]\label{lemma:biased_term_MGR}
Let $\widehat{\bm{\theta}}_{t,a}=\SigmaInv_{t,a} X_t \ell_t(X_t,A_t)\ind{A_t=a} $ for all $a \in [K]$,
 and let  $\widetilde{\bm{\theta}}_{t,a}=\hat{\mathbf{\Sigma}}_{t,a}^{+} X_t \ell_t(X_t,A_t)\ind{A_t=a} $ for all $a \in [K]$ where $\hat{\mathbf{\Sigma}}_{t,a}^{+}$ is obtained by \MGR\ with $\rho=\tfrac12$ of \Cref{alg:MGR}.
Then, we have
      \begin{align*}
\big|\E[\la X_t, \widetilde{\bm{\theta}}_{t,a}-\widehat{\bm{\theta}}_{t,a} \ra  |  \cF_{t-1}] \big|
\leq \exp{\left(- \frac{\gamma_t\lambdaminSigma}{2K}M_t  \right)}.
     \end{align*}   
\end{lemma}

\section{APPENDIX FOR REDUCTION APPROACH}\label{appendix:dann+2023}

We summarize the known results of the black-box reduction framework of \citet{Dann+2023}, when it is adapted to our $K$-armed linear contextual bandit problem, although \citet{Dann+2023} provided for several other different problem settings.
Then, as a naive adaption of \citet{Dann+2023},  we describe a base algorithm for $K$-armed linear contextual bandits with adaptive learning rates and provide its analysis, resulting in Proposition~\ref{propsi:linexp3_iw_bounds}.
 For notational convenience, we use $R(\tau,a^*)$ to denote the pseudo-regret of $\E\left[ \sum_{t=1}^{\tau }\ell_{t}(X_t,A_t)-\ell_t(X_t,a^*) \right]$ for round $\tau \in [1,T]$ and comparator action $a^* \in [K]$ fixed in hindsight.
All the pseudo-codes of reduction algorithms are also detailed in this appendix to make the paper self-contained.

\begin{algorithm}
\caption{BoBW via local-self-bounding (LSB) algorithm, Adaption of Algorithm 1 in \citet{Dann+2023}}
\label{alg:LSBtoBOBW}
 	\SetKwInOut{Input}{Input}
 	\SetKwInOut{Output}{Output}
	\Input{LSB algorithm $\cL$}
	
     $T_1 \leftarrow 0;~~~ T_0 \leftarrow -c_2 \log T$;

     $\hat{A}_1 \sim \boldsymbol{\mathrm{unif}}([K])$, $t \leftarrow 1$;

	\For{$k=1,2,\ldots$}{

        Initialize $\cL$ with candidate action $\hat{A}_k$;
        
        Set the number of pulls $N_k(a)$ for all $a \in [K]$;
        
 	      \For{$t=T_k+1, T_k+2, \ldots$}{ 

            Observe $X_t$;
            
            Choose action $A_t$ according to $\cL$, and advance $\cL$ by one step;

            $N_k(a_t) \leftarrow N_k(a_t)+1$;

            \If{$t-T_k \geq 2(T_k-T_{k-1})$ and $\exists a \in [K]\setminus \{\hat{A}_k\}$ such that $N_k(a) \geq \frac{t-T_k}{2}$}{
                $\hat{A}_{k+1} \leftarrow a$;
                
                $T_{k+1} \leftarrow t$;

                \textbf{break}
                }  
            }
	}
\end{algorithm}
\begin{algorithm}
\caption{LSB via Corral, Adaption of Algorithm 2 in \citet{Dann+2023}}
\label{alg:LSBviaCorral}
 	\SetKwInOut{Input}{Input}
 	\SetKwInOut{Output}{Output}
	\Input{candidate action $\hat{a} \in [K]$, $\frac{1}{2}$-iw-stable algorithm $\cB$ over $[K] \setminus \{ \hat{a}\}$ with constants $c_1$ and $c_2$}

    \textbf{Define}: $\psi_t(q)= -\frac{2}{\eta_t} \sum_{i=1}^2 \sqrt{q_i} + \frac{1}{\beta} \sum_{i=1}^2 \ln{\frac{1}{q_i}}$

    $B_0=0$;

	\For{$t=1,2,\ldots$}{

        Observe $X_t$;

        Compute
        \begin{equation*}
            \bar{q}_t \leftarrow  \argmin_{q \in \Delta([2])} \left\{ \left<q, \sum_{\tau=1}^{t-1} z_{\tau}-  \begin{bmatrix}
                0  \\
                B_{t-1}\\
            \end{bmatrix}+\psi_t(q)
            \right> \right\}, \ q_t \leftarrow \left(1-\frac{1}{2t^2} \right)\bar{q}_t + \frac{1}{4t^2} \bm{1}
        \end{equation*} 
        with $\eta_t\leftarrow \frac{1}{\sqrt{t}+8\sqrt{c_1}}, \ \beta=\frac{1}{8c_2}$;

        Sample $i_t \sim q_t$;
        
         \If{$i_t=1$}{
         Choose $A_t=\hat{a}$ and observe $\ell_t(X_t,A_t)$;
          }\Else{
            Choose $A_t$ according to base algorithm $\cB$ and observe $\ell_t(X_t,A_t)$;
            }

        Define $z_{t,i}\leftarrow  \frac{(\ell_t(X_t,A_t)+1)\ind{i_t=i}}{q_{t,i}}-1$ and
            $B_t\leftarrow \sqrt{c_1 \sum_{\tau=1}^t \frac{1}{q_{\tau,2}}} +\frac{c_2}{\min_{\tau \leq t} q_{\tau,2}}$;
         
        }

\end{algorithm}

\subsection{Zero-order bound via reduction framework}\label{sec:Zero-order bound via reduction framework}

 Inspired by the techniques of model selections,
the reduction approach of \cite{Dann+2023} relies on an algorithm satisfying the following condition, called \emph{$\alpha$-local-self-bounding condition} (LSB).
\begin{definition}[$\alpha$-local-self-bounding condition or $\alpha$-LSB, Adaption of Definition 4 of~\citep{Dann+2023}]
We say an algorithm satisfies the $\alpha$-local-self-bounding condition if it takes a candidate action $\hat{a} \in [K]$ as input and has the following pseudo-regret guarantee for any stopping time $\tau \in [1, T]$ and  for any $a^* \in [K]$:
\begin{align}
    R(\tau,a^*) 
    \leq \min\left\{ c_0^{1-\alpha} \E[\tau]^{\alpha}, (c_1 \log T)^{1-\alpha} \E \left[ \sum_{t=1}^{\tau} (1- \ind{a^*=\hat{a}}p_t(a^*|X_t)) \right]^{\alpha}   \right\}
    +c_2 \log T
    ,
\end{align}
where $c_0, c_1, c_2$ are problem dependent constants and $p_t(a^*|X_t)$ is the probability choosing $a^*$ at round $t$.
\end{definition}

For a reduction procedure, detailed in Algorithm~\ref{alg:LSBtoBOBW}, that turns any LSB algorithm into a best-of-both-world algorithm, its BoBW guarantees are stated in the following proposition. 
\begin{proposition}[Adaption of Theorem 6 of~\citet{Dann+2023}]\label{proposi:theorem6_Dann+2023}
If an algorithm $\cL$ satisfies $\alpha$-LSB with $(c_0,c_1,c_2)$, then the regret of Algorithm~\ref{alg:LSBtoBOBW} with $\cL$  as the base algorithm is upper bounded by $\cO(c_0^{1-\alpha} T^{\alpha}+c_2 \log^2(T))$ in the adversarial regime and by $\cO(c_1 \log(T)\Delta_{\min}^{-\frac{\alpha}{1-\alpha}}+(c_1 \log T)^{1-\alpha} (C \Delta_{\min}^{-1})^{\alpha} + c_2 \log(T)\log(C \Delta_{\min}^{-1}))$ in the corrupted stochastic regime. 
\end{proposition}

Since algorithms satisfying the LSB condition are not common, \citet{Dann+2023} further
introduced the notion of the \emph{importance-weighting stability} (iw-stable), and presented a variant of Corral algorithm (Algorithm~\ref{alg:LSBviaCorral})~\citep{agarwalcorral17b} that runs over a candidate action $\hat{a}$ and an importance-weighting stable algorithm $\cB$ over the action set $[K]\setminus \{\hat{a}\}$.

\begin{definition}[iw-stable, Adaption of Definition 8 of \citet{Dann+2023}]\label{def:iw-stable}
Given an adaptive sequence of weights $q_1, q_2, \ldots \in (0,1]$, suppose that the feedback in round $t$ is observed with probability $q_t$.
Then, an algorithm is $\frac{1}{2}$-importance-weighting stable if it obtains the following pseudo-regret guarantee for any stopping time $\tau \in [1,T]$ and any $a^* \in [K] $:
\begin{align}
   R(\tau,a^*) \leq 
   \E\left[ \sqrt{c_1 \sum_{t=1}^{\tau} \frac{1}{q_t}} +\frac{c_2}{\min_{t \leq \tau} q_t}   
   \right].
\end{align}
\end{definition}

\begin{proposition}[Theorem 11 of~\citet{Dann+2023}]\label{proposi:theorem11_Dann+2023}
If an algorithm $\cB$ is $\frac{1}{2}$-iw-stable with constant $(c_1, c_2)$,
then Algorithm~\ref{alg:LSBviaCorral} with $\cB$ as the base algorithm satisfies $\frac{1}{2}$-LSB with constants $(\bar{c}_0, \bar{c}_1, \bar{c}_2)$, where $\bar{c}_0=\bar{c}_1=\cO(c_1)$ and $\bar{c}_2=\cO(c_2)$.
\end{proposition}

\subsection{First- and second-order bounds via reduction framework}\label{subsec:appendix dd reduction framework}

Next, we introduce a  reduction scheme that can also be adapted to obtain a data-dependent bound relying on a notion of \emph{data-dependent local self-bounding} (dd-LSB)~\citep{Dann+2023}, when it is applied to our setting.
In order to make the paper self-contained, we detail the pseudo-code of a Corral algorithm (Algorithm 6 of \citet{Dann+2023}) in \Cref{alg:dd LSBviaCorral}.
\begin{definition}[dd-LSB, Definition 20 of \citet{Dann+2023}]
    An algorithm is said to be dd-LSB (data-dependent LSB)
if it takes a candidate action $\hat{a} \in \cA$ as input and satisfies the following pseudo-regret guarantee  for any stopping time ${\tau}\in [1,T]$ and action $a^* \in [K]$,
\begin{align*}
  R(\tau,a^*) \leq 
   \sqrt{ c_1 \ln(T) \E \left[ \sum_{t=1}^{\tau} \bigg( \sum_{a \in [K]}(p_{t}(a|X_t) \xi^2_{t,a}- \ind{a^*=\hat{a}} p_{t}(a^*|X_t)^2 \xi^2_{t,a^*}) \bigg) \right]} +c_2 \log T
\end{align*}
where $c_1, c_2$ are problem-dependent constants and $p_t(a^*|X_t)$ is the probability choosing $a^*$ at round $t$.
\end{definition}

The performance of an algorithm with dd-LSB condition is guaranteed as the following proposition. 
\begin{proposition}[Theorem 23 of~\citet{Dann+2023}]\label{proposi:theorem23ofDann+}
If an algorithm $\cL$ satisfies dd-LSB, then the regret of Algorithm~\ref{alg:LSBtoBOBW} with $\cL$  as the base algorithm is upper bounded by $\cO(\sqrt{c_1 \E \left[\sum_{t=1}^T \xi^2_{t,A_t} \right] \log^2 T}  +c_2 \log^2(T))$ in the adversarial regime and by $\cO\left(\frac{c_1 \log(T)}{\Delta_{\min}}+  \sqrt{ \frac{c_1 \log T C}{\Delta_{\min}}} + c_2 \log(T)\log(C \Delta_{\min}^{-1})\right)$ in the corrupted stochastic regime. 
\end{proposition}

To achieve the dd-LSB condition, \citet{Dann+2023} also proposed a variant of \emph{Corral algorithm} of \citet{agarwalcorral17b}, which is detailed in Algorithm~\ref{alg:dd LSBviaCorral}.
This Corral algorithm is run over two base algorithms with refined weights $(q_t)$: one is to play the current candidate action $\hat{a}$ and the other is an algorithm with the \emph{data-dependent-importance-weighting-stable} (dd-iw-stable) condition over the action set of $\cA \setminus  \{\hat{a}\}$, given in \Cref{def:ddiw-stable}.
It is guaranteed that the Corral algorithm (\Cref{alg:dd LSBviaCorral}) satisfies the dd-LSB condition when a base algorithm is dd-iw-stable, formally stated in \Cref{proposi:theorem22ofDann+}.

\begin{definition}\label{def:ddiw-stable}
[dd-iw-stable, Adaption of Definition 21 of \citet{Dann+2023}]
Given an adaptive sequence of weights $q_1, q_2, \ldots \in (0,1]$, suppose that the feedback in round $t$ is observed with probability $q_t$.
Then, an algorithm is $\frac{1}{2}$-dd-iw-stable (data-dependent-iw-stable)
if it satisfies the following pseudo-regret guarantee for any stopping time $\tau \in [1,T]$ and for any $a^* \in [K]$:
\begin{align*}
  R(\tau,a^*) \leq 
  \sqrt{c_1   \E\left[ \sum_{t=1}^{\tau} \frac{\upd_t \cdot \xi^2_{t,A_t}}{q_t^2}\right]} +\E\left[\frac{c_2}{\min_{t \leq \tau} q_t} \right],
\end{align*}
where $\upd_t=1$ if feedback is observed in round $t$ and 
$\upd_t=0$ otherwise.
\end{definition}

\begin{proposition}[Theorem 22 of~\citet{Dann+2023}]\label{proposi:theorem22ofDann+}
If a base algorithm $\cB$ is ${\tfrac12}$-dd-iw-stable with constants $(c_1, c_2)$,
then \Cref{alg:dd LSBviaCorral} with $\cB$
satisfies
${\tfrac12}$-dd-LSB with constants $(\bar{c}_1, \bar{c}_2)$ where $\bar{c}_1 = \cO(c_1)$
and $\bar{c}_2= \cO(\sqrt{c_1}+\sqrt{c_2})$.
\end{proposition}

\begin{algorithm}
\caption{dd-LSB via Corral, Adaption of Algorithm 6 in \citet{Dann+2023}}\label{alg:dd LSBviaCorral}
 	\SetKwInOut{Input}{Input}
 	\SetKwInOut{Output}{Output}
    \Input{candidate action $\hat{a} \in [K]$, $\frac{1}{2}$-iw-stable algorithm $\cB$ over $[K] \setminus \{\hat{a}\}$ with constants $(c_1, c_2)$}
    
    \textbf{Define}: $\psi(q):= \sum_{i=1}^2 \ln\frac{1}{q_{i}}$,~~~~~~ $B_0:=0\,;$

    \For{$t=1,2,\ldots$}{
        Observe $X_t$;

        Let $\cB$ output an action $\tilde{A}_t$; 
        
        Receive predictors $\bm{m}_{t,a}$ for all $a \in [K]$, and set $y_{t,1}=\la X_t, \bm{m}_{t,\hat{a}} \ra$ and $y_{t,2}=\la X_t, \bm{m}_{t,\tilde{A}_t} \ra$;

        Compute
        \begin{align*}
            &\bar{q}_{t}\leftarrow  \argmin_{q\in\Delta_2}\left\{\left\langle q, \sum_{\tau=1}^{t-1} z_{\tau} + y_t - \begin{bmatrix} 0 \\ B_{t-1}
            \end{bmatrix}  \right\rangle + \frac{1}{\eta_{t}}\psi(q)\right\}, \quad
            q_{t} \leftarrow  \left(1-\frac{1}{2t^2}\right)\bar{q}_{t} + \frac{1}{4t^2}\one,
         \end{align*}
           where $\eta_{t} \leftarrow  \frac{1}{4}(\log T)^{\frac{1}{2}}\left(\sum_{\tau=1}^{t-1}(\ind{i_\tau=i}-q_{\tau,i})^2\xi^2_{\tau,A_\tau} + (c_1 + c_2^2)\log T\right)^{-\frac{1}{2}}$; 
           
        Sample $i_t\sim q_t$;
        
        \If{$i_t=1$}{
            Choose $A_t=\hat{a}$ and observe $\ell_t(X_t,A_t)$;
        }
        \Else{
            Choose $A_t=\tilde{A}_t$ and observe $\ell_t(X_t,A_t)$;
        }
        Define $z_{t,i} \leftarrow  \frac{(\ell_t(X_t,A_t)-y_{t,i})\ind{i_t=i}}{q_{t,i}} + y_{t,i}$ and
            $B_t\leftarrow  \sqrt{c_1\sum_{\tau=1}^{t} \frac{\xi^2_{t,A_t}\ind{i_\tau=2}}{q_{\tau,2}^2}} + \frac{c_2}{\min_{\tau\leq t} q_{\tau,2}}$;
    }
\end{algorithm}

\subsection{Naive adaption}\label{subsec:appendix naive adaption}

\begin{algorithm}[t]
\caption{RealLinExp3 with adaptive learning rate (\textsc{Adaptive-RealLinExp3})}
\label{alg:real_linexp3}
 	\SetKwInOut{Input}{Input}
 	\SetKwInOut{Output}{Output}
	\Input{Arms $[K]$}

    Receive update probability $q_t \in (0,1]$;

     Let 
        $\eta_t \leftarrow \min\left\{\sqrt{\frac{\log K}{\sum_{s=1}^t \frac{1}{q_s}}},\frac{1}{2c  }\min_{s \leq t} q_{s}   \right\},~~
         \gamma_t \leftarrow \frac{ c \cdot \eta_t}{q_t}$,
   where $c=\frac{K}{ \lambdaminSigma};$\\
    \medskip
     \textbf{Initialization:}  Set $\hattheta_{0,i}=\bm{0}$ for all $i \in [K]$;
	
	\For{$t=1,2,\ldots, T$}{

        Observe $X_t$, and for all $a \in [K]$, set
            \begin{align*}
            p_t(a|X_t)=\exp\left(-\eta_t \sum_{s=1}^{t-1}\la X_t, \hattheta_{s,a} \ra \right);
            \end{align*}

        Sample an action $A_t$ from the policy defined as 
            \begin{align*}
            \pi_t(a|X_t)=(1-\gamma_t) \frac{p_t(a|X_t)}{\sum_{b \in [K]}p_t(b|X_t)} +\gamma_t \frac{1}{K};
            \end{align*}

        With probability $q_t$, observe the loss $\ell_t(X_t,a_t)$ (in this case, set $\upd_t=1$, otherwise set $\upd_t=0$);

        
        Compute $\widehat{\bm{\theta}}_{t,a}=\frac{\upd_t}{q_t} \mathbf{\Sigma}_{t,a}^{-1} X_t \ell_t(X_t,A_t)\ind{A_t=a} $ for all $a \in [K]$; \label{alglin_reallinexp estimator}

	}
 
\end{algorithm}

As we discussed in \Cref{sec:Zero-order bound via reduction framework},
the work of \cite{Dann+2023} devised a black-box reduction framework to obtain a zero-order regret bound in the adversarial regime as well as the regret in the form of $\frac{\log T}{\Delta_{\min}}$ in the (corrupted) stochastic regime.
In this section, we demonstrate that a basic \textsc{Exp3}-type algorithm with an adaptive learning rate satisfies the importance-weighting stability (\Cref{def:iw-stable}), where its pseudocode is detailed in \Cref{alg:real_linexp3}.
Specifically, the base algorithm is built upon \textsc{RealLinExp3} in \citet{neu2020efficient}, but we assume that $\mathbf{\Sigma}^{-1}$ is known to the learner.

\begin{proposition}[iw-stable condition of \textsc{Adaptive-RealLinExp3} as a base algorithm] \label{propo:linexp3_iwstability}
Assume that $\SigmaInv$ is known to the learner.
Then, RealLinExp3 with adaptive learning rate (Algorithm~\ref{alg:real_linexp3}) for $K$-armed linear contextual bandits is $\frac{1}{2}$-importance-weighting stable, where $c_1=  \cO \left(  \log(K)K^2 \left(d+\frac{1}{\lambdaminSigma} \right)^2 \right)$ and $c_2=\frac{K \log K}{\lambdaminSigma}$.
\end{proposition}
The proof of \Cref{propo:linexp3_iwstability} will be stated soon.
Using Propositions \ref{proposi:theorem6_Dann+2023},~\ref{proposi:theorem11_Dann+2023}, and ~\ref{propo:linexp3_iwstability},
we have the following proposition.
\begin{proposition}[BoBW reduction with a base algorithm of \textsc{Adaptive-RealLinExp3}]\label{propsi:linexp3_iw_bounds}
Assume that $\SigmaInv$ is known to the learner.
Combining Algorithms~\ref{alg:LSBtoBOBW}, \ref{alg:LSBviaCorral} and \ref{alg:real_linexp3} results in the following the regret bound:
for the adversarial regime, 
\[
R_T =\cO \left(\sqrt{c_1 T} +c_2 \log^2 T \right),
\]
and for the corrupted stochastic regime,
\[
R_T= \cO \left(\frac{c_1 \log T}{\Delta_{\min}} + \sqrt{\frac{c_1 \log T}{\Delta_{\min}} C} + c_2 \log(T) \log \left(\frac{C}{\Delta_{\min}} \right)  \right),
\]
where $c_1=  \cO \left(  \log(K)K^2 \left(d+\frac{1}{\lambdaminSigma} \right)^2 \right)$ and $c_2=\frac{K \log K}{\lambdaminSigma}$.
\end{proposition}
\Cref{propsi:linexp3_iw_bounds} implies that we obtain desired BoBW bounds if the learner access to $\mathbf{\Sigma}^{-1}_{t,a}:=\E_t[\ind{A_t=a}X_t X_t^{\top}]$ for computing the unbiased estimator $\hattheta_{t,a}$ at each round $t$ and $a \in [K]$. However, it only gives the zero-order bound in the adversarial regime.
To obtain data-dependent bounds we use a continuous MWU approach as described in \Cref{sec:data-dependent}.
Importantly, removing the prior knoweges of $\mathbf{\Sigma}^{-1}_{t,a}$ is addressed in \Cref{sec:FTRLShannon}.
In what follows, we state the proof of \Cref{propo:linexp3_iwstability}.
\begin{proof}[Proof of Proposition~\ref{propo:linexp3_iwstability}]
While $\pi^* \in \Pi$ is a deterministic policy, we will also write it using the notations of a probabilistic policy: Let $\pi^*(a|\bm{x})=1$ if $a=\pi^*(\bm{x})$ otherwise 0 for $a \in [K]$, and $\bm{x} \in \cX$. 
Let $X_0 \sim \cD$ be a ghost sample chosen independently from the entire history.
Then, we have 
\begin{align*}
    \E_t[ \la X_t, \bm{\theta}_{t, \pi(X_t)} \ra]=  \E_t[ \la X_0, \bm{\theta}_{t, \pi(X_0)} \ra].
\end{align*}
We define $\widehat{R}_T(\bm{x})$ as the regret of auxiliary game for context $\bm{x}$ and unbiased loss estimator $\hattheta_{t,a}$ at round $t$:
     \begin{align}\label{eq: unbiased regret of auxiliary game}
             \widehat{R}_T(\bm{x}):=\E \left[\sum_{t=1}^T \la \bm{x}, \hattheta_{t,A_t}  \ra -\la \bm{x}, \hattheta_{t,\pi^*(\bm{x})} \ra \right].
     \end{align}

Using this property and unbiased estimator $\hattheta_{t,a}$, as also analyzed in Lemma 3 in \cite{olkhovskaya2023first},
we have
\begin{align}\label{eq:regret with ghost sample}
R_{\tau}&=\E\left[ \sum_{t=1}^{\tau}\Big(\ell_{t}(X_t,A_t)-\ell_t(X_t,\pi^*(X_t)) \Big) \right]\notag\\
& =
\E\left[ \sum_{t=1}^{\tau}\Big(\ell_{t}(X_0,A_t)-\ell_t(X_0,\pi^*(X_0)) \Big) \right]\notag\\
&=\E\left[ \sum_{t=1}^{\tau}\Big(\la X_0,\hattheta_{t,A_t} \ra-\la X_0,\hattheta_{t,\pi^*(X_0)} \ra\Big) \right].
\end{align}
Then, by the definition of $\widehat{R}_T(\bm{x})$ in \Cref{eq: unbiased regret of auxiliary game}, RHS of \Cref{eq:regret with ghost sample} can be  written as $\E[\hat{R}_\tau(X_0)]$:
\begin{align*}
\E[\hat{R}_\tau(X_0)]=\sum_{t=1}^\tau \E_t \left[ \sum_{a \in [K]} ( \pi_t(a|X_0)-\pi^*(a|X_0)) \la X_0, \widehat{\bm{\theta}}_{t,a} \ra \right].
\end{align*}

We begin with the following lemma using a basic FTRL analysis.
\begin{lemma}\label{lemma:BoundforAuxiliaryBame}
For any context $\bm{x} \in \cX$, and suppose that $\widehat{\bm{\theta}}_{t,a}$ satisfies $|\eta_t \la \bm{x},\widehat{\bm{\theta}}_{t,a} \ra| \leq  1$. 
Then, for any time step $\tau$,
we have
      \begin{align*}
            \E[\hat{R}_\tau(\bm{x})]
            \leq 2\sum_{t=1}^{\tau}\E_t \left[ \gamma_t\right] + \E \left[\frac{\log K}{\eta_{\tau}}\right]+\sum_{t=1}^{\tau}\E_t \left[ \eta_t \sum_{a=1}^K \pi_t(a|\bm{x}) \la \bm{x}, \widehat{\bm{\theta}}_{t,a} \ra^2 \right]
     \end{align*}     
            
\end{lemma}
\begin{proof}[Proof of \Cref{lemma:BoundforAuxiliaryBame}]
Since $\pi_t(a|\bm{x}) = (1-\gamma_t) p_t(a|\bm{x})+\gamma_t\frac{1}{K}$ where we recall that $p_t(a|\bm{x})$ is given in \cref{eq: exp3}:
\begin{align*}
    p_t(a|\bm{x})= \frac{\exp{(-\eta_t \sum_{s=1}^{t-1} \la \bm{x}, \hattheta_{s,a} \ra  )}}{\sum_{b \in [K]}\exp{(-\eta_t \sum_{s=1}^{t-1} \la \bm{x}, \hattheta_{s,b} \ra  )}}  \ \mathrm{for}\ a \in [K],
    \end{align*}
we see that
  \begin{alignat*}{4}
     &\E[\hat{R}_\tau(\bm{x})]=\sum_{t=1}^\tau \E_t \left[ \sum_{a \in [K]} ( \pi_t(a|\bm{x})-\pi^*(a|\bm{x})) \la \bm{x}, \widehat{\bm{\theta}}_{t,a} \ra \right]\\
    & \leq \sum_{t=1}^\tau \E_t \left[(1-\gamma_t) \sum_{a \in [K]} ( p_t(a|\bm{x})-\pi^*(a|\bm{x})) \la \bm{x}, \widehat{\bm{\theta}}_{t,a} \ra \right]+\sum_{t=1}^\tau \E_t \left[\frac{\gamma_t}{K}\sum_{a \in [K]} (\la \bm{x}, \widehat{\bm{\theta}}_{t,a} \ra-\la \bm{x}, \hattheta_{t,\pi^*(x)} \ra)  \right].
 \end{alignat*}


As discussed in \Cref{sec:preliminary}, $p_t(\cdot|\bm{x})$ can also be described as the FTRL with negative Shannon entropy:
\begin{align}\label{eq:FTRL}
    p_t(\cdot|\bm{x}) \in \argmin_{p \in \simplexK} \left\{ \sum_{s=1}^{t-1} \la p, \hat{\bell}_s(\bm{x})  \ra + \psi_t(p)\right\},
\end{align}
where 
    $\psi_t(p)=-\frac{1}{\eta_t}H(p)=\frac{1}{\eta_t} \sum_{a \in [K]} p_a \ln p_a$.      
 By a standard FTRL analysis as in Lemma~\ref{lemma:FTRLanalysis_lincon} and similar analysis of 
 derivation of \cref{eq:stabilitybound} in
 Lemma~\ref{lem:regretdecom_auxiliarygame},
 we have 
  \begin{alignat*}{4}
    & \sum_{t=1}^\tau \E_t \left[(1-\gamma_t) \sum_{a \in [K]} ( p'_t(a|\bm{x})-\pi^*(a|\bm{x})) \la \bm{x}, \widehat{\bm{\theta}}_{t,a} \ra \right]\\
    & \leq  \sum_{t=1}^{\tau}\E_t \left[ \eta_t \sum_{a=1}^K \pi_t(a|\bm{x}) \la \bm{x}, \widehat{\bm{\theta}}_{t,a} \ra^2 \right]+\E\left[\frac{\log K}{\eta_t} \right].
 \end{alignat*}

Since 
\[
\sum_{t=1}^\tau \E_t \left[\frac{\gamma_t}{K}\sum_{a \in [K]} (\la \bm{x}, \widehat{\bm{\theta}}_{t,a} \ra-\la \bm{x}, \hattheta_{t,\pi^*(x)} \ra)  \right]
= \sum_{t=1}^\tau\E_t \left[ \frac{\gamma_t}{K}\sum_{a \in [K]} (\la \bm{x}, \bm{\theta}_{t,a} \ra-\la \bm{x}, \bm{\theta}_{t,\pi^*(x)} \ra)  \right]
\leq 2\sum_{t=1}^\tau \E_t \left[\gamma_t \right]
\]
by $|\la \bm{x},\bm{\theta}_{t,a} \ra| \leq 1$, combining above equalites gives the desired result.
\end{proof}

\allowdisplaybreaks
We next introduce the following lemma, which is implied by \Cref{lemma:Lemma6_neu}, for known $\SigmaInv_{t,a}$ and unbiased estimator $\hattheta_{t,a}$ in \Cref{alglin_reallinexp estimator} of \Cref{alg:real_linexp3}.
\begin{lemma}\label{lemma:adap_Lemma6_neu2020}
Let $X_0 \sim \cD$ be a ghost sample chosen independently from the entire interaction history. Then for any time step $t$,
we have
      \begin{align}
\E_t \left[  \sum_{a=1}^K \pi_t(a|X_0) \la X_0, \widehat{\bm{\theta}}_{t,a} \ra^2 \right] \leq \frac{\sum_{a=1}^K \E_t[\mathrm{tr}(\bm{\Sigma}_{t,a}  \bm{\Sigma}_{t,a}^{-1} \bm{\Sigma}_{t,a}  \bm{\Sigma}_{t,a}^{-1})] }{q_t} \leq \frac{3Kd}{q_t}.
     \end{align}     
\end{lemma}

Then, we are ready to prove \Cref{propo:linexp3_iwstability}.
We first show $|\eta_t \la \bm{x},\widehat{\bm{\theta}}_{t,a} \ra| \leq  1$.
\begin{align}
|\eta_t \la \bm{x},\widehat{\bm{\theta}}_{t,a} \ra|  = \eta_t |\la \bm{x},\widehat{\bm{\theta}}_{t,a} \ra| = \eta_t \left| \bm{x}^{\top} \frac{\upd_t}{q_t} \bm{\Sigma}_{t,a}^{-1} X_t \ell_t(X_t,a) \ind{A_t=a}       \right| \leq \frac{\eta_t}{q_t}|\bm{x}^{\top} \bm{\Sigma}_{t,a}^{-1} X_t|\\\notag
\leq \frac{\eta_t}{q_t} \|\bm{\Sigma}_{t,a}^{-1} \|_{\mathrm{op}}\cdot \max_{\bm{x} \in \cX} \|\bm{x}\|^2
\leq  \frac{\eta_t}{q_t} \frac{1}{ \lambda_{\min}(\bm{\Sigma}_{t,a}) }
\leq \frac{\eta_t}{q_t} \frac{K}{ \lambdaminSigma \gamma_t} \leq 1, 
\end{align}
where we used $\ell_t(X_t,a)\leq 1$ in the first inequality,
$\lambda_{\min}(\bm{\Sigma}_{t,a})   \geq \frac{\gamma_t \lambdaminSigma }{ K}$ in the forth inequality,
and the definition of $\gamma_t= \frac{\eta_t K}{q_t \lambdaminSigma}$ in the last inequality.

Next, we will give the bound of $\sum_{t=1}^{\tau} \frac{\eta_t}{q_t}$.
Since we have
\begin{align*}
\sum_{t=1}^{\tau} \frac{1}{q_t} \frac{1}{\sqrt{\sum_{s=1}^{t} \frac{1}{q_s} }} \leq 2 \sum_{t=1}^{\tau} \frac{\frac{1}{q_t}}{\sqrt{ \sum_{s=1}^{t}  \frac{1}{q_s}  }+\sqrt{\sum_{s=1}^{t-1}  \frac{1}{q_s} }} = 2 \sum_{t=1}^{\tau} \left(  \sqrt{\sum_{s=1}^{t}  \frac{1}{q_s}} - \sqrt{\sum_{s=1}^{t-1} \frac{1}{q_s}}  \right) =2 \sqrt{\sum_{s=1}^{\tau} \frac{1}{q_s}}
\end{align*}
and using the definition of $\eta_t$, we obtain
\begin{align}\label{eq: sum of eta_t q_t}
    \sum_{t=1}^{\tau} \frac{\eta_t}{q_t}
    \leq \sqrt{\log K}\sum_{t=1}^{\tau} \frac{1}{q_t} \sqrt{\frac{1}{\sum_{s=1}^t \frac{1}{q_s}}} 
    \leq \sqrt{4 \log K \sum_{t=1}^{\tau} \frac{1}{q_t}}.
\end{align}
Furthermore,  by the definition of $\eta_t$, it is easy to see that  
\begin{align*}
   \frac{1}{\eta_\tau} \leq \sqrt{ \frac{  \sum_{t=1}^{\tau} \frac{1}{q_t} }{\log K}   }+\frac{c}{\min_{t \leq \tau} q_t}.
\end{align*}

Therefore, by combining the above inequalities, we have
for any $a^* \in [K]$ and $\tau \in [T]$,
 \begin{alignat*}{4}
   R(\tau,a^*)=\E\left[ \sum_{t=1}^{\tau}\Big(\ell_{t}(X_t,A_t)-\ell_t(X_t,a^*)\Big)  \right]&= \E\left[ \sum_{t=1}^{\tau}\Big(\ell_{t}(X_0,A_t)-\ell_t(X_0,a^*) \Big) \right]\\
   &\leq \E\left[ \sum_{t=1}^{\tau}\Big(\ell_{t}(X_0,A_t)-\ell_t(X_0,\pi^*(X_0)) \Big) \right]\\
    &=\E\left[ \sum_{t=1}^{\tau}\Big(\la X_0,\hattheta_{t,A_t} \ra-\la X_0,\hattheta_{t,\pi^*(X_0)} \ra\Big) \right]\\
    &=\E \left[\widehat{R}_T(X_0)\right]\\
    &\leq 2\sum_{t=1}^{\tau}\E_t \left[ \gamma_t\right] + \E \left[\frac{\log K}{\eta_{\tau}}\right]+\sum_{t=1}^{\tau}\E_t \left[ \eta_t \sum_{a=1}^K \pi_t(a|X_0) \la X_0, \widehat{\bm{\theta}}_{t,a} \ra^2 \right]\\
    &\leq 2c  \cdot  \E\left[ \sum_{t=1}^{\tau} \frac{\eta_t}{q_t} \right]+ \E \left[\frac{\log K}{\eta_{\tau}}\right]+3Kd \cdot \E\left[ \sum_{t=1}^{\tau} \frac{\eta_t}{q_t}\right]\\ 
       &\leq \left( 2c+3Kd \right) \sqrt{4 \log K \sum_{t=1}^{\tau} \frac{1}{q_t}} +   \sqrt{\log K \sum_{t=1}^{\tau} \frac{1}{q_t}    }+\frac{2c \log K }{\min_{t \leq \tau} q_t}\\
         &\leq \sqrt{ (4(2c+3Kd)^2+1)      \log K \sum_{t=1}^{\tau} \frac{1}{q_t}}+  \frac{2c \log K}{\min_{t \leq \tau} q_t}\\
         &\leq \sqrt{ 36 K^2\left(d+\frac{1}{\lambdaminSigma} \right)^2  \log(K)\sum_{t=1}^{\tau} \frac{1}{q_t}}+  \frac{\frac{2K}{ \lambdaminSigma} \log K}{\min_{t \leq \tau} q_t}\,,
 \end{alignat*}
 where the first and second equalities follow from the property of $X_0$ and the fact that $\hattheta_{t,a}$ is unbiased for all $t$ and $a$,
 the first inequality follows from the definition of the optimal policy $\pi^*(X_0)$,
 the second inequality follows from Lemma~\ref{lemma:BoundforAuxiliaryBame},
 and third inequality follows from the definition $\gamma_t$ and Lemma~\ref{lemma:adap_Lemma6_neu2020},
 the fourth inequality follows from \Cref{eq: sum of eta_t q_t} and the definition of $\eta_t$.
 Lastly, we have the statement plugging in the definition of $c=\frac{K}{ \lambdaminSigma}$.

\end{proof}

\section{APPENDIX FOR DATA-DEPENDENT BOUNDS}\label{appendix:second_order}

In this section, we describe how to find a positive semidefinite matrix $\mathbf{S} \in \mathbb{R}^{d \times d}$ to compute a loss predictor $\bm{m}_{t,a}$ in \Cref{eq:predictor} for each round $t$ and $a \in [K]$, and provide omitted proofs for both \Cref{cor:min of first and second order bound} and \Cref{proposi:ddiw_forlincon}.
Combining \Cref{proposi:theorem23ofDann+}, \ref{proposi:theorem22ofDann+}, and \Cref{proposi:ddiw_forlincon} immediately implies \Cref{thm:secondorder}.


\subsection{Concrete choice for a loss predictor}\label{sec:choice for a loss predictor}
As in \citet{Ito_secondorder_linear2020} for linear bandits, if we have the prior knowledge of the support of $\cD$, i.e., context space $\cX$, we can find an appropriate matrix $\mathbf{S}$ such that
$\left\|\bm{m}^{*}\right\|_{\mathbf{S}}^{2}=\cO(d)$ for any vector $\bm{m}^{*} \in \cM$, and $\max _{\bm{x} \in \cX}\|\bm{x}\|_{\mathbf{S}^{-1}}^{2}=\cO(d)$ in our case.
$\cX_{\mathrm{span}}=\left\{\bm{x}_{1}, \ldots, \bm{x}_{d}\right\} \subseteq {\cX}$ is said to be  2-barycentric spanner for $\cX$ if each $\bm{x} \in \cX$ can be expressed as linear combination of elements in $\cX_{\mathrm{span}}$ with coefficients in $[-2,2]$.
Define $\mathbf{S} \in \mathbb{R}^{d \times d}$ as
\begin{equation}\label{eq:choice of S}
    \mathbf{M}=\left(\bm{x}_{1} \bm{x}_{2} \cdots \bm{x}_{d}\right), \quad \mathbf{S}=\mathbf{M} \mathbf{M}^{\top}=\sum_{i=1}^{d} \bm{x}_{i} \bm{x}_{i}^{\top}.
\end{equation}
Then, for $\bm{m} \in \cM$, we can easily confirm $\|\bm{m}\|_{\mathbf{S}}^{2}=\bm{m}^{\top}\left(\sum_{i=1}^{d} \bm{x}_{i} \bm{x}_{i}^{\top}\right) \bm{m} \leq d$ and $\|\bm{x}\|_{\mathbf{S}^{-1}}^{2}=\bm{x}^{\top}\left(\mathbf{M}^{-1}\right)^{\top} \mathbf{M}^{-1} \bm{x}=\bm{u}^{\top} \bm{u} \leq 4 d$ using some $\bm{u} \in[-2,2]^{d}$ such that $\bm{x}=\mathbf{M} \bm{u}$.
Due to Proposition 2.4 in \citet{AwerbuchKleinberg2004}, computation of 2-barycentric spanner for $\cX$ can be done in polynomial time, making $O(d^2 \log d)$-call for linear optimization oracle over $\cX$.

\subsection{Proof of  Corollary~\ref{cor:min of first and second order bound}}
We prove \Cref{cor:min of first and second order bound} based on \Cref{lemma:predictor} with a  concrete choice of a loss predictor.
\Cref{lemma:predictor} provides the upper bound of  $\mathbb{E}\left[\sum_{t=1}^{T}\xi_{t,A_t}^2\right]$ if we choose $\bm{m}_{t,a}$ by \Cref{eq:predictor}. 
\begin{lemma}\label{lemma:predictor}
Let $\mathcal{M}:=\{\bm{m} \in \mathbb{R}^d \mid \la \bm{x}, \bm{m}\ra \leq 1, \ \forall \bm{x} \in \cX \}$.
For $a \in [K]$ and any positive semi-definite matrix $\mathbf{S} \in \mathbb{R}^{d \times d}$, define the predictor $\bm{m}_{t,a}$ as
\begin{equation*}
\bm{m}_{t,a} \in \argmin_{\bm{m} \in \cM} \left\{
\|\bm{m}\|_{\mathbf{S}}^2+  \sum_{j=1}^{t-1}\ind{A_j=a}\left(\left\langle \bm{\theta}_{j,a}-\bm{m}, X_j\right\rangle\right)^{2}
\right\}.
\end{equation*}
Then, for any $\bm{m}^* \in \mathcal{M}$, it holds that
    \begin{align*}
&\mathbb{E}\left[\sum_{t=1}^{T}\xi^2_{t,A_t}\right]
\leq  \mathbb{E}\left[\sum_{t=1}^{T} \left(\left\langle\bm{\theta}_{t,A_t}-\bm{m}^{*}, X_t\right\rangle\right)^{2} \right]
+K\left\|\bm{m}^{*}\right\|_{S}^{2}
+8 K d \log \left(1+\frac{T}{d} \max _{\bm{x} \in \cX}\|\bm{x}\|_{\mathbf{S}^{-1}}^{2}\right),
\end{align*}
where $\xi_{t,A_t}=\left(\left\langle\bm{\theta}_{t,A_t}-\bm{m}_{t,A_t}, X_t\right\rangle\right)$.
\end{lemma}

\begin{proof}[Proof of \Cref{lemma:predictor}]
The proof can be shown in a proof similar to Lemma 3 of \citet{Ito_secondorder_linear2020} and Theorem 11.7 of \citet{cesa2006prediction}, by carefully considering contexts and definition of the predictor of $\bm{m}_{t,a}$.

For any $a \in [K]$ and any $\bm{m}^* \in \mathcal{M}$, we first need to show
\begin{align}\label{eq:squaredloss_bound_eachaction}
\sum_{t=1}^{T}\ind{A_t=a}\left\langle\bm{\theta}_{t,a}-\bm{m}_{t,a}, X_t\right\rangle^{2} 
\leq
\sum_{t=1}^{T}\ind{A_t=a}\left(\left\langle\bm{\theta}_{t,a}-\bm{m}^*, X_t\right\rangle\right)^{2} +\left\|\bm{m}^*\right\|_{S}^{2}+8 \sum_{t=1}^{T}\left\|X_t\right\|_{\mathbf{G}_t^{-1}}^{2}.
\end{align}
From this,
we have that
\[
\sum_{t=1}^{T} \sum_{a \in [K]}\ind{A_t=a}\left\langle\bm{\theta}_{t,a}-\bm{m}_{t,a}, X_t\right\rangle^{2} 
\leq
\sum_{t=1}^{T} \sum_{a \in [K]}\ind{A_t=a}\left\langle\bm{\theta}_{t,a}-\bm{m}^*, X_t\right\rangle^{2} +K\left\|\bm{m}^*\right\|_{S}^{2}+8 K\sum_{t=1}^{T}\left\|X_t\right\|_{\mathbf{G}_t^{-1}}^{2}.
\]
Therefore, we obtain
\allowdisplaybreaks
\begin{align}\label{eq:squareloss}
&\mathbb{E}\left[\sum_{t=1}^{T}\left\langle\bm{\theta}_{t,A_t}-\bm{m}_{t,A_t}, X_t\right\rangle^{2}\right]\notag \\
&=
\mathbb{E}\left[\sum_{t=1}^{T}  \mathbb{E}_{A_t \sim Q_t}[\left\langle \bm{\theta}_{t,A_t}-\bm{m}_{t,A_t}, X_t\right\rangle^{2}]\right]\notag\\
&=
\mathbb{E}\left[\sum_{t=1}^{T} \sum_{a \in [K]} \ind{A_t=a}\left\langle \bm{\theta}_{t,a}-\bm{m}_{t,a}, X_t\right\rangle^{2}\right]\notag\\
& \leq \mathbb{E}\left[\sum_{t=1}^{T} \sum_{a \in [K]}\ind{A_t=a}\left\langle\bm{\theta}_{t,a}-\bm{m}^*, X_t\right\rangle^{2} +K\left\|\bm{m}^*\right\|_{S}^{2}+8 K\sum_{t=1}^{T}\left\|X_t\right\|_{\mathbf{G}_t^{-1}}^{2}\right] \notag\\
& \leq \mathbb{E}\left[\sum_{t=1}^{T} \left\langle\bm{\theta}_{t,A_t}-\bm{m}^*, X_t\right\rangle^{2} \right]
+K\left\|\bm{m}^*\right\|_{S}^{2}
+8 K\mathbb{E}\left[\sum_{t=1}^{T}\left\|X_t\right\|_{\mathbf{G}_t^{-1}}^{2}\right].
\end{align}

For $t=0,1, \ldots, T$, we define convex functions $f_{t}: \mathcal{M} \rightarrow \mathbb{R}$ and $F_{t}: \mathcal{M} \rightarrow \mathbb{R}$ as follows:
\begin{align*}
    &f_{0}(\bm{m})  =\frac{1}{2}\|\bm{m}\|_{\mathbf{S}}^{2}, & \\
&f_{t}(\bm{m})=  \frac{1}{2} \ind{A_t=a}\left(\left\langle \bm{\theta}_{t,a}-\bm{m}, X_t\right\rangle\right)^{2} & (t \in[T]), \\
&F_{t}(\bm{m})  =\sum_{j=0}^{t} f_{j}(\bm{m}) & (t \in\{0,1, \ldots, T\}) .
\end{align*}
Then, the definition of $\bm{m}_{t,a}$ in \Cref{eq:predictor} can be rewritten as:
\begin{align}\label{eq:argmin_F_t}
    \bm{m}_{t,a} \in \underset{\bm{m} \in \mathcal{M}}{\operatorname{argmin}}\ F_{t-1}(\bm{m}).
\end{align}
By applying this fact repeatedly, we can derive the following for arbitrary $\bm{m}^* \in \mathcal{M}$.
\begin{align*}
F_{T}\left(\bm{m}^*\right) & \geq F_{T}\left(\bm{m}_{T+1,a}\right)=F_{T-1}\left(\bm{m}_{T+1,a}\right)+f_{T}\left(\bm{m}_{T+1,a}\right) \geq F_{T-1}\left(\bm{m}_{t,a}\right)+f_{T}\left(\bm{m}_{T+1,a}\right) \\
& =f_{T-2}\left(\bm{m}_{t,a}\right)+f_{T-1}\left(\bm{m}_{t,a}\right)+f_{T}\left(\bm{m}_{T+1,a}\right) \geq \cdots \geq f_{0}\left(\bm{m}_{1,a}\right)+\sum_{t=1}^{T} f_{t}\left(\bm{m}_{T+1,a}\right) \\
& \geq \sum_{t=1}^{T} f_{t}\left(\bm{m}_{T+1,a}\right).
\end{align*}
From this, we have
\begin{align*}
& \sum_{t=1}^{T}\ind{A_t=a}\left(\left\langle\bm{\theta}_{t,a}-\bm{m}_{t,a}, X_t\right\rangle\right)^{2}-\sum_{t=1}^{T}\ind{A_t=a}\left(\left\langle\bm{\theta}_{t,a}-\bm{m}^*, X_t\right\rangle\right)^{2}\\
&=2 \sum_{t=1}^{T} f_{t}\left(\bm{m}_{t,a}\right)-2 \sum_{t=1}^{T} f_{t}\left(\bm{m}^*\right) \\
& =2 \sum_{t=1}^{T} f_{t}\left(\bm{m}_{t,a}\right)-2\left(F_{T}\left(\bm{m}^*\right)-f_{0}\left(\bm{m}^*\right)\right) \leq 2 f_{0}\left(\bm{m}^*\right)+2 \sum_{t=1}^{T}\left(f_{t}\left(\bm{m}_{t,a}\right)-f_{t}\left(\bm{m}_{T+1,a}\right)\right) \\
& =\left\|\bm{m}^*\right\|_{S}^{2}+2 \sum_{t=1}^{T}\left(f_{t}\left(\bm{m}_{t,a}\right)-f_{t}\left(\bm{m}_{T+1,a}\right)\right)
\end{align*}
We next show
$$
f_{t}\left(\bm{m}_{t,a}\right)-f_{t}\left(\bm{m}_{T+1,a}\right) \leq 4\left\|X_t\right\|_{\mathbf{G}_t^{-1}}^{2} ,
$$
where we define positive semi-definite matrices $\mathbf{G}_t \in \mathbb{R}^{d \times d}$ for $t=0,1, \ldots, T$ by
$$
\mathbf{G}_t=\mathbf{S}+\sum_{j=1}^{t} X_j X_j^{\top}.
$$
For positive definite matrix $\mathbf{S}$, $f_0(\bm{m})$
 is strongly convex with respect to the norm $\|\bm{u}\|^2_\mathbf{S}$.
 Also note that $f_t(\bm{m})$ for $t \in [T]$ is a convex function.
 Therefore, $F_{t}$ is $\mathbf{G}_t$-strongly convex, i.e., it holds for any $\bm{m}, \bm{m}^{\prime} \in \mathcal{M}$ that
\begin{align}\label{eq:strongly convex}
F_{t}\left(\bm{m}^{\prime}\right) \geq F_{t}(\bm{m})+\left\langle\nabla F_{t}(\bm{m}), \bm{m}^{\prime}-\bm{m}\right\rangle+\left\|\bm{m}^{\prime}-\bm{m}\right\|_{\mathbf{G}_t}^{2} .
\end{align}

Further, \Cref{eq:argmin_F_t} implies that
\begin{align}\label{eq:nabra_innerproduct}
\left\langle\nabla F_{t-1}\left(\bm{m}_{t,a}\right), \bm{m}-\bm{m}_{t,a}\right\rangle \geq 0
\end{align}
for any $\bm{m} \in \mathcal{M}$ and $t \in[T]$.
From \Cref{eq:strongly convex} and this inequality, we can show that
\begin{align*}
&f_{t}\left(\bm{m}_{t,a}\right)-f_{t}\left(\bm{m}_{T+1,a}\right)\\
&=F_{t}\left(\bm{m}_{t,a}\right)-F_{t}\left(\bm{m}_{T+1,a}\right)-F_{t-1}\left(\bm{m}_{t,a}\right)+F_{t-1}\left(\bm{m}_{T+1,a}\right) \\
&\leq\left\langle\nabla F_{t}\left(\bm{m}_{t,a}\right), \bm{m}_{t,a}-\bm{m}_{T+1,a}\right\rangle-\left\|\bm{m}_{t,a}-\bm{m}_{T+1,a}\right\|_{\mathbf{G}_t}^{2}+\left\langle\nabla F_{t-1}\left(\bm{m}_{T+1,a}\right), \bm{m}_{T+1,a}-\bm{m}_{t,a}\right\rangle \\
&\leq\left\langle\nabla F_{t}\left(\bm{m}_{t,a}\right)-\nabla F_{t-1}\left(\bm{m}_{t,a}\right), \bm{m}_{t,a}-\bm{m}_{T+1,a}\right\rangle\\
&\quad \quad +\left\langle\nabla F_{t-1}\left(\bm{m}_{T+1,a}\right)-
\nabla F_{t}\left(\bm{m}_{T+1,a}\right), \bm{m}_{T+1,a}-\bm{m}_{t,a}\right\rangle 
-\left\|\bm{m}_{t,a}-\bm{m}_{T+1,a}\right\|_{\mathbf{G}_t}^{2} \\
&=\left\langle\nabla f_{t}\left(\bm{m}_{t,a}\right), \bm{m}_{t,a}-\bm{m}_{T+1,a}\right\rangle-\left\|\bm{m}_{t,a}-\bm{m}_{T+1,a}\right\|_{\mathbf{G}_t}^{2}-\left\langle\nabla f_{t}\left(\bm{m}_{T+1,a}\right), \bm{m}_{T+1,a}-\bm{m}_{t,a}\right\rangle \\
&=\left\langle\nabla f_{t}\left(\bm{m}_{t,a}\right)+\nabla f_{t}\left(\bm{m}_{T+1,a}\right), \bm{m}_{t,a}-\bm{m}_{T+1,a}\right\rangle-\left\|\bm{m}_{t,a}-\bm{m}_{T+1,a}\right\|_{\mathbf{G}_t}^{2} \\
&\leq\left\|\nabla f_{t}\left(\bm{m}_{t,a}\right)+\nabla f_{t}\left(\bm{m}_{T+1,a}\right)\right\|_{\mathbf{G}_t^{-1}}\left\|\bm{m}_{t,a}-\bm{m}_{T+1,a}\right\|_{\mathbf{G}_t}-\left\|\bm{m}_{t,a}-\bm{m}_{T+1,a}\right\|_{\mathbf{G}_t}^{2} \\
&\leq \frac{1}{4}\left\|\nabla f_{t}\left(\bm{m}_{t,a}\right)+\nabla f_{t}\left(\bm{m}_{T+1,a}\right)\right\|_{\mathbf{G}_t^{-1}}^{2}=\frac{1}{4}\left\|\left(\left\langle \bm{m}_{t,a}-\bm{\theta}_{t,a}, X_t\right\rangle+\left\langle \bm{m}_{T+1,a}-\bm{\theta}_{t,a}, X_t\right\rangle\right) X_t\right\|_{\mathbf{G}_t^{-1}}^{2}\\
& \leq 4\left\|X_t\right\|_{\mathbf{G}_t^{-1}}^{2},
\end{align*}
where the first and second inequalities follow from  \Cref{eq:strongly convex} and \Cref{eq:nabra_innerproduct} respectively,
the third inequality follows from the Cauchy-Schwarz inequality,
the forth inequality follows from the fact that $a^{2}-a b+ b^{2} / 4=(a-b / 2)^{2} \geq 0$ for $a, b \in \mathbb{R}$.
Therefore, we obtain
$$
\sum_{t=1}^{T}\ind{A_t=a}\left\langle\bm{\theta}_{t,a}-\bm{m}_{t,a}, X_t\right\rangle^{2}-\sum_{t=1}^{T}\ind{A_t=a}\left\langle\bm{\theta}_{t,a}-\bm{m}^*, X_t\right\rangle^{2} \leq\left\|\bm{m}^*\right\|_{S}^{2}+8 \sum_{t=1}^{T}\left\|X_t\right\|_{\mathbf{G}_t^{-1}}^{2},
$$
which is \Cref{eq:squaredloss_bound_eachaction}.
We next show
\begin{align}\label{eq:norm of X_t upper bound}
\sum_{t=1}^{T}\left\|X_t\right\|_{\mathbf{G}_t^{-1}}^{2} \leq d \log \left(1+\frac{T}{d} \max _{\bm{x} \in \cX}\|\bm{x}\|_{\mathbf{S}^{-1}}^{2}\right) .
\end{align}
Using Lemma 11.11 and similar analysis of Theorem 11.7 in \citet{cesa2006prediction},
we have 
\begin{align*}
& \log \operatorname{det} \mathbf{G}_t-\log \operatorname{det} \mathbf{G}_{t-1}=-\left(\log \operatorname{det}\left(\mathbf{G}_t-X_t X_t^{\top}\right)-\log \operatorname{det} \mathbf{G}_t\right) \\
& =-\log \operatorname{det}\left(\mathbf{G}_t^{-\frac{1}{2}}\left(\mathbf{G}_t-X_t X_t^{\top}\right) \mathbf{G}_t^{-\frac{1}{2}}\right)=-\log \operatorname{det}\left(I-\mathbf{G}_t^{-\frac{1}{2}} X_t X_t^{\top} \mathbf{G}_t^{-\frac{1}{2}}\right) \\
& =-\log \left(1-\left\|\mathbf{G}_t^{-\frac{1}{2}} X_t\right\|_{2}^{2}\right) \geq\left\|\mathbf{G}_t^{-\frac{1}{2}} X_t\right\|_{2}^{2}=\left\|X_t\right\|_{\mathbf{G}_t^{-1}}^{2},
\end{align*}
where the forth equality holds since the matrix $\left(I-\mathbf{G}_t^{-\frac{1}{2}} X_t X_t^{\top} \mathbf{G}_t^{-\frac{1}{2}}\right)$ has eigenvalues $\lambda_{1}^{\prime}=1-$ $\left\|\mathbf{G}_t^{-\frac{1}{2}} X_t\right\|_{2}^{2}$ and $\lambda_{2}^{\prime}=\lambda_{3}^{\prime}=\cdots=\lambda_{d}^{\prime}=1$, and the inequality follows from $\log (1+y) \leq y$ for $y>-1$.
Therefore, we obtain
\begin{align*}
X_t^{\top}\mathbf{G}_t^{-1}X_t \leq \log \frac{\operatorname{det} \mathbf{G}_t}{\operatorname{det} \mathbf{G}_{t-1}}.
\end{align*}
Let $\lambda_{1}, \lambda_{2}, \ldots, \lambda_{d} \geq 0$ be eigenvalues of $\sum_{t=1}^{T} \mathbf{S}^{-\frac{1}{2}} X_t X_t^{\top} \mathbf{S}^{-\frac{1}{2}}$.
Then, we have
$$
\sum_{t=1}^{T}\left\|X_t\right\|_{\mathbf{G}_t^{-1}}^{2} \leq \log \operatorname{det} \mathbf{G}_t-\log \operatorname{det}\mathbf{G}_0=\log \operatorname{det}\left(I+\sum_{t=1}^{T} \mathbf{S}^{-\frac{1}{2}} X_t X_t^{\top} \mathbf{S}^{-\frac{1}{2}}\right)=\sum_{i=1}^{d} \log \left(1+\lambda_{i}\right).
$$

Since we have
$
\sum_{i=1}^{d} \lambda_{i}=\operatorname{tr}\left(\sum_{t=1}^{T} \mathbf{S}^{-\frac{1}{2}} X_t X_t^{\top} \mathbf{S}^{-\frac{1}{2}}\right)=\sum_{t=1}^{T}\left\|X_t\right\|_{\mathbf{S}^{-1}}^{2} \leq T \max _{\bm{x} \in \cX}\|\bm{x}\|_{\mathbf{S}^{-1}}^{2}$,
it holds that $\sum_{i=1}^{d} \log \left(1+\lambda_{i}\right) \leq d \log \left(1+\frac{T}{d} \max _{\bm{x} \in \cX}\|\bm{x}\|_{\mathbf{S}^{-1}}^{2}\right)$
which gives us \Cref{eq:norm of X_t upper bound}.
Combining it with \Cref{eq:squareloss},
we obtain
\begin{align*}
&\mathbb{E}\left[\sum_{t=1}^{T}\left\langle\bm{\theta}_{t,A_t}-\bm{m}_{t,A_t}, X_t\right\rangle^{2}\right]\\
& \leq \mathbb{E}\left[\sum_{t=1}^{T} \left\langle\bm{\theta}_{t,A_t}-\bm{m}^*, X_t\right\rangle^{2} \right]
+K\left\|\bm{m}^*\right\|_{S}^{2}
+8 K\mathbb{E}\left[\sum_{t=1}^{T}\left\|X_t\right\|_{\mathbf{G}_t^{-1}}^{2}\right]\\
&\leq  \mathbb{E}\left[\sum_{t=1}^{T} \left\langle\bm{\theta}_{t,A_t}-\bm{m}^*, X_t\right\rangle^{2} \right]
+K\left\|\bm{m}^*\right\|_{S}^{2}
+8 K d \log \left(1+\frac{T}{d} \max _{\bm{x} \in \cX}\|\bm{x}\|_{\mathbf{S}^{-1}}^{2}\right),
\end{align*}
which concludes the proof.
\end{proof}
We are ready to prove \Cref{cor:min of first and second order bound}.
\begin{proof}[Proof of \Cref{cor:min of first and second order bound}]
Since we choose $\mathbf{S}$ by \Cref{eq:choice of S}, it holds that
$\left\|\bm{m}^{*}\right\|_{\mathbf{S}}^{2}=\cO(d)$ and $\max _{\bm{x} \in \cX}\|\bm{x}\|_{\mathbf{S}^{-1}}^{2}=\cO(d)$.
Then, by \Cref{lemma:predictor} and \Cref{thm:secondorder}, it holds that for any $\bm{m}^* \in \cM$,
 \begin{align*}
 &R_T=\cO\left(\sqrt{\kappa_1(d,K,T) \E \left[\sum_{t=1}^T \xi_{t,A_t}^2 \right] \log^2 T}  +\kappa_2(d,K,T) \log^2(T) \right)\\
 &=\cO\left(\sqrt{\kappa_1(d,K,T) \left( \mathbb{E}\left[\sum_{t=1}^{T} \left(\left\langle\bm{\theta}_{t,A_t}-\bm{m}^*, X_t\right\rangle\right)^{2} \right]
+K\left\|\bm{m}^*\right\|_{S}^{2}
+K d \log \left(1+\frac{T}{d} \max _{\bm{x} \in \cX}\|\bm{x}\|_{\mathbf{S}^{-1}}^{2}\right)\right) \log^2 T} \right. \\
&\left. +\kappa_2(d,K,T) \log^2(T) \right)\\
 &=\cO\left(\sqrt{\kappa_1(d,K,T) \left( \mathbb{E}\left[\sum_{t=1}^{T} \left(\left\langle\bm{\theta}_{t,A_t}-\bm{m}^*, X_t\right\rangle\right)^{2} \right]
+K d \log \left(1+T\right)\right) \log^2 T}  +\kappa_2(d,K,T) \log^2(T) \right)\\
 &=\cO\left(K d \log(dK T) \log^2(T) \sqrt{\left( \mathbb{E}\left[\sum_{t=1}^{T} \left(\left\langle\bm{\theta}_{t,A_t}-\bm{m}^*, X_t\right\rangle\right)^{2} \right]
+K d \log (T)\right) }  +(dK)^{3/2} \log(dK T) \log^3(T) \right)\\
 &=\cO\left(K d \log(dK T) \log^2(T) \sqrt{ \mathbb{E}\left[\sum_{t=1}^{T} \left(\left\langle\bm{\theta}_{t,A_t}-\bm{m}^*, X_t\right\rangle\right)^{2} \right]}  +(dK)^{3/2} \log^{3/2}(dK T) \log^3(T) \right)\\
& =\tilde{\cO} \left( Kd\sqrt{ \mathbb{E}\left[\sum_{t=1}^{T} \left(\left\langle\bm{\theta}_{t,A_t}-\bm{m}^*, X_t\right\rangle\right)^{2} \right]}  + (dK)^{3/2} \right),
 \end{align*}
 in the adversarial regime. 
Therefore, we obtain
 \begin{align}\label{second order regret total variance}
  R_T=\tilde{\cO} \left( Kd\sqrt{ \mathbb{E}\left[\sum_{t=1}^{T} \left(\left\langle\bm{\theta}_{t,A_t}-\bm{m}^*, X_t\right\rangle\right)^{2} \right]}  + (dK)^{3/2} \right)= \tilde{\cO} \left( Kd \sqrt{\bar{\Lambda}}   + (dK)^{3/2} \right).
 \end{align}

On the other hand, for $\bm{m}^*=\bm{0}$, we also have that
 \begin{align*}
 &\frac{R_T}{\hat{c}} \leq \sqrt{\kappa_1(d,K,T) \left( \mathbb{E}\left[\sum_{t=1}^{T} \left(\left\langle X_t, \bm{\theta}_{t,A_t}\right\rangle\right)^{2} \right]
+K d \log \left(1+T\right)\right) \log^2 T}  +\kappa_2(d,K,T) \log^2(T)\\
&\leq \sqrt{  \mathbb{E}\left[\sum_{t=1}^{T} \left\langle X_t, \bm{\theta}_{t,A_t}\right\rangle \right]
+K d \log \left(1+T\right)}\sqrt{ \kappa_1(d,K,T)\log^2 T}  +\kappa_2(d,K,T) \log^2(T)\\
& \leq \sqrt{ \kappa_1(d,K,T)\log^2 T}\sqrt{  \mathbb{E}\left[\sum_{t=1}^{T} \left\langle X_t, \bm{\theta}_{t,A_t}\right\rangle \right]}+
\sqrt{ \kappa_1(d,K,T)\log^2 T}\sqrt{ K d \log \left(1+T\right)}   +\kappa_2(d,K,T) \log^2(T),
 \end{align*}
 where $\hat{c}$ is a universal constant and the second inequality follows from $0 \leq \E[\left\langle X_t, \bm{\theta}_{t,A_t}\right\rangle]\leq 1$.
  By the definition of $R_T=\E \left[\sum_{t=1}^T \la X_t, \bm{\theta}_{t,A_t} \ra\right]-L^*$, and solving the quadratic inequality for $\E \left[\sum_{t=1}^T \la X_t, \bm{\theta}_{t,A_t} \ra\right]$,
we obtain
 \begin{align*}
 &\sqrt{\E \left[\sum_{t=1}^T \la X_t, \bm{\theta}_{t,A_t} \ra\right]}\\
&<\frac{\sqrt{ \kappa_1(d,K,T)\log^2 T}+ \sqrt{ \kappa_1(d,K,T)\log^2 T+4\left(L^*+\sqrt{ \kappa_1(d,K,T)\log^2 T}\sqrt{ K d \log \left(1+T\right)}   +\kappa_2(d,K,T) \log^2(T)\right)}}{2}\\
& \leq \sqrt{\frac{\kappa_1(d,K,T)\log^2 T+ \kappa_1(d,K,T)\log^2 T+4\left(L^*+\sqrt{ \kappa_1(d,K,T)\log^2 T}\sqrt{ K d \log \left(1+T\right)}   +\kappa_2(d,K,T) \log^2(T)\right)}{2}  }.
 \end{align*}
This indicates that 
  \begin{align*}
 &\sqrt{\E \left[\sum_{t=1}^T \la X_t, \bm{\theta}_{t,A_t} \ra\right]}
=\cO \left(\sqrt{ L^* +K^2 d^2 \log^2(dK T) \log^4(T) }\right)\\
&=\tilde{\cO} \left(\sqrt{ L^* +K^2 d^2 }\right).
 \end{align*}

Therefore, in this case, we obtain 

 \begin{alignat*} {2}  
 &R_T = \cO\left(\sqrt{ \kappa_1(d,K,T)\log^2 T} \sqrt{ L^* +K^2 d^2 \log^2(dK T) \log^4(T) }\right.\\
 & \quad\left. +
\sqrt{ \kappa_1(d,K,T)\log^2 T}\sqrt{ K d \log \left(1+T\right)} 
 +\kappa_2(d,K,T) \log^2(T) \right)\\
& =\cO\left( K d \log(dK T) \log^2(T) \sqrt{ L^*} +K^2 d^2 \log^2(dK T) \log^4(T)  \right)\\
& =\tilde{\cO} \left( Kd \sqrt{ L^*} +K^2 d^2 \right).
 \end{alignat*}

From \Cref{second order regret total variance} and this, we conclude that 
  \begin{align*}
 R_T =\tilde{\cO} \left( Kd \sqrt{\min\{ L^*, \bar{\Lambda}  \}} +K^2 d^2 \right).
 \end{align*}
 \end{proof}

 We note that instead computing $\bm{m}_{t,a}$ in \Cref{eq:predictor} at round $t$ for each $a \in [K]$,
 we can still get the first-order regret bound in the adversarial regime, just by setting $\bm{m}_{t,a}=\bm{0} \in \mathbb{R}^d$.

\begin{corollary}\label{coro:first-order}
Let  $\kappa_1(d,K,T)=\cO\left(K^2 d^2 \log^2(dK T) \log^2(T)\right)$ and $\kappa_2(d,K,T)=\cO\left((dK)^{3/2} \log(dK T) \log(T) \right)$.
Combining Algorithms~\ref{alg:secondorder_linearcon}, \ref{alg:LSBtoBOBW}, and \ref{alg:dd LSBviaCorral} results in the following the regret bound
  \begin{align*}
 R_T&=\cO \left( \sqrt{\kappa_1(d,K,T)\log^2 T \cdot  L^* }+\log T^{3/2} \kappa_1(d,K,T)^{3/4} +\kappa_2(d,K,T) \log^2(T) \right).
 \end{align*}
 in the adversarial regime,
 and 
 \[
 R_T=\cO \left(\frac{\kappa_1(d,K,T) \log(T)}{\Delta_{\min}}+  \sqrt{ \frac{\kappa_1(d,K,T) \log T C}{\Delta_{\min}}} + \kappa_2(d,K,T) \log(T)\log(C \Delta_{\min}^{-1})\right)
 \]
 in the corrupted stochastic regime.
\end{corollary}

\begin{proof}[Proof of Corollary~\ref{coro:first-order}]
Taking $\bm{m}_{t,a}=\bm{0}$ in Theorem~\ref{thm:secondorder},
for a universal constant $\hat{c}>0$, we have
 \begin{align*}
     \frac{R_T}{\hat{c}} &\leq \sqrt{\kappa_1(d,K,T)\log^2 T} \cdot \sqrt{ \E \left[\sum_{t=1}^T (\la X_t, \bm{\theta}_{t,A_t} \ra)^2 \right] }  +\kappa_2(d,K,T) \log^2(T)\\
& \leq   \sqrt{\kappa_1(d,K,T)\log^2 T} \cdot \sqrt{ \E \left[\sum_{t=1}^T \la X_t, \bm{\theta}_{t,A_t} \ra \right] }  +\kappa_2(d,K,T) \log^2(T),
 \end{align*}
 where the second inequality follows from $0 \leq \ell_t(X_t,A_t) \leq 1$.
 By the definition of $R_T=\E \left[\sum_{t=1}^T \la X_t, \bm{\theta}_{t,A_t} \ra\right]-L^*$, and solving the quadratic inequality for $\E \left[\sum_{t=1}^T \la X_t, \bm{\theta}_{t,A_t} \ra\right]$,
 we obtain $\E \left[\sum_{t=1}^T \la X_t, \bm{\theta}_{t,A_t} \ra\right]=\cO(L^*+\log T\sqrt{\kappa_1(d,K,T)})$.
 Therefore,
 we have
 \begin{align*}
 R_T&=\cO \left( \sqrt{\kappa_1(d,K,T)\log^2 T \cdot  L^* }+(\log T)^{3/2} \kappa_1(d,K,T)^{3/4} +\kappa_2(d,K,T) \log^2(T) \right),
 \end{align*}
which completes the proof.
    
\end{proof}

\subsection{Proof of Proposition~\ref{proposi:ddiw_forlincon}}

Before we state the proof,
 we introduce the concentration property of a log-concave distribution, which is proved in Lemma 1 in \cite{Ito_secondorder_linear2020}.
 \begin{lemma}[Lemma 1, \citet{Ito_secondorder_linear2020}]\label{lem: concentration of log concave}
If $y$ follows a log-concave distribution $p$ over $\mathbb{R}^d$ and $\E_{y \sim p}[y y^{\top}] \preceq I$, we have
\begin{align*}
    \Pr[ \| y\|_2^2 \geq d \alpha^2] \leq d \exp{(1-\alpha)}
\end{align*}
for arbitrary $\alpha \geq 0$.
 \end{lemma}

In order to proceed with further analysis, we introduce several definitions.
 For a probability vector $r \in \simplexK$ and $d$-dimensional context $\bm{x}\in \cX$,
we denote the $dK$-dimentional vector $\bm{z}(r,\bm{x}):=(r_1 \cdot \bm{x}^{\top}, \ldots, r_K \cdot \bm{x}^{\top})^{\top} \in \mathbb{R}^{dK}$.
 We define the $dK \times dK$ matrix $\barSigmaBlock(t):= \mathrm{diag}_{a \in [K]} (\barSigma_{t,a}) \in \mathbb{R}^{dK} \times \mathbb{R}^{dK}$ as a block diagonal arrangement of the covariance matrices per arm, where $\barSigma_{t,a}$ is given in \Cref{eq: def of barSigma}.
 Similarly, we also define $\tilSigmaBlock(t):= \mathrm{diag}_{a \in [K]} (\tilSigma_{t,a}) \in \mathbb{R}^{dK} \times \mathbb{R}^{dK}$,
 where $\tilSigma_{t,a}$ is given in \Cref{eq: def tilde Sigma}.
Using these notation, we can rewrite the $\widetilde{p}_t(r | \bm{x})$ for $r \in \simplexK$ and a context $\bm{x} \in \cX$ as follows:

\begin{align}\label{lincon_trancdist2}
\widetilde{p}_t(r | \bm{x})= \frac{p_t(r|\bm{x}) \ind {\sum_{a=1}^K r_a^2 \|\bm{x}\|^2_{\barSigmaInv_{t,a}}  \leq dK \tilde{\gamma}_t^2 } }{\mathbb{P}_{y \sim p_t(\cdot |\bm{x})}\left[\sum_{a=1}^K y_a^2 \|\bm{x}\|_{\barSigmaInv_{t,a}}^2  \leq d K\tilde{\gamma}_t^2  \right]}
=\frac{p_t(r | \bm{x}) \ind{\| \bm{z}(r,\bm{x}) \|^2_{\barSigmaBlockInv(t)} \leq d K \tilde{\gamma}_t^2 }} {\mathbb{P}_{y \sim p_t(\cdot | \bm{x})}
\left[
\|\bm{z}(y,\bm{x})\|^2_{\barSigmaBlockInv(t)}  \leq dK \tilde{\gamma}_t^2\right]}.
\end{align}

For a context $\bm{x}$, we define  $\finalsampleQt(\bm{x})$  as a sample generated from $\widetilde{p}_t(\cdot | \bm{x})$ in \Cref{lincon_trancdist2},
and  define $\distsampleq(\bm{x})$  as a sample generated from $p_t(\cdot | \bm{x})$ in \Cref{dist:lincon_continuous} wherein $X_t$ is replaced with $\bm{x}$.
Let $\bm{\theta}_t:= (\bm{\theta}^{\top}_{t,1}, \ldots, \bm{\theta}^{\top}_{t,K})^{\top}  \in \mathbb{R}^{dK}$,
and let its estimate be $\hat{\bm{\theta}_t}:= (\hat{\bm{\theta}}^{\top}_{t,1}, \ldots, \hat{\bm{\theta}}^{\top}_{t,K})^{\top}  \in \mathbb{R}^{dK}$.
We denote $\bm{m}_t:= (\bm{m}^{\top}_{t,1}, \ldots, \bm{m}^{\top}_{t,K})^{\top}  \in \mathbb{R}^{dK}$.

 For notational convenience, we also define random vector $Z(\bm{x})^{\top}  \in \mathbb{R}^{dK}$ for context $\bm{x}$ and $\distsampleq(\bm{x}) \sim p_t(\cdot | \bm{x})$ as:
 \begin{equation*}
Z(\bm{x}):=\bm{z}(\distsampleq(\bm{x}),\bm{x}) =(\distsampleq_{1}(\bm{x}) \cdot \bm{x}^{\top}, \ldots, \distsampleq_{K}(\bm{x})\cdot \bm{x}^{\top})^{\top}  \in \mathbb{R}^{dK}.
\end{equation*}
And, we define $\tilde{Z}_t(\bm{x})$ for context $\bm{x}$ and $\finalsampleQt(\bm{x}) \sim \widetilde{p}_t(\cdot | \bm{x})$ as:
 \begin{equation*}
\tilde{Z}_t(\bm{x}):=\bm{z}(\finalsampleQt(\bm{x}),\bm{x}) =  (\finalsampleQsub{t,1}(\bm{x})\cdot \bm{x}^{\top}, \ldots, \finalsampleQsub{t,K}(\bm{x}) \cdot \bm{x}^{\top})^{\top} \in \mathbb{R}^{dK}.
\end{equation*}
For the optimal policy $\pi^* \in \Pi$ and context $\bm{x}$,
we define $Z^*(\bm{x})$  as:
 \begin{equation*}
Z^*(\bm{x}):=(\bm{0}^{\top}, \ldots, \bm{x}^{\top}, \ldots, \bm{0}^{\top})^{\top}  \in \mathbb{R}^{dK},
\end{equation*}
where the term of $\bm{x}$ is placed on $\pi^*(\bm{x})$-th element and $\bm{0} \in \mathbb{R}^d$ is placed on other elements. 
Finally for the uniform distribution over $K$-action $\mu_0=(\frac{1}{K},\ldots,\frac{1}{K})$, and context $\bm{x}$, we define $\bar{Z}(\bm{x})$ as:
 \begin{equation*}
\bar{Z}(\bm{x}):=\left(\frac{1}{K}\bm{x}^{\top}, \ldots,\frac{1}{K}\bm{x}^{\top} \right)^{\top}  \in \mathbb{R}^{dK}.
\end{equation*}

\begin{proof}[Proof of Proposition~\ref{proposi:ddiw_forlincon}]

Using the above notations,
the regret  can be decomposed as: 
\begin{align}\label{eq:linearcon_regret}
    R_\tau&=\mathbb{E}\left[ \sum_{t=1}^\tau (\ell_t(X_t, A_t)-\ell_t(X_t, \pi^*(X_t))  ) \right]
   \notag \\ &=   \mathbb{E}\left[ \sum_{t=1}^\tau   \left\la 
\tilde{Z}_t(\bm{x})-Z^*(X_t),\bm{\theta}_t \right\ra    \right] \notag\\ 
& = \mathbb{E}\left[ \sum_{t=1}^\tau   \left\la 
 \tilde{Z}_t(X_t)-Z(X_t),\bm{\theta}_t \right\ra    \right]
 + \mathbb{E}\left[ \sum_{t=1}^\tau   \left\la 
Z(X_t)-Z^*(X_t),\bm{\theta}_t \right\ra    \right].
   \end{align}

Following the idea of the auxiliary game as presented in \Cref{eq:auxiliary game}, 
for the optimal policy $\pi^* \in \Pi$,  the unbiased estimate of loss vectors $\hattheta_t$, and a fixed context $\bm{x} \in \cX$, we define  
\[
\hat{R}_\tau(\bm{x}):=\sum_{t=1}^\tau    \E_t\left[   \left\la 
Z(\bm{x})-Z^*(\bm{x}),\hattheta_t \right\ra    \right].
\]
Let $X_0 \sim \cD$ be a ghost sample drawn independently from the entire interaction history. Then we have
\begin{align}
    \mathbb{E}\left[ \sum_{t=1}^\tau   \left\la 
Z(X_t)-Z^*(X_t),\bm{\theta}_t \right\ra    \right] 
= \mathbb{E}\left[ \sum_{t=1}^\tau   \left\la Z(X_0)-Z^*(X_0) ,\hattheta_t \right\ra    \right]
=\mathbb{E}[\hat{R}_\tau(X_0)],
\end{align}
where we used the property of unbiased estimates $\hattheta_t$ and the fact that $X_0$ is independent of any past history to constract $\hattheta_t$.

For further analysis, we introduce some lemmas from the prior analysis. The following lemmas hold for our unbiased estimator $\hattheta$ and definitions of $\tilde{\gamma}_t$ and $\eta_t$, since we sample $Q(\bm{x})$ from the distribution $p_t(\cdot | \bm{x})$ defined in \Cref{dist:lincon_continuous} and $Q_t(\bm{x})$ from the truncated distribution $\tilde{p}_t(\cdot | \bm{x})$ defined in \Cref{lincon_trancdist2} for context $\bm{x}$.
We begin with Lemma C.1 of \citet{olkhovskaya2023first}, implying that $Z(\bm{x})$ follows a log-concave distribution under the assumption that the underlying context distribution $\cD$ is log-concave. 
\begin{lemma}[c.f. Lemma C.1 of \citet{olkhovskaya2023first}]
Suppose that $\bm{z}(q,\bm{x})=\sum_{a\in[K]}q_a \varphi(\bm{x},a)$ for $q \in \simplexK$ and $\varphi(\bm{x},a)=(\bm{0}^{\top},\ldots,\bm{x}^{\top},\ldots, \bm{0})$ such that $\bm{x}$ is on the $a$-th co-ordinate and $Q(\bm{x}) \sim p(\cdot|\bm{x})$ for log-concave $p(\cdot|\bm{x})$.
 If $X \sim p_X(\cdot)$ and $p_X(\cdot)$ is log-concave and $Z(X)=\bm{z}(Q(X),X)$, then $Z(X)$ also follows a log-concave distribution.
\end{lemma}

To see that the first term of $ \mathbb{E}\left[ \sum_{t=1}^\tau   \left\la 
 \tilde{Z}_t(X_t)-Z(X_t),\bm{\theta}_t \right\ra    \right]$ in \Cref{eq:linearcon_regret} is a constant, we make use of  Lemma C.2 in \cite{olkhovskaya2023first}, which is the analog of Lemma~4 \citet{Ito_secondorder_linear2020}.
This lemma implies that $\tilde{Z}_t(X_t)$ is close to $Z(X_t)$, and also provides a useful relation between covariance matrices $\barSigmaBlock(t)$ and $\tilSigmaBlock(t)$.
The log-concavity of $Z(X_t)$ is crucial in the proof to utilize its concentration property stated in Lemma 1 of \citet{Ito_secondorder_linear2020} (\Cref{lem: concentration of log concave}).
 
\begin{lemma}[c.f. Lemma C.2 in \cite{olkhovskaya2023first}]\label{lemma: Lemma C.2 olkhovskaya2023first}
   Suppose that $\tilde{\gamma}_t  \geq  4\log(10dKt)$ 
  and $\la (r_1 \cdot \bm{x}^{\top} , \ldots, r_K \cdot \bm{x}^{\top} ), \bm{\theta}_t \ra\in [-1,1]$
for any $t$, a policy $r \in \simplexK$ and context $X_t \in \cX$.
    Then, we have
   \begin{align*}
    \left|  \E_{t}\left[  \left\la 
 \tilde{Z}_t(X_t)-Z(X_t),\bm{\theta}_t \right\ra  \right] \right| \leq \frac{1}{2t^2}.
    \end{align*}
    Further, we have
    \begin{align}\label{ineq:covariancematrix}
       \frac{3}{4} \barSigmaBlock(t) \preceq \tilSigmaBlock(t)  \preceq \frac{4}{3}\barSigmaBlock(t).
    \end{align}
\end{lemma}

Next, we introduce Lemma 4.4 in \citet{olkhovskaya2023first}, the analog of Lemma 5 in \citet{Ito_secondorder_linear2020}, which can be shown via standard the OMD analysis \citep{RakhlinSridharan2013online}.

\begin{lemma}[c.f. Lemma 4.4 in \citet{olkhovskaya2023first}]\label{lemma:4.4inolkhovskaya2023first}

Assume that $\eta_{t+1} \leq \eta_t$ for all $t$, let $\mu_0$ be a uniform distribution over $[K]$ and $\psi(y)=\exp{(y)}-y-1$.
Then, the regret $\hat{R}_\tau(\bm{x})$ for fixed $\bm{x} \in \cX$ of Algorithm~\ref{alg:secondorder_linearcon} almost surely satisfies
\begin{align} \label{eq:regret bound via OMD}
    \hat{R}_\tau(\bm{x}) \leq \frac{1}{\tau}\sum_{t=1}^\tau\left\la \bar{Z}(\bm{x})- Z^*(\bm{x}) , \hattheta_t  \right \ra +\frac{K \log \tau}{ \eta_{\tau}} + \sum_{t=1}^\tau \frac{1}{\eta_t} \mathbb{E}_{t} \left[ \psi(-\eta_t \la 
Z(\bm{x}),\hattheta_t-\bm{m}_t  \ra) \right].
\end{align}
    
\end{lemma}

Next, we introduce Lemma 6 of \citet{Ito_secondorder_linear2020} to evaluate the third term of RHS of \Cref{eq:regret bound via OMD}.
\begin{lemma}[Lemma 6 in \citet{Ito_secondorder_linear2020}]\label{lemma:ito_lemma6}
If $y$ follows a log-concave distribution over $\mathbb{R}$ and if $\E[y^2] \leq \frac{1}{100}$, we have
\begin{align*}
    \E[\psi(y)] \leq \E[y^2]+30\exp{\left(-\frac{1}{\sqrt{\E[y^2]}}\right)} \leq 2 \E[y^2] \ \mathrm{where} \ \psi(x)=\exp{(y)}-y-1.
\end{align*}
\end{lemma}

Now, we start by evaluating the term
$\mathbb{E}_{t} \left[ (-\eta_t \la 
Z(X_0),\hattheta_t-\bm{m}_t  \ra)^2 \right]$.
We recall that the definition of $\widehat{\bm{\theta}}_{t,a}$ is given by
\begin{align*}
\widehat{\bm{\theta}}_{t,a}:= \bm{m}_{t,a}+ \frac{\upd_t}{q_t}  
  \finalsampleQsub{t,a}(X_t) \tilSigmaInv_{t,a} X_t
  \xi_{t,a} \ind{A_t=a},
\end{align*}
where $\xi_{t,a}:=(\ell_t(X_t,a)-\la X_t,\bm{m}_{t,a} \ra)$.
Then, we have that
\begin{align}
 &\E_{\distsampleq(X_0) \sim p_t(\cdot | X_0),\upd_t \sim q_t}\left[ (-\eta_t \la 
Z(X_0),\hattheta_t-\bm{m}_t  \ra)^2  \mid  \cF_{t-1} \right] \notag
    \\\notag
    &=\mathbb{E}_{\upd_t \sim q_t} \left[\eta_t^2 \frac{\upd^2_t}{q^2_t} \E_t \left[  \xi_{t,A_t}^2 Z(X_t)^{\top} \tilSigmaBlockInv(t) Z(X_0) Z(X_0)^{\top} \tilSigmaBlockInv(t) Z(X_t) \right]\right]\\\notag
    &= \eta_t^2\mathbb{E}_{\upd_t \sim q_t} \left[ \frac{\upd^2_t}{q^2_t} \E_t \left[ \xi_{t,A_t}^2 Z(X_t)^{\top} \tilSigmaBlockInv(t) \barSigmaBlock(t) \tilSigmaBlockInv(t) Z(X_t) \right] \right]\\\notag
    & \leq \frac{4}{3}\eta_t^2\mathbb{E}_{\upd_t \sim q_t} \left[ \frac{\upd^2_t}{q^2_t}  \E_t \left[\xi_{t,A_t}^2 Z(X_t)^{\top} \tilSigmaBlockInv(t) \tilSigmaBlock(t) \tilSigmaBlockInv(t) Z(X_t)  \right]\right]\\\notag
    & = \frac{4}{3}\eta_t^2\mathbb{E}_{\upd_t \sim q_t} \left[ \frac{\upd^2_t}{q^2_t}  \E_t \left[\xi_{t,A_t}^2 Z(X_t)^{\top} \tilSigmaBlockInv(t)  Z(X_t)  \right]\right]\\  \notag
     & \leq 2\eta_t^2\mathbb{E}_{\upd_t \sim q_t} \left[ \frac{\upd^2_t}{q^2_t}  \E_t \left[\xi_{t,A_t}^2 Z(X_t)^{\top} \barSigmaBlockInv(t)  Z(X_t) \right] \right]\\\notag     
     & = 2\eta_t^2\mathbb{E}_{\upd_t \sim q_t} \left[ \frac{\upd^2_t}{q^2_t} \E_t \left[ \xi_{t,A_t}^2 \|Z(X_t) \|^2_{\barSigmaBlockInv(t)}   \right]\right]\\\notag  
     & \leq \frac{2d K \eta_t^2 \tilde{\gamma}_t^2}{q_t}  \E_t \left[\xi_{t,A_t}^2 \right]  \label{eq:variance_bound_lincon}\\
     &\leq \frac{1}{100},
\end{align}
where
the first and second inequalities follow from \Cref{lemma: Lemma C.2 olkhovskaya2023first},
the third inequality follows from \Cref{algline:rejection step} in \Cref{alg:secondorder_linearcon} of $\| Z(X_t) \|^2_{\barSigmaBlockInv(t)} \leq dK \tilde{\gamma}_t^2$,
and we used $\eta_t
 \leq \frac{2 \sqrt{q_t}}{\sqrt{800dK} \tilde{\gamma}_t}$ and the assumptions that $|\ell_t(X_t,A_t)| \leq 1$  and  $|\la X_t, \bm{m}_{t,A_t}\ra| \leq 1$ in the last inequality.
Then using Lemma~\ref{lemma:ito_lemma6} for $y=-\eta_t \la 
Z(X_0),\hattheta_t-\bm{m}_t  \ra$ and \Cref{eq:variance_bound_lincon},
we obtain
\begin{align}\label{eq:varianceterm_lincon}
 &   \frac{1}{\eta_t} \E \left[ 
 \psi(-\eta_t \la 
Z(X_0),\hattheta_t-\bm{m}_t  \ra ) \right] 
\leq    \frac{2}{\eta_t} \E \left[ (-\eta_t \la 
Z(X_0),\hattheta_t-\bm{m}_t  \ra)^2  \right] \notag\\
& \leq   \frac{4 d K \eta_t \tilde{\gamma}_t^2}{q_t}\E_t \left[ \xi_{t,A_t}^2  \right].
\end{align}
\allowdisplaybreaks

From the fact that
 $(r_1 \cdot \bm{x}, \ldots, r_K \cdot \bm{x})^{\top} \bm{\theta}_t \in [-1,1]$ for any $t$, $r \in \simplexK$ and $\bm{x} \in \cX$,
 we also see that the first term of RHS in \Cref{eq:regret bound via OMD} is bounded by a constant:
\begin{align}\label{eq:constantterm_lincon}
    \E\left[\frac{1}{\tau}\sum_{t=1}^\tau\left\la \bar{Z}(X_0)- Z^*(X_0) , \hattheta_t  \right \ra \right] =  \frac{1}{\tau}\sum_{t=1}^\tau\left\la \bar{Z}(X_0)- Z^*(X_0) , \bm{\theta}_t  \right \ra \leq 2.
\end{align}

Now, we are ready to prove the main statement.
For any stopping time $\tau \in [1,T]$ and $a^* \in [K]$, 
we have that
\begin{align*}
 &\mathbb{E}\left[ \sum_{t=1}^\tau (\ell_t(X_t, a_t)-\ell_t(X_t, a^*)  ) \right]\\
 & \leq \mathbb{E}\left[ \sum_{t=1}^\tau (\ell_t(X_t, a_t)-\ell_t(X_t, \pi^*(X_t))  ) \right]\\
& = \mathbb{E}\left[ \sum_{t=1}^\tau   \left\la 
 Z_t(X_t)-Z(X_t) ,\bm{\theta}_t \right\ra    \right] + \mathbb{E}\left[ \sum_{t=1}^\tau   \left\la 
Z(X_t)-Z^*(X_t),\bm{\theta}_t \right\ra    \right] \\
& =\mathbb{E}\left[ \sum_{t=1}^\tau   \left\la 
 Z_t(X_t)-Z(X_t) ,\bm{\theta}_t \right\ra    \right]+\mathbb{E}[\hat{R}_\tau(X_0, \pi^*)]\\
& \leq     \sum_{t=1}^\tau \frac{1}{2\tau^2} +    \E\left[  \frac{1}{\tau}\sum_{t=1}^\tau\left\la \bar{Z}(X_0)-Z^*(X_0), \hattheta_t  \right \ra \right]+  \E\left[\frac{K \log \tau}{ \eta_{\tau}} \right]+   \E\left[\sum_{t=1}^\tau \frac{1}{\eta_t} \mathbb{E}_{t} \left[ \psi(-\eta_t \la 
Z(X_0),\hattheta_t-\bm{m}_t  \ra) \right]\right]\\
& \leq     3 +\E\left[\frac{K \log \tau}{ \eta_{\tau}} \right]+  \E\left[\sum_{t=1}^\tau \frac{1}{\eta_t} \mathbb{E}_{t} \left[ \psi(-\eta_t \la 
Z(X_0),\hattheta_t-\bm{m}_t  \ra) \right]\right]\\
& \leq 3+ \E\left[\frac{K \log \tau}{ \eta_{\tau}} \right]+ \E\left[ \sum_{t=1}^\tau \frac{4 d K \eta_t \tilde{\gamma}_t^2}{q_t}\E_t \left[ \xi_{t,A_t}^2  \right]\right],
   \end{align*}
where we use Lemma~\ref{lemma:4.4inolkhovskaya2023first} in the first inequality and we use \Cref{eq:constantterm_lincon} in the second inequality and we use \Cref{eq:varianceterm_lincon} in the last inequality.

   Recall that  $\beta_t:=16 \tilde{\gamma}_t^2 \xi_{t,A_t}^2$, and $\tilde{\gamma}_t=4 \log (10dKt)$.
Also, recall that the learning rate $\eta_t$ is defined as follows:
\begin{align*}
  \eta_t = \frac{1}{\sqrt{\frac{ 800dK \tilde{\gamma}_t^2}{\min_{j \leq t} q_j} +    \sum_{j=1}^{t-1} \frac{\beta_j}{q_j}}}.
\end{align*}
We also define $\eta_t'$ as follows:
\begin{align*}
  \eta_t': = \frac{1}{\sqrt{\frac{ 800dK \tilde{\gamma}_{t-1}^2}{\min_{j \leq t-1} q_j} +    \sum_{j=1}^{t-1} \frac{\beta_j}{q_j}}}.
\end{align*}

Using $  \frac{x}{2\sqrt{y}} \leq \sqrt{y}-\sqrt{y-x}$ for $x=\frac{\beta_t}{q_t}, y=\frac{ 800dK \tilde{\gamma}_t^2}{\min_{j \leq t} q_j}  +\sum_{j=1}^t \frac{\beta_j}{q_j} $,
we have that 
 \begin{alignat}{4}\label{eq:integral_eta_t_lincon1}
 \frac{\frac{\beta_t}{q_t}}{2\sqrt{\frac{ 800dK \tilde{\gamma}_t^2}{\min_{j \leq t} q_j} +    \sum_{j=1}^{t} \frac{\beta_j}{q_j}}}
  &\leq \sqrt{\frac{ 800dK \tilde{\gamma}_t^2}{\min_{j \leq t} q_j} +    \sum_{j=1}^{t} \frac{\beta_j}{q_j}}-\sqrt{\frac{ 800dK \tilde{\gamma}_t^2}{\min_{j \leq t} q_j} +    \sum_{j=1}^{t-1} \frac{\beta_j}{q_j}}\notag\\
 & \leq \sqrt{\frac{ 800dK \tilde{\gamma}_{t}^2}{\min_{j \leq t} q_j} +    \sum_{j=1}^{t} \frac{\beta_j}{q_j}}-\sqrt{\frac{ 800dK \tilde{\gamma}_{t-1}^2}{\min_{j \leq t-1} q_j} +    \sum_{j=1}^{t-1} \frac{\beta_j}{q_j}}\notag\\
  & =\frac{1}{\eta'_{t+1}}-\frac{1}{\eta_t'},   
\end{alignat}
where we used $\frac{\tilde{\gamma}_t}{\min_{j \leq t} q_j} \geq \frac{\tilde{\gamma}_{t-1}}{\min_{j \leq t-1} q_j}$ in the second inequality. Summing up over $t = 1,\ldots, \tau$ gives
\begin{align}\label{eq:integral_eta_t_lincon}
    \sum_{t=1}^{\tau}\left(\frac{1}{\eta'_{t+1}}-\frac{1}{\eta_t'}\right)=\frac{1}{\eta'_{\tau+1}}-\frac{1}{\eta_1'} \leq \frac{1}{\eta'_{\tau+1}}=\sqrt{\frac{ 800dK \tilde{\gamma}_{\tau}^2}{\min_{j \leq \tau} q_j} +    \sum_{j=1}^{\tau} \frac{\beta_j}{q_j}}.
\end{align}

Therefore, using the definition of $\eta_t$ and $\beta_t$, we have that
\begin{alignat*}{4}
  & \E \left[ \sum_{t=1}^\tau \frac{4 d K \eta_t \tilde{\gamma}_t^2}{q_t} 
  \xi_{t,A_t}^2 \right]
  =  \E \left[ \sum_{t=1}^\tau \frac{dK\beta_t}{4q_t} \eta_t\right]
  =  \E \left[\sum_{t=1}^\tau \frac{dK\beta_t}{4q_t}
 \frac{1}{\sqrt{\frac{ 800dK \tilde{\gamma}_t^2}{\min_{j \leq t} q_j} +    \sum_{j=1}^{t-1} \frac{\beta_j}{q_j}}}\right]\\
  & \leq  \E \left[  \sum_{t=1}^\tau  \frac{dK\beta_t}{2q_t}
 \frac{1}{\sqrt{\frac{ 800dK \tilde{\gamma}_t^2}{\min_{j \leq t} q_j} +    \sum_{j=1}^{t} \frac{\beta_j}{q_j}}}\right]
  \leq  dK  \E \left[\sum_{t=1}^{\tau}\left(\frac{1}{\eta'_{t+1}}-\frac{1}{\eta_t'}\right)\right]
 \leq  dK  \E \left[ \sqrt{\frac{ 800dK \tilde{\gamma}_{\tau}^2}{\min_{j \leq \tau} q_j} +    \sum_{j=1}^{\tau} \frac{\beta_j}{q_j}}\right]\\
 & = dK   \E \left[ \sqrt{\frac{ 800dK \tilde{\gamma}_{\tau}^2}{\min_{j \leq \tau} q_j} +    \sum_{t=1}^{\tau} \frac{16 \tilde{\gamma}_t^2 \xi_{t,A_t}^2}{q_t}}\right]\\
 & \leq 4dK \tilde{\gamma}_{\tau}  \E \left[ \sqrt{  \frac{50dK}{\min_{j \leq \tau} q_j}  +\sum_{t=1}^{\tau}\frac{\xi_{t,A_t}^2}{q_t}  }\right]\\
 & =   16 dK \log (10dK \tau)    \sqrt{  \frac{50dK}{\min_{j \leq \tau} q_j}  +\E \left[\sum_{t=1}^{\tau}\frac{  \upd_t 
 \xi_{t,A_t}^2}{q_t^2} \right] },
 \end{alignat*}
where we used $\frac{\beta_t}{q_t} \leq \frac{800dK\tilde{\gamma}_t^2}{\min_{j \leq t} q_j}$ in the first inequality and the second inequality follows from \Cref{eq:integral_eta_t_lincon1},
and the third inequality follows from \Cref{eq:integral_eta_t_lincon}. 

Next we evaluate the term $\E\left[\frac{K \log \tau}{ \eta_{\tau}} \right]$.

\begin{align*}
    &\E\left[\frac{K \log \tau}{ \eta_{\tau}} \right] \leq 
      K \log \tau \E\left[ \sqrt{\frac{ 800dK \tilde{\gamma}_{\tau}^2}{\min_{j \leq \tau} q_j} +    \sum_{t=1}^{\tau} \frac{16 \tilde{\gamma}_t^2 \xi_{t,a}^2}{q_t}}\right]\\
      & \leq    16K \log(\tau) \log(10dK \tau) \cdot  \sqrt{  \frac{50dK}{\min_{j \leq \tau} q_j}  +\E \left[\sum_{t=1}^{\tau}\frac{\upd_t \xi_{t,A_t}^2}{q_t^2}\right]  }
\end{align*}

Therefore, we conclude that
\begin{align*}
&\mathbb{E}\left[ \sum_{t=1}^\tau (\ell_t(X_t, a_t)-\ell_t(X_t, a^*)  ) \right]\\
   &\leq  16K\log(10dK \tau)  \left(  \log(\tau)  + d  \right)\cdot  \sqrt{  \frac{50dK}{\min_{j \leq \tau} q_j}+\E \left[\sum_{t=1}^{\tau}\frac{\upd_t \xi_{t,A_t}^2}{q_t^2}\right]  } +3\\
   & \leq 32 K d \log(10dK \tau) \log(\tau) \left(\sqrt{  \E \left[\sum_{t=1}^{\tau}\frac{\upd_t \xi_{t,A_t}^2}{q_t^2}\right]  } + \E \left[ 
\frac{\sqrt{50dK}}{\min_{j \leq \tau} q_j} \right] \right).
\end{align*}
\end{proof}

\begin{remark}
  We omitted the proof of \Cref{thm:secondorder} since
    using \Cref{proposi:theorem23ofDann+} and \ref{proposi:theorem22ofDann+} , and the dd-iw-stable condition proved in \Cref{proposi:ddiw_forlincon} immediately implies \Cref{thm:secondorder}. 
\end{remark}

\section{APPENDIX FOR \textsf{FTRL-LC} (ALGORITHM~\ref{alg:FTRLforcontextual})}\label{appendix_forFTRL}
In this appendix, we describe the detailed procedure of \MGR\ and all the technical proof for analysis of \textsf{FTRL-LC}.

\subsection{Matrix geometric resampling}\label{appendix:MGR}
We detail the whole procedure of \MGR\ in \Cref{alg:MGR} \citep{Neu_Bartok2013,Neu_Bartok2016, neu2020efficient}.
\MGR\ takes inputs of context distribution $\cD$, policy $\pi_t$, action $a \in [K]$, number of iterations $M_t$, and constant $\rho$, and outputs 
$\hatSigma_{t,a}^+ = \rho \mathbf{I} + \rho \sum_{k=1}^{M_t} \mathbf{A}_{k,a}$ as the estimate of the inverse of the covariance matrix $\SigmaInv_{t,a}$.
In this work, we set $\rho=\frac{1}{2}$.

\begin{algorithm}[t]
\caption{Matrix Geometric Resampling (\MGR) \citep{neu2020efficient}}\label{alg:MGR}
	\SetKwInOut{Input}{Input}
 	\SetKwInOut{Output}{Output}
	\Input{Context distribution $\cD$, policy $\pi_t$, action $a \in [K]$, number of iterations $M_t$, constant $\rho=\tfrac12$}
	\For{$k=1,2,\ldots, M_t$}{
        Draw $X(k) \sim \cD$ and $A(k) \sim \pi_t(\cdot |X(k))$

        Compute $\mathbf{B}_{k,a}= \ind{ A(k)=a } X(k)X(k)^{\top}$

        Compute $\mathbf{A}_{k,a}=\Pi_{j=1}^k (\mathbf{I} -\rho \mathbf{B}_{k,a})$
        }

        \Output{$\hatSigma_{t,a}^+ = \rho \mathbf{I} + \rho \sum_{k=1}^{M_t} \mathbf{A}_{k,a}$}
\end{algorithm}

\subsection{Useful lemma for the entropy term}
First, we introduce the following lemma, which implies that the definition of $\beta_t'$ based on entropy terms is crucial in the analysis for FTRL with Shannon entropy regularizer.
The proof follows the similar argument as Proposition 1 of~\citet{ito+2022nearly}.
\begin{lemma}\label{lem:penaltyterm_betaprime}
Let $\beta_t'$ be updated by \Cref{def_etabetagamma} for each round $t$.
Then for a ghast sample $X_0$,
we have
\begin{align*}
 \E \left [\sum_{t=1}^T \left(\beta'_{t+1}-\beta_t' \right) H(p_{t+1}(\cdot| X_0)) \right]=\cO \left(\alphaone \sqrt{\log K}  \sqrt{  \sum_{t=1}^T \E \left[ H(p_t(\cdot |X_0))\right] } \right).
\end{align*}
\end{lemma}

\begin{proof}[Proof of Lemma~\ref{lem:penaltyterm_betaprime}]
From our definition of $\beta'_t$, we have 
\begin{alignat*}{4}
     &\E \left [\sum_{t=1}^T \left(\beta'_{t+1}-\beta_t' \right) H(p_{t+1}(\cdot| X_0)) \right]   =  \E \left [\sum_{t=1}^T \frac{\alphaone }{\sqrt{1+ (\log K)^{-1} \sum_{s=1}^t H(p_s(\cdot | X_s))}}H(p_{t+1}(\cdot| X_0))\right]\\
     &  =  2\alphaone \sqrt{\ln{K}} \E \left [  \sum_{t=1}^T \frac{H(p_{t+1}(\cdot| X_0)) }{\sqrt{4 \ln{K}+ 4 \sum_{s=1}^t H(p_s(\cdot | X_s))}}\right]\\
     & = 2\alphaone \sqrt{\ln{K}} \E \left [  \sum_{t=1}^T \frac{H(p_{t+1}(\cdot| X_0))}{  \sqrt{ \ln{K}+ \sum_{s=1}^t H(p_s(\cdot | X_s))}  +  \sqrt{ \ln{K}+ \sum_{s=1}^t H(p_s(\cdot | X_s))} }\right]\\
     & \leq   2\alphaone \sqrt{\ln{K}} \E \left [  \sum_{t=1}^T \frac{H(p_{t+1}(\cdot| X_0))}{  \sqrt{ \sum_{s=1}^{t+1} H(p_s(\cdot | X_s))}  +  \sqrt{\sum_{s=1}^t H(p_s(\cdot | X_s))} }\right],
\end{alignat*}
where in the last step we used the fact that $H(p_s(\cdot|X_s)) \leq H(p_1(\cdot|X_1))=\log K$.
Using the property that $\E_{X_{t+1} \sim \cD}[ H(p_{t+1}(\cdot|X_{t+1}) ) | \cF_{t}]=\E_{X_0 \sim \cD}[ H(p_{t+1}(\cdot|X_0) ) | \cF_{t}]$,
we have  
\begin{alignat*}{4}
 &2\alphaone \sqrt{\ln{K}} \E \left [  \sum_{t=1}^T \frac{H(p_{t+1}(\cdot| X_0))}{  \sqrt{ \sum_{s=1}^{t+1} H(p_s(\cdot | X_s))}  +  \sqrt{\sum_{s=1}^t H(p_s(\cdot | X_s))} }\right]\\
 &= 2\alphaone \sqrt{\ln{K}} \E \left [  \sum_{t=1}^T \frac{H(p_{t+1}(\cdot| X_0))  \left(\sqrt{ \sum_{s=1}^{t+1} H(p_s(\cdot | X_s))}  -  \sqrt{\sum_{s=1}^t H(p_s(\cdot | X_s))}\right)}{   H(p_{t+1}(\cdot | X_{t+1})) }    \right]\\
  &= 2\alphaone \sqrt{\ln{K}} \E \left [  \sum_{t=1}^T \left( \sqrt{ \sum_{s=1}^{t+1} H(p_s(\cdot | X_s))}  -  \sqrt{\sum_{s=1}^t H(p_s(\cdot | X_s))}  \right) \right]\\
  & = 2\alphaone \sqrt{\ln{K}} \E \left [   \left( \sqrt{ \sum_{s=1}^{T+1} H(p_s(\cdot | X_s))}  -  \sqrt{ H(p_1(\cdot | X_1))}  \right) \right]\\
  & \leq 2\alphaone \sqrt{\ln{K}} \E \left [    \sqrt{ \sum_{s=1}^{T} H(p_s(\cdot | X_s))}  \right],
\end{alignat*}
where in the last step we again used the fact that $H(p_s(\cdot|X_s)) \leq H(p_1(\cdot|X_1))=\log K$.
Hence, again using the fact that $X_0$ and $X_t$ follows the same distribution $\cD$ and the linearity of the expectation, we obtain
\begin{align*}
 \E \left [\sum_{t=1}^T \left(\beta'_{t+1}-\beta_t' \right) H(p_{t+1}(\cdot| X_0)) \right]=
 \cO \left(\alphaone \sqrt{\log K}  \sqrt{  \sum_{t=1}^T \E \left[ H(p_t(\cdot |X_t))\right] } \right)= \cO \left(\alphaone \sqrt{\log K}  \sqrt{  \sum_{t=1}^T \E \left[ H(p_t(\cdot |X_0))\right] } \right),
\end{align*}
which concludes the proof.

\end{proof}

\subsection{Proof of Lemma~\ref{lem:regretdecom_auxiliarygame}}
The proof follows the standard analysis of FTRL with the negative Shannon entropy.

\begin{proof}[Proof of Lemma~\ref{lem:regretdecom_auxiliarygame}]
By \Cref{lemma:FTRLanalysis_lincon},
for any context $\bm{x} \in \cX$, we have 
\begin{alignat}{4}\label{eq: lemma:FTRLanalysis_lincon}
&\widetilde{R}_T(\bm{x})
=\E_{A_t}\left[ \sum_{t=1}^T \left( \la \bm{x}, \tiltheta_{t,A_t}  \ra -\la \bm{x}, \tiltheta_{t,\pi^*(\bm{x})} \ra \right)\right] \notag\\
\leq
&\sum_{t=1}^T \left(\psi_t(p_{t+1}(\cdot|\bm{x}))-\psi_{t+1}(p_{t+1}(\cdot|\bm{x})) \right) + \psi_{T+1}(\pi^*(\cdot|\bm{x}))-\psi_{1}(p_{1}(\cdot |\bm{x})) \notag\\
&\quad+\sum_{t=1}^{T} (1-\gamma_t)\left(\left\la p_t(\cdot |\bm{x})-p_{t+1}( \cdot |\bm{x})) , \tilde{\bell}_t(\bm{x})  \right\ra  - D_t(p_{t+1}(\cdot|\bm{x}), p_t(\cdot|\bm{x})) \right) 
   +U(\bm{x}).
\end{alignat}
We first bound the stability term 
$\left\la p_t(\cdot |\bm{x})-p_{t+1}( \cdot |\bm{x})) , \tilde{\bell}_t(\bm{x})  \right\ra  - D_t(p_{t+1}(\cdot|\bm{x}), p_t(\cdot|\bm{x}))$.
Since the function $f(q)=\sum_{a \in [K]}(p_t(a|\bm{x})-q(a))   \la \bm{x}, \tiltheta_{t,a} \ra - D_t(q, p_t(\cdot|\bm{x}))$ is concave with respect to $q \in \simplexK$, its maximum solution is obtained by computing the point where its derivative is equal to zero.
For each $a \in [K]$, we have
\begin{align*}
    \frac{\partial}{\partial q(a)}\left(\sum_{a \in [K]}(p_t(a|\bm{x})-q(a))   \la \bm{x}, \tiltheta_{t,a} \ra - D_t(q, p_t(\cdot|\bm{x}))\right)=-  \la \bm{x}, \tiltheta_{t,a} \ra-\frac{1}{\eta_t}(\log q(a)- \log p_t(a|\bm{x})),
\end{align*}
and thus the maximum solution is obtained for $q^*(a)=p_t(a|\bm{x}) \exp(-\eta_t \la \bm{x}, \tiltheta_{t,a} \ra)$. 
Hence, we can show 
\begin{alignat}{4}\label{eq:stability_step1}
    &\sum_{a \in [K]}(p_t(a|\bm{x})-p_{t+1}(a|\bm{x}))   \la \bm{x}, \tiltheta_{t,a} \ra - D_t(p_{t+1}(\cdot|\bm{x}), p_t(\cdot|\bm{x})) \notag \\
    & \leq \sum_{a \in [K]}(p_t(a|\bm{x})-q^*(a))   \la \bm{x}, \tiltheta_{t,a} \ra - D_t(q^*, p_t(\cdot|\bm{x}))\notag \\
    & = \sum_{a \in [K]} \left(  \la \bm{x}, \tiltheta_{t,a} \ra (p_t(a|\bm{x})-q^*(a)) -\frac{1}{\eta_t} (q^*(a) \log p_t(a|\bm{x})-p_t(a|\bm{x}) \log p_t(a|\bm{x})-(\log p_t(a|\bm{x})+1)(q^*(a)-p_t(a|\bm{x}))) \right)\notag \\
    &= \sum_{a \in [K] } \left(  \la \bm{x}, \tiltheta_{t,a} \ra p_t(a|\bm{x}) + \frac{1}{\eta_t}(q^*(a)-p_t(a|\bm{x})) \right)\notag \\
    & = \frac{1}{\eta_t} \sum_{a \in [K]} p_t(a|\bm{x}) \left( \exp(-\eta_t   \la \bm{x}, \tiltheta_{t,a} \ra) +\eta_t   \la \bm{x}, \tiltheta_{t,a} \ra -1\right).
\end{alignat}
Using the inequality
$\exp{(-x)} \leq 1-x+x^2$ that holds for any $x \geq -1$ and the assumption that $|\eta_t \la \bm{x},\widetilde{\bm{\theta}}_{t,a} \ra| \leq 1$, we can bound the RHS of \Cref{eq:stability_step1} is bounded as
\[
 \frac{1}{\eta_t} \sum_{a \in [K]} p_t(a|\bm{x}) \left( \exp(-\eta_t   \la \bm{x}, \tiltheta_{t,a} \ra) +\eta_t   \la \bm{x}, \tiltheta_{t,a} \ra -1\right) \leq 
 \eta_t \sum_{a \in [K]} p_t(a|\bm{x}) \la \bm{x}, \tiltheta_{t,a} \ra^2,
\]
implying that 
\begin{align*}
(1-\gamma_t)\sum_{a \in [K]}(p_t(a|\bm{x})-p_{t+1}(a|\bm{x}))   \la \bm{x}, \tiltheta_{t,a} \ra - D_t(p_{t+1}(\cdot|\bm{x}), p_t(\cdot|\bm{x})) \leq  (1-\gamma_t)\eta_t \sum_{a \in [K]} p_t(a|\bm{x}) \la \bm{x}, \tiltheta_{t,a} \ra^2.
\end{align*}
Since $p_t(a|\bm{x})=\frac{1}{1-\gamma_t} (\pi_t(a|\bm{x})-\frac{\gamma_t}{K})$ from the definition of $\pi_t(a|\bm{x})$,
we obtain
\begin{align}\label{eq:stabilitybound}
(1-\gamma_t)\sum_{a \in [K]}(p_t(a|\bm{x})-p_{t+1}(a|\bm{x}))   \la \bm{x}, \tiltheta_{t,a} \ra - D_t(p_{t+1}(\cdot|\bm{x}), p_t(\cdot|\bm{x})) \leq \eta_t \sum_{a \in [K]} \pi_t(a|\bm{x}) \la \bm{x}, \tiltheta_{t,a} \ra^2.
\end{align}

For the penalty term,
using  $0 \leq H(p) \leq \log K$ that holds for any $p \in \simplexK$, we can show  
\begin{align}\label{eq:penaltybound}
    \sum_{t=1}^T \left(\psi_t(p_{t+1}(\cdot|\bm{x}))-\psi_{t+1}(p_{t+1}(\cdot|\bm{x})) \right) + \psi_{T+1}(\pi^*(\cdot |\bm{x}))-\psi_{1}(p_{1}(\cdot|\bm{x}))\notag\\
    \leq   \sum_{t=1}^T \left(\beta_{t+1}-\beta_t \right) H(p_{t+1}(\cdot|\bm{x})) +\beta_1 \log K.
\end{align}
Combining \Cref{eq: lemma:FTRLanalysis_lincon}, \Cref{eq:stabilitybound}, and \Cref{eq:penaltybound} completes the proof of Lemma~\ref{lem:regretdecom_auxiliarygame}.
\end{proof}

\subsection{Proof of Lemma~\ref{lemma:parametersforMGR}}
\begin{proof}[Proof of Lemma~\ref{lemma:parametersforMGR}]
Let $\| \cdot \|_{\mathrm{op}}$ be the operator norm of any positive semi-definite matrix.
Recall that definitions of the baised estimator $\widetilde{\bm{\theta}}_{t,a}=\hatSigmaInv X_t \ell_t(X_t,A_t)\ind{A_t=a}$ and unbiased estimator $\widehat{\bm{\theta}}_{t,a}=\bm{\Sigma}_{t,a}^{-1} X_t \ell_t(X_t,A_t)\ind{A_t=a}$.
The first statements of (i) can be shown by using these definitions and adapting a similar analysis for Lemma 5 in \citet{neu2020efficient} (\Cref{lemma:biased_term_MGR}).
For $\hatSigmaInv$, the output of \MGR\ procedure in \Cref{alg:MGR} with $\rho=\tfrac12$,
we have $\E_t[\mathbf{A}_{t,a}]=\E_t\left[ \Pi_{j=1}^k \left(I -\rho \mathbf{B}_{k,a} \right) \right]=\left(I -\frac{1}{2} \mathbf{\Sigma}_{t,a} \right)^k$ for each $a \in [K]$.
Then, it gives $\E_t[\hatSigmaInv]= \frac{1}{2} \sum_{k=0}^{M_t} \left(I-\frac{1}{2} \mathbf{\Sigma}_{t,a}^{-1} \right)^k= \mathbf{\Sigma}^{-1}_{t,a} - \left(I-\frac{1}{2}\mathbf{\Sigma}_{t,a}  \right)^{M_t} \mathbf{\Sigma}^{-1}_{t,a}$.
Using these expressions, for the biased estimator $\widehat{\bm{\theta}}_{t,a}$ of each action $a \in [K]$, we have that
\begin{align*}
    \E_t[\tiltheta_{t,a}]&=\E_t[\hatSigmaInv X_t \ell_t(X_t,a) \ind{A_t=a} ]\\
    &= \E_t[\hatSigmaInv] \E_t[X_t \la X_t, \bm{\theta}_{t,a} \ra  \ind{A_t=a} ]\\
    & = \E_t[\hatSigmaInv] \E_t[ X_t X_t^{\top} \ind{A_t=a}] \cdot \bm{\theta}_{t,a}\\
    & =\E_t[\hatSigmaInv] \mathbf{\Sigma}_{t,a} \bm{\theta}_{t,a}\\
    &= \left(\mathbf{\Sigma}^{-1}_{t,a} - \left(I-\frac{1}{2}\mathbf{\Sigma}_{t,a}  \right)^{M_t} \mathbf{\Sigma}^{-1}_{t,a} \right)\mathbf{\Sigma}_{t,a} \bm{\theta}_{t,a}\\
    &=\bm{\theta}_{t,a}-\left(I-\frac{1}{2}\mathbf{\Sigma}_{t,a} \right)^{M_t}\bm{\theta}_{t,a},
\end{align*}
implying that 
\[
 \E_t[\tiltheta_{t,a}-\hattheta_{t,a}]=-\left(I-\frac{1}{2}\mathbf{\Sigma}_{t,a} \right)^{M_t}\bm{\theta}_{t,a}.
\]
Therefore, we obtain
\begin{align*}
    \E_t[\la X_t, \tiltheta_{t,a}-\hattheta_{t,a} \ra] & \leq \|X_t\|_2 \|\bm{\theta}_{t,a}\|_2   \left\| \left(I-\frac{1}{2} \mathbf{\Sigma}_{t,a} \right)^{M_t}  \right\|_{\mathrm{op}} \leq \left\| \left(I-\frac{1}{2} \mathbf{\Sigma}_{t,a} \right)^{M_t}  \right\|_{\mathrm{op}} \\
    & \leq \left( 1- \frac{\gamma_t \lambdaminSigma}{ 2K}\right) \leq \exp{ \left( - \frac{\gamma_t \lambdaminSigma}{ 2K} \cdot M_t \right)  } \leq \frac{1}{t^2}\,,
\end{align*}
where we used $\|X_t\| \leq 1$ and  $\|\bm{\theta}_{t,a}\|_2\leq 1$ in the second inequality,
we used the fact that the policy $\pi(\cdot | X_t)$ employs the uniform exploration with mixing rate $\gamma_t$ in the third inequality,
and the last step follows by $M_t = \left\lceil \frac{4K}{\gamma_t \lambdaminSigma } \log t \right\rceil$.

Next we consider the second statement of (ii), which can be shown via our careful tuning of learning parameters. 
For the output of \MGR\ procedure in \Cref{alg:MGR} with $\rho=\tfrac12$ and any $\bm{x} \in \cX$,
$|\eta_t \la \bm{x},\widetilde{\bm{\theta}}_{t,a} \ra|$ for each $a \in [K]$ is bounded as follows:
\begin{align}\label{eq:normloss}
|\eta_t \la \bm{x},\widetilde{\bm{\theta}}_{t,a} \ra| & = \eta_t \big|\la \bm{x}, \hatSigmaInv X_t \ell_t(X_t,A_t) \ind{A_t=1} \ra \big|
 \leq \eta_t \big|\bm{x}^{\top} (\hatSigmaInv X_t) \big| \leq \eta_t \|\hatSigmaInv \|_{\mathrm{op}}  \notag \\
&\leq \eta_t\left(  \left\| \rho I + \rho \sum_{k=1}^{M_t} \mathbf{A}_{k,a}   \right\|_{\mathrm{op}}\right)
\leq  \frac{\eta_t}{2}\left(  1+ \sum_{k=1}^{M_t}  \left\| \Pi_{j=1}^k \left(I-\frac{1}{2} \mathbf{B}_{k,a} \right)  \right\|_{\mathrm{op}} \right)
\leq \frac{\eta_t (M_t+1)}{2},
\end{align}
where the first equality follows from the definition of $\widetilde{\bm{\theta}}_{t,a}$,
the first inequality follows from $\ell_t(X_t,A_t) \leq 1$,
and the second inequality follows from $\max_{\bm{x} \in \cX}\|\bm{x}\|_2 \leq 1$.
Setting $M_t = \left\lceil \frac{4K}{\gamma_t \lambdaminSigma } \log t \right\rceil$ gives
\[
 \frac{1}{\eta_t}=\frac{2}{\eta_t} -\frac{\alphacirct}{\alphacirct\eta_t}
 = \frac{2}{\eta_t} -\frac{\alphacirct}{\gamma_t} \leq \frac{2}{\eta_t} -(M_t-1),
\]
where we used the definition of  $\gamma_t=\alphacirct \eta_t$ for $\alphacirct=\frac{4K\log t}{\lambdaminSigma}$.
Therefore, from the definition of $\eta_t \leq \frac{1}{2}$,
we have  $2  \leq \frac{2}{\eta_t}-(M_t-1)   \Leftrightarrow \eta_t \leq \frac{2}{M_t+1}$.
Combining it with \Cref{eq:normloss} guarantees that $|\eta_t \la \bm{x},\widetilde{\bm{\theta}}_{t,a} \ra| \leq 1$, as desired.
\end{proof}

\subsection{Proof of Lemma~\ref{keylemma:Expected_regret_auxiliarygame}}

\begin{proof}[Proof of Lemma~\ref{keylemma:Expected_regret_auxiliarygame}]

By Lemma~\ref{lemma:parametersforMGR} and the definitions of $\beta_t, \eta_t, \gamma_t$ and $M_t$, we can see that $|\eta_t \la X_0,\widetilde{\bm{\theta}}_{t,a} \ra| \leq 1$ holds, which allow us to use Lemma~\ref{lem:regretdecom_auxiliarygame} for fixed $X_0$.
Then we have
\begin{align}\label{eq:regretdecom_auxiliarygame}
             \E[\widetilde{R}_T(X_0)] 
\leq \underbrace{\E \left[ \sum_{t=1}^T \left(\beta_{t+1}-\beta_t \right) H(p_{t+1}(\cdot|X_0)) \right]}_{\text{term}\, A}
+\underbrace{\E \left[\sum_{t=1}^T\eta_t \sum_{a \in [K]} \pi_t(a|X_0) \la X_0, \tiltheta_{t,a} \ra^2 \right]}_{\text{term}\, B}
+ \E \left[U(X_0)\right]+\beta_1 \log K .
     \end{align} 
Using the definition of $\beta_t=\max\{2,\alphatwo \log T, \beta_t'\}$, we have
\begin{align}\label{eq:beta_1logK}
    \beta_1 \log K \leq\alphatwo \log K \log T.
\end{align}

Next, we will evaluate term$\,B$ and $\E[U(X_0)]$.
From the definition of $\beta'_t$ in \Cref{def_etabetagamma}, we see that
\begin{alignat*}{4}
    & \beta'_t=\alphaone +\sum_{s=1}^{t-1} \frac{\alphaone }{\sqrt{1+(\log K)^{-1} \sum_{u=1}^{s-1} H(p_u(\cdot | X_u))  }} \geq  \frac{\alphaone  t }{\sqrt{1+(\log K)^{-1} \sum_{s=1}^{t} H(p_s(\cdot | X_s))  }}, 
\end{alignat*}
and thus 
\begin{align}\label{eq:sumofeta}
  &\sum_{t=1}^T  \eta_t \leq  \sum_{t=1}^T  \frac{1}{\beta'_t} \leq   \sum_{t=1}^T \frac{\sqrt{1+(\log K)^{-1} \sum_{s=1}^{t} H(p_s(\cdot | X_s))  }}{\alphaone  t } \notag \\
 & \leq \frac{1+\log T}{\alphaone } \sqrt{1+(\log K)^{-1} \sum_{s=1}^{T} H(p_s(\cdot | X_s)) } = \cO \left(\frac{\log T}{\alphaone  \sqrt{\log K}}  \sqrt{ \sum_{t=1}^T H(p_t(\cdot| X_t))}\right),
\end{align}
where we used $H(p_1(\cdot|X_1))=\log K$.

By Lemma~\ref{lemma:Lemma6_neu} and \Cref{eq:sumofeta}, we obtain
\begin{align}\label{eq:varianceterm}
\text{term}\, B=&\E \left[ \sum_{t=1}^T\eta_t \sum_{a \in [K]} \pi_t(a|X_0) \la X_0, \tiltheta_{t,a} \ra^2 \right]
=\cO \left(\E \left[  \frac{3Kd \cdot  \log T}{\alphaone  \sqrt{\log K}} \sqrt{ \sum_{t=1}^T H(p_t(\cdot| X_t))} \right] \right)\notag \\
&=\cO \left(  \frac{3Kd \cdot  \log T}{\alphaone  \sqrt{\log K}} \sqrt{ \E \left[\sum_{t=1}^T H(p_t(\cdot| X_0))\right]}  \right),
\end{align}
where we used the fact that $\E_{X_0 \sim \cD}[p_t(\cdot| X_0)|\tiltheta_t]=\E_{X_t \sim \cD}[p_t(\cdot| X_t))|\tiltheta_t]$.

For $\E[U(X_0)]$,
from Lemma~\ref{lemma:parametersforMGR},
we have $|\E[\la X_t, \widetilde{\bm{\theta}}_{t,a}-\hattheta_{t,a} \ra  |  \cF_{t-1}] |
\leq \frac{1}{t^2} \leq 1$, and thus
\begin{align}\label{eq:boundforuniformallocation1}
&\E[U(X_0)]= \E \left[ \sum_{t=1}^T \gamma_t \sum_{a \in [K]} \left( \frac{1}{K}-\pi^*(a|X_0) \right) \la \bm{x}, \tiltheta_{t,a} \ra\right] \leq 
\E \left[ \sum_{t=1}^T \gamma_t  \max_{a \in [K]} \la X_0, \tiltheta_{t,a} - \hattheta_{t,a}+\hattheta_{t,a}\ \ra\right]\notag \\
&\leq 
\E \left[ \sum_{t=1}^T \gamma_t \left(  \max_{a \in [K]} \la X_0, \tiltheta_{t,a} - \hattheta_{t,a}\ra + \ell_t(X_0, a) \right) \right]
\leq \E \left[ \sum_{t=1}^T \gamma_t \left(  \max_{a \in [K]} \big| \la X_0, \tiltheta_{t,a} - \hattheta_{t,a} \ra \big|+ 1 \right) \right] \notag\\
&\leq 2\E \left[\sum_{t=1}^T \gamma_t   \right].
\end{align}
where we used $\E[\hattheta_{t,a}]=\bm{\theta}_{t,a}$ and  $\E[\ell_t(X_0, a)] \leq 1$ in the second and third inequality.
From the definition of $\gamma_t$ and \Cref{eq:sumofeta}, we have
\begin{align}\label{eq:sum_gamma}
\E \left[\sum_{t=1}^T \gamma_t \right]
=\E \left[\sum_{t=1}^T \alphacirct \eta_t\right] \leq 
\E\left[\sum_{t=1}^T \frac{4K\log T}{\lambdaminSigma } \cdot \eta_t\right]=
\cO\left(\frac{K\log T}{\lambdaminSigma } \cdot \frac{\log T}{\alphaone  \sqrt{\log K}}  \sqrt{ \E \left[\sum_{t=1}^T H(p_t(\cdot| X_0))\right]}\right).
\end{align}
Thus from \Cref{eq:boundforuniformallocation1} and \Cref{eq:sum_gamma}, we obtain
\begin{align}\label{eq:boundforuniformallocation}
      \E[U(X_0)] \leq  2\E \left[\sum_{t=1}^T \gamma_t  \right] =\cO \left(\frac{K\log^2 T}{\alphaone  \lambdaminSigma  \sqrt{\log K}}  \sqrt{ \E \left[\sum_{t=1}^T H(p_t(\cdot| X_0))\right]}\right).
\end{align}

Finally, we will evaluate term$\,A$.
Let $t_0$ be the first round in which $\beta_t'$ becomes larger than the constant $F:=\max\{2,\alphatwo \log T  \}$, i.e., $t_0 = \min\{ {t \in [T]}: \beta_t' \geq F\}$. Then, by the definition of $\beta_t$, we have that 
\begin{align}\label{eq:termA_entropy}
    & \E \left[\sum_{t=1}^T \left(\beta_{t+1}-\beta_t \right) H(p_{t+1}(\cdot | X_0))\right] \notag \\
    &=\E \left[\sum_{t=1}^{t_0-2}\left(\beta_{t+1}-\beta_t \right) H(p_{t+1}(\cdot | X_0)) +\left(\beta_{t_0}-\beta_{t_0-1} \right) H(p_{t+1}(\cdot | X_0)) +\sum_{t=t_0}^{T}\left(\beta_{t+1}-\beta_t \right)H(p_{t+1}(\cdot | X_0))\right] \notag \\
    & \leq \E \left[0+ \left(\beta'_{t_0}-\beta'_{t_0-1} \right) H(p_{t+1}(\cdot | X_0))+\sum_{t=t_0}^{T}\left(\beta'_{t+1}-\beta'_t \right) H(p_{t+1}(\cdot | X_0))\right]\notag \\
    & \leq \E \left[\sum_{t=1}^{T}\left(\beta'_{t+1}-\beta'_t \right)H(p_{t+1}(\cdot | X_0))\right]=
    \cO \left(\alphaone \sqrt{\log K}  \sqrt{  \sum_{t=1}^T \E \left[ H(p_t(\cdot |X_0))\right] } \right),
     \end{align}
where the first inequality is due to the fact that $\beta_t$ is the constant while $t \in [t_0-1]$, $\beta'_t \leq \beta_t$ for any $t$, and $\beta_t'=\beta_t$ for $t \geq t_0$. The last step follows by Lemma~\ref{lem:penaltyterm_betaprime}.
Hence using \Cref{eq:termA_entropy} and the fact that $X_0$ and $X_t$ follows the same distribution $\cD$, we obtain
\begin{align}\label{eq:boundforpenalty}
    \E \left[\sum_{t=1}^T \left(\beta_{t+1}-\beta_t \right) H(p_{t+1}(\cdot | X_0))  \right] = \cO \left(\alphaone \sqrt{\log K}  \sqrt{  \E \left[\sum_{t=1}^T  H(p_t(\cdot| X_0)) \right]} \right).
\end{align}

Combining  \Cref{eq:beta_1logK},
\Cref{eq:varianceterm}, \Cref{eq:boundforuniformallocation}, \Cref{eq:boundforpenalty} with \Cref{eq:regretdecom_auxiliarygame}, we obtain
\begin{equation*}
\E[\widetilde{R}_T(X_0)]=
\mathcal{O}\left(\left(\alphaone  \sqrt{\log K}
+\frac{\left( 3Kd+\frac{2 K \log T}{\lambdaminSigma}\right) \log T }{\alphaone  \sqrt{\log K}}\right)\sqrt{ \E \left[\sum_{t=1}^T H(p_t(\cdot | X_0))\right]}
+\alphatwo \log K \log T\right),
\end{equation*}
and plugging
$\alphatwo=\frac{8K}{\lambdaminSigma}$ to this bound concludes the proof.

\end{proof}

\subsection{Proof of Theorem~\ref{thm:FTRLforcontextual}}

\begin{proof}[Proof of Theorem~\ref{thm:FTRLforcontextual}]

Using Lemmas~\ref{lemma:equ6_neu2020} and~\ref{keylemma:Expected_regret_auxiliarygame}, we have
\begin{alignat*}{4}
&R_{T} \leq  \E[\widetilde{R}_T(X_0)]+2 \sum_{t=1}^{T}\max_{a \in [K]}|\E[\la X_t, \bm{b}_{t,a} \ra]|\\
&=\cO\left(\left(\alphaone  \sqrt{\log K}
+\frac{\left(3Kd+\frac{ K \log T}{\lambdaminSigma}\right) \log T }{\alphaone  \sqrt{\log K}}\right) \sqrt{ \E \left[\sum_{t=1}^T H(p_t(\cdot| X_0))\right]}
+\alphatwo \log K \log T+4\right)
\end{alignat*}
    where in the second step we used Lemma~\ref{lemma:parametersforMGR} with $M_t = \left\lceil \frac{4K}{\gamma_t  \lambdaminSigma } \log t \right\rceil$ to have
\begin{align}
    \sum_{t=1}^{T}\max_{a \in [K]}|\E[\la X_t, \bm{b}_{t,a} \ra]| \leq \sum_{t=1}^{T} \frac{1}{t^2} \leq 2.
\end{align}

Setting 
\begin{align*}
\alphaone =\sqrt{\left(3Kd+\frac{ 2K \log T}{\lambdaminSigma }\right)\frac{\log T}{\log K}}
\end{align*}
gives
\begin{align}\label{eq:regret_bound_Hp}
&R_T=\cO \left(\left(\alphaone \sqrt{\log K}+ \frac{\left(Kd+\frac{ K \log T}{\lambdaminSigma }\right) \log T}{\alphaone  \sqrt{\log K}}\right) \sqrt{ \E \left[\sum_{t=1}^T H(p_t(\cdot| X_0))\right]}+\alphatwo \log K \log T\right)\notag\\
& =\cO \left(
\sqrt{\left(Kd+\frac{ K \log T}{\lambdaminSigma }\right)\log T} \sqrt{ \E \left[\sum_{t=1}^T H(p_t(\cdot| X_0))\right]}+\alphatwo \log K \log T \right)\notag\\
&=\cO \left(
\sqrt{\left(d+\frac{\log T}{\lambdaminSigma }\right)K\log T \cdot  \E \left[\sum_{t=1}^T H(p_t(\cdot| X_0))\right]}+\frac{K}{\lambdaminSigma }\log K \log T \right)
\end{align}
where we used $\alphatwo=\frac{8K}{\lambdaminSigma}$ in the third equality.

For the adversarial regime, due to \Cref{eq:regret_bound_Hp} and 
the fact that $\sum_{t=1}^T H(p_t(\cdot| X_0)) \leq T \log K$, 
it holds that 
\begin{align*}
    R_{T} =O  \left( 
\sqrt{T \left(d+\frac{\log T}{\lambdaminSigma}\right)K\log(T)\log(K)} +\frac{K}{\lambdaminSigma} \log K \log T \right),
\end{align*}
as desired. 

\paragraph{Applying self-bounding techniques.}
Now, we will apply self-bounding techniques~\citep{ZimmertSeldin2021,WeiLuo2018} to proceed with further analysis.
\begin{lemma}\label{lemma:self-bounding constraints}
For any corrupted stochastic regime, 
the regret is bounded from below by
\begin{align*}
    R_T \geq \mathbb{E}\left[ \sum_{t=1}^T \Delta_{X_t}(A_t) \right] - 2 C. 
\end{align*}

\end{lemma}
\begin{proof}

Recall that $\Delta_x(a)$ is defined  as $\Delta_x(a):= \la \bm{x},\bm{\theta}_{a}-\bm{\theta}_{\pi^*(\bm{x})}\ra$ for $\bm{x} \in \cX$ and each action $a \in [K]$.

We have
\begin{align*}
    R_T&=\E\left[\sum_{t=1}^T (\ell_t(X_t,A_t)-\ell_t(X_t, \pi^*(X_t)))\right]\\
    & = \E\left[\sum_{t=1}^T \la X_t, \bm{\theta}_{t,A_t}-  \bm{\theta}_{t,\pi^*(X_t)} \ra\right] + \E\left[\sum_{t=1}^T  \la X_t,   \bm{\theta}_{A_t}-\bm{\theta}_{A_t}  \ra \right]+\E\left[\sum_{t=1}^T  \la X_t,   \bm{\theta}_{\pi^*(X_t)}-\bm{\theta}_{\pi^*(X_t)}  \ra \right]\\
    & = \E\left[\sum_{t=1}^T \la X_t, \bm{\theta}_{A_t}-  \bm{\theta}_{\pi^*(X_t)} \ra\right] + \E\left[\sum_{t=1}^T  \la X_t,   \bm{\theta}_{t,A_t}-\bm{\theta}_{A_t}  \ra \right]+ \E\left[\sum_{t=1}^T  \la X_t,   \bm{\theta}_{\pi^*(X_t)}-\bm{\theta}_{t,\pi^*(X_t)}  \ra \right]\\
    & \geq  \E\left[\sum_{t=1}^T \la X_t, \bm{\theta}_{A_t}-  \bm{\theta}_{\pi^*(X_t)} \ra\right] - \E\left[\sum_{t=1}^T \big|  \la X_t,   \bm{\theta}_{t,A_t}-\bm{\theta}_{A_t}  \ra  \big |\right]- \E\left[\sum_{t=1}^T  \big| \la X_t,   \bm{\theta}_{\pi^*(X_t)}-\bm{\theta}_{t,\pi^*(X_t)}  \ra  \big|\right]\\
    & \geq\E\left[\sum_{t=1}^T \la X_t, \bm{\theta}_{A_t}-  \bm{\theta}_{\pi^*(X_t)} \ra\right] - 2\E\left[\sum_{t=1}^T \max_{a \in [K]}  \| X_t\|_2   \|\bm{\theta}_{t,a}-\bm{\theta}_a \|_2\right]\\
    & \geq  \sum_{t=1}^T \Delta_{X_t}(A_t)-    2\E\left[ \sum_{t=1}^T   \max_{a \in [K]} 
 \|\bm{\theta}_{t,a}-\bm{\theta}_a \|_2 \right],\\  
&\geq \sum_{t=1}^T \Delta_{X_t}(A_t)-2  C,
\end{align*}
where  we used the definition of $\Delta_{X_t}(A_t)$ in the third inequality,
and we used the definition of the corruption level $C\geq 0$ in the last inequality.

\end{proof}

We further show the regret upper bound based on the following notation.
For the optimal policy $\pi^* \in \Pi$,
\begin{align}\label{def_Q}
  \ \varrho_0(\pi^*):=\sum_{t=1}^T (1-p_t(\pi^*(X_0)|X_0)), \quad   
  \varrho_{(X_t)_{t=1}^T}(\pi^*):=\sum_{t=1}^T (1-p_t(\pi^*(X_t)|X_t)), \quad \bar{\varrho}_X(\pi^*):=\E[\varrho_{(X_t)_{t=1}^T}(\pi^*)]. 
    \end{align} 
Note that it holds that $0  \leq \bar{\varrho}_X(\pi^*) \leq T$. 
We also confirm the property on them in the following lemma.
\begin{lemma}\label{lemma:property of varrho}
Let $\pi^*$ be the optimal policy defined in \Cref{def:optimal_policy}. Then we have
$\bar{\varrho}_X(\pi^*)=\E[\varrho_0(\pi^*)]$.
\end{lemma}
\begin{proof}[Proof of \Cref{lemma:property of varrho}]
Notice that since the optimal policy $\pi^* \in \Pi$ is the deterministic policy, it holds that $\E_{X_0 \sim \cD}[\pi^*(X_0)]= \E_{X_t \sim \cD}[\pi^*(X_t)]$.
Let $\tiltheta_t=(\tiltheta_{t,1},\ldots,\tiltheta_{t,K})$.
Then we have
\begin{align*}
   & \E_t[p_t(\pi^*(X_0)|X_0)]=\E_{X_0 \sim \cD}[p_t(\pi^*(X_0)|X_0)|\tiltheta_{t}]=\E_{X_t \sim \cD}[p_t(\pi^*(X_t)|X_t)|\tiltheta_{t}]=\E_t[p_t(\pi^*(X_t)|X_t)].
\end{align*}
Hence, we have
\[
\sum_{t=1}^T \E_t[(p_t(\pi^*(X_0)|X_0)]=\sum_{t=1}^T \E_t[(p_t(\pi^*(X_t)|X_t)],
\]
which concludes the proof.
    
\end{proof}

We next show that the regret is bounded in terms of $\bar{\varrho}_X(\pi^*)$.
\begin{lemma}\label{lem:self-boundingineq}
  In the corrupted stochastic setting, the regret is bounded from below as
  \begin{equation}
    R_T 
    \geq 
    \frac{\Delta_{\min}}{2}\bar{\varrho}_X(\pi^*)-2C.
  \end{equation}
\end{lemma}
\begin{proof}[Proof of Lemma~\ref{lem:self-boundingineq}]

Recall that  $\Delta_x(a):=\bm{x}^{\top}(\theta_{a}-\theta_{\pi^*(\bm{x})})$ for $a \in [K] \setminus \{\pi^*(\bm{x})\}$, where $\pi^*$ is the unique optimal policy given by~\Cref{def:optimal_policy}.
Also recall that  $\Delta_{\min}(\bm{x}):=\min_{a \neq \pi^*(\bm{x})} \Delta_x(a)$ and $\Delta_{\min}:=\min_{\bm{x} \in \cX} \Delta_{\min}(\bm{x})$.
Then, using these gap definitions and \Cref{lemma:self-bounding constraints}, 
the regret is bounded from below as 
\begin{align*}
 R_T & \geq \E \left[\sum_{t=1}^T \Delta_{X_t}(A_t)-2C \right]
 =\E \left[\sum_{t=1}^T\sum_{a \in [K] \setminus \{\pi^*(X_t)\}} \pi_t(a|X_t) \Delta_{X_t}(a) \right]-2C\\
&  \geq \E \left[\sum_{t=1}^T\sum_{a \in [K] \setminus \{\pi^*(X_t)\}} (1-\gamma_t) p_t(a|X_t) \Delta_{X_t}(a) \right]-2C\\
&  \geq\frac{1}{2} \E \left[\sum_{t=1}^T\sum_{a \in [K] \setminus \{\pi^*(X_t)\}}  p_t(a|X_t) \Delta_{X_t}(a) \right]-2C\\
&  \geq\frac{1}{2} \E \left[\sum_{t=1}^T\sum_{a \in [K] \setminus \{\pi^*(X_t)\}}  p_t(a|X_t) \min_{a \in [K] \setminus \{\pi^*(X_t)\}}\Delta_{X_t}(a) \right]-2C\\
& =\frac{1}{2} \E \left[\sum_{t=1}^T\sum_{a \in [K] \setminus \{\pi^*(X_t)\}}  p_t(a|X_t) \Delta_{\min}( X_t) \right]-2C\\
& \geq \frac{1}{2} \E \left[\sum_{t=1}^T\sum_{a \in [K] \setminus \{\pi^*(X_t)\}}  p_t(a|X_t) \min_{\bm{x} \in \cX} \Delta_{\min}(\bm{x}) \right]-2C\\
& =\frac{ \Delta_{\min}}{2} \E \left[\sum_{t=1}^T\sum_{a \in [K] \setminus \{\pi^*(X_t)\}}  p_t(a|X_t) \right]-2C\\
& = \frac{ \Delta_{\min}}{2} \E \left[\sum_{t=1}^T (1- p_t(\pi^*(X_t)|X_t)) \right]-2C\\
&= \frac{ \Delta_{\min}}{2} \E \left[\varrho_{(X_t)_{t=1}^T}(\pi^*) \right]-2C=\frac{ \Delta_{\min}}{2} \bar{\varrho}_X(\pi^*)-2C,
\end{align*}
 where the second inequality follows by \Cref{eq:FTRL_shannon_entropy}, the third inequality follows by $\gamma_t \leq \frac{1}{2}$,
 and the last steps follows by the definitions of $\varrho_{(X_t)_{t=1}^T}(\pi^*):=\sum_{t=1}^T (1- p_t(\pi^*(X_t)|X_t))$ and $\bar{\varrho}_X(\pi^*):=\E[\varrho_{(X_t)_{t=1}^T}(\pi^*)]$.
\end{proof}

 The following lemma that bounds the sum of entropy in terms of $\varrho_0(\pi^*)$ follows by a similar argument as Lemma 4 of~\citet{ito+2022nearly}.
\begin{lemma}\label{lemma4_ito+2022nearly}
    For any $\pi \in \Pi$ and for a fixed ghost sample $X_0$, we have
\begin{align*}
  \sum_{t=1}^T  H(p_t(\cdot| X_0)) \leq \varrho_0(\pi^*) \log \frac{\mathrm{e}KT}{\varrho_0(\pi)},
\end{align*}
where $\varrho_0(\pi)=\sum_{t=1}^T (1-p_t(\pi(X_0)|X_0))$.
\end{lemma}

\begin{proof}[Proof of Lemma~\ref{lemma4_ito+2022nearly}]
By the similar calculation of (30) in \cite{ito+2022nearly},
we see that for any distribution $p \in \simplexK$, and for any $i^* \in [K]$, it holds that
    \begin{align*}
        H(p) \leq (1-p_{i^*}) \left( \log \frac{K-1}{1-p_{i^*}}+1 \right).  
    \end{align*}
Using this inequality, for a fixed $X_0$, it holds that
\begin{align*}
  \sum_{t=1}^T  H(p_t(\cdot| X_0))& \leq   \sum_{t=1}^T (1-p_t(\pi^*(X_0)|X_0)) \left( \log \frac{K-1}{1-p_t(\pi^*(X_0)|X_0)}+1 \right)\\
 & \leq  \varrho_0(\pi^*) \left(\log \frac{(K-1)T}{\varrho_0(\pi^*)}+1\right)  \leq  \varrho_0(\pi^*) \log \frac{\mathrm{e}KT}{\varrho_0(\pi^*)},
\end{align*}
    where the second inequality follows from Jensen's inequality.
\end{proof}

Using Lemma~\ref{lemma4_ito+2022nearly}, we have $\sum_{t=1}^T  H(p_t(\cdot| X_0)) \leq \mathrm{e}\log (\mathrm{e}KT)+\mathrm{e}^{-1}$ in the case of $\varrho_0(\pi^*)<\mathrm{e}$, which gives us the desired bound.
Next, we consider the case of $\varrho_0(\pi^*) \geq \mathrm{e}$.
In this case, we have
$\sum_{t=1}^T  H(p_t(\cdot| X_0)) \leq \varrho_0(\pi^*) \log (KT)$.
Hence, for  $\pi^* \in \Pi$ we obtain
\begin{align}\label{eq:Hp_barQ}
\E \left[ \sum_{t=1}^T  H(p_t(\cdot| X_0)) \right] \leq \E \left[\varrho_0(\pi^*) \log (KT)\right] = \E \left[ \varrho_{(X_t)_{t=1}^T}(\pi^*) \log (KT)\right] = \bar{\varrho}_X(\pi^*)\log(KT),
\end{align}
where we used \Cref{lemma:property of varrho} in the first equality.

Let $c_{4}=\frac{K}{\lambdaminSigma } \log (K) \log (T)+4$.
Therefore, by Lemma~\ref{lem:self-boundingineq}, \Cref{eq:regret_bound_Hp}, and \Cref{eq:Hp_barQ} for any $\lambda>0$,
it holds that
\begin{align*}
    R_T&=(1+\lambda)R_T-\lambda R_T\\
    &\leq \E\left[(1+\lambda)\sqrt{\left(d+\frac{\log T}{\lambdaminSigma }\right)K\log T \log(KT)\cdot \bar{\varrho}_X(\pi^*)}- \lambda\frac{\Delta_{\min}}{2}\bar{\varrho}_X(\pi^*) \right] +\lambda \cdot 2C+ (1+\lambda)c_{4}\\
    &=\cO \left(\frac{(1+\lambda)^2 (d+\frac{\log T}{\lambdaminSigma })K \log T \log (KT)}{\lambda \Delta_{\min}}   +\lambda C +\lambda c_{4}\right)\\
    &=\cO \left(\frac{(d+\frac{\log T}{\lambdaminSigma })K \log T \log (KT)}{\Delta_{\min}} + \lambda \left(  \frac{(d+\frac{\log T}{\lambdaminSigma })K \log T \log (KT)}{\Delta_{\min}}  +C \right) \right.\\
    & \left. \qquad \qquad+ \frac{(d+\frac{\log T}{\lambdaminSigma })K \log T \log (KT)}{\Delta_{\min}\cdot \lambda}+\lambda c_{4}   \right),
\end{align*}
where we used $a\sqrt{x}-\frac{bx}{2} \leq \frac{a^2}{2b}$ for any $a,b,x \geq 0$ in the first equality.

By letting $0 \leq \lambda \leq 1$ to be
\[
\lambda=\sqrt{\frac{(d+\frac{\log T}{\lambdaminSigma })K \log T \log (KT)\Delta^{-1}_{\min}}{ (d+\frac{\log T}{\lambdaminSigma })K \log T \log (KT)\Delta^{-1}_{\min}+C+c_4}},
\]
we have
\begin{align*}
R_T&=\cO \left(\frac{(d+\frac{\log T}{\lambdaminSigma })K \log T \log (KT)}{\Delta_{\min}} + \sqrt{\frac{(d+\frac{\log T}{\lambdaminSigma })K \log T \log (KT)}{\Delta_{\min}} \cdot C}\right.\\
    & \left. \quad \quad + \sqrt{\frac{c_4 (d+\frac{\log T}{\lambdaminSigma })K \log T \log (KT)}{\Delta_{\min}}}\right)\\
    &=\cO \left(\frac{(d+\frac{\log T}{\lambdaminSigma })K \log T \log (KT)}{\Delta_{\min}} + \sqrt{\frac{(d+\frac{\log T}{\lambdaminSigma })K \log T \log (KT)}{\Delta_{\min}} \cdot C}\right.\\
    & \left. \quad \quad +
    K\log T \sqrt{
    \frac{\frac{1}{\lambdaminSigma}\left(d+\frac{\ln T}{\lambdaminSigma}
    \right) \log (K) \log (KT)}{\Delta_{\min}}
    }
    \right)
    ,
\end{align*}
which concludes the proof of the theorem.
\end{proof}













\vfill




\end{document}